\pdfoutput=1
\RequirePackage[T1]{fontenc}
\documentclass[12pt]{article}
\usepackage[skip=1em plus 0.2em minus 0.1em]{parskip}
\usepackage[height=8.85in,width=6.45in]{geometry}

\usepackage[utf8]{inputenc}
\usepackage{amsmath}
\usepackage{amssymb}
\usepackage{mathtools}
\numberwithin{equation}{section}
\usepackage{slashed}
\usepackage{braket}
\usepackage{enumitem}
\usepackage[svgnames]{xcolor}
\usepackage[colorlinks,citecolor=DarkGreen,linkcolor=FireBrick,urlcolor=FireBrick,linktocpage,unicode]{hyperref}

\usepackage{tabularx}

\usepackage{tocloft}

\usepackage{footmisc}
\allowdisplaybreaks
\usepackage{graphicx}
\usepackage{tikz}
\usepackage{tikz-cd}
\usepackage{times}
 
\usepackage{bm}
\usepackage{physics}
\usepackage{xcolor}
\usepackage{natbib}
\usepackage{mdframed}
\usepackage{nicefrac}
\usepackage{booktabs}
\usepackage{lipsum}
\usepackage{titlesec}
\usepackage{wrapfig,lipsum,booktabs}
\usepackage{blindtext}
\usepackage{caption}

\usepackage{algorithm}
\usepackage{algpseudocode}
\newcommand{\commentsymbol}{//}%
\algrenewcommand\algorithmiccomment[1]{\hfill {\footnotesize \commentsymbol{} #1}}

\usepackage{titletoc}
\usepackage{todonotes}
\usepackage{setspace}
\usepackage{dsfont} %
\newcommand{\one}{\mathds{1}}

\usepackage[nameinlink,capitalize,noabbrev]{cleveref}

\usepackage{CJK}

\usepackage{array}
\usepackage{bbm}
\usepackage{makecell}

\usepackage{dashbox}
\usepackage{xcolor}
\usepackage{colortbl}
\definecolor{lightyellow}{rgb}{1.0, 0.95, 0.7}
\definecolor{Blue}{rgb}{0, 0, 0.8}
\definecolor{blue}{rgb}{0,0,1}
\definecolor{darkgreen}{rgb}{0,0.40,0}
\definecolor{firebrick}{rgb}{0.698,0.133,0.133}

\definecolor{colorA}{rgb}{1,0,0}
\definecolor{colorB}{rgb}{0,0.3,1}
\definecolor{colorC}{rgb}{0.9,0.8,0.2}
\definecolor{colorD}{rgb}{0,0.65,0}
\definecolor{lesslightgray}{rgb}{0.5,0.5,0.5}
\definecolor{light-gray}{gray}{0.95}

\let\tilde\widetilde
\let\hat\widehat

\newcommand{\Softmax}{{\rm{Softmax}}}

\def\P{\mathbb{P}}
\def\R{\mathbb{R}}

\def\p{\partial }
\let\cite\citep 

\usepackage{amsthm}
\usepackage[many]{tcolorbox}
\usepackage{etoolbox}
\BeforeBeginEnvironment{proof}{\par\vspace{-1\baselineskip}}
\AfterEndEnvironment  {proof}{\par\vspace{-1\baselineskip}}

\newtheoremstyle{theoremstyle}
  {.5\baselineskip} %
  {.5\baselineskip} %
  {}                  %
  {}                  %
  {\bfseries}        %
  {.}                 %
  {1em}               %
  {}                  %

\theoremstyle{theoremstyle}
\newtheorem{theorem}{Theorem}[section]
\newtheorem{lemma}{Lemma}[section]

\newtheorem{proposition}{Proposition}[section]
\newtheorem{definition}{Definition}[section]

\newtheorem{remark}{Remark}[section]

\tcolorboxenvironment{theorem}{
  breakable,
  colback=black!10,
  colframe=white,%
  width=\dimexpr\linewidth+10pt\relax,%
  enlarge left by=-5pt,%
  enlarge right by=-5pt,%
  boxsep=5pt,%
  boxrule=0pt,
  left=0pt,right=0pt,top=0pt,bottom=0pt,
  sharp corners,
  before skip=0.5\baselineskip, %
  after skip=0.5\baselineskip,  %
  fonttitle=\bfseries, %
  coltitle=black %
}
\tcolorboxenvironment{remark}{
  blanker,
  breakable,
  before skip=.8\baselineskip,  %
  after  skip=.8\baselineskip   %
}

\tcolorboxenvironment{proposition}{
  breakable,
  colback=black!10,
  colframe=white,%
  width=\dimexpr\linewidth+10pt\relax,%
  enlarge left by=-5pt,%
  enlarge right by=-5pt,%
  boxsep=5pt,%
  boxrule=0pt,
  left=0pt,right=0pt,top=0pt,bottom=0pt,
  sharp corners,
  before skip=0.5\baselineskip, %
  after skip=0.5\baselineskip,  %
  fonttitle=\bfseries, %
  coltitle=black %
}

\tcolorboxenvironment{lemma}{
  breakable,
  colback=black!10,
  colframe=white,%
  width=\dimexpr\linewidth+10pt\relax,%
  enlarge left by=-5pt,%
  enlarge right by=-5pt,%
  boxsep=5pt,%
  boxrule=0pt,
  left=0pt,right=0pt,top=0pt,bottom=0pt,
  sharp corners,
  before skip=0.5\baselineskip, %
  after skip=0.5\baselineskip,  %
  fonttitle=\bfseries, %
  coltitle=black %
}

\tcolorboxenvironment{corollary}{
  breakable,
  colback=black!10,
  colframe=white,%
  width=\dimexpr\linewidth+10pt\relax,%
  enlarge left by=-5pt,%
  enlarge right by=-5pt,%
  boxsep=5pt,%
  boxrule=0pt,
  left=0pt,right=0pt,top=0pt,bottom=0pt,
  sharp corners,
  before skip=0.5\baselineskip, %
  after skip=0.5\baselineskip,  %
  fonttitle=\bfseries, %
  coltitle=black %
}

\tcolorboxenvironment{definition}{
  breakable,
  colback=black!10,
  colframe=white,%
  width=\dimexpr\linewidth+10pt\relax,%
  enlarge left by=-5pt,%
  enlarge right by=-5pt,%
  boxsep=5pt,%
  boxrule=0pt,
  left=0pt,right=0pt,top=0pt,bottom=0pt,
  sharp corners,
  before skip=0.5\baselineskip, %
  after skip=0.5\baselineskip,  %
  fonttitle=\bfseries, %
  coltitle=black %
}
\tcolorboxenvironment{assumption}{
  breakable,
  colback=black!10,
  colframe=white,%
  width=\dimexpr\linewidth+10pt\relax,%
  enlarge left by=-5pt,%
  enlarge right by=-5pt,%
  boxsep=5pt,%
  boxrule=0pt,
  left=0pt,right=0pt,top=0pt,bottom=0pt,
  sharp corners,
  before skip=0.5\baselineskip, %
  after skip=0.5\baselineskip,  %
  fonttitle=\bfseries, %
  coltitle=black %
}

\crefname{theorem}{Theorem}{Theorems}
\crefname{proposition}{Proposition}{Propositions}
\crefname{lemma}{Lemma}{Lemmas}
\crefname{corollary}{Corollary}{Corollaries}
\crefname{definition}{Definition}{Definitions}
\crefname{assumption}{Assumption}{Assumptions}
\crefname{remark}{Remark}{Remarks}
\crefname{problem}{Problem}{Problems}
\crefname{property}{Property}{property}

\tcolorboxenvironment{hypothesis}{
  breakable,
  colback=black!10,
  colframe=white,%
  width=\dimexpr\linewidth+10pt\relax,%
  enlarge left by=-5pt,%
  enlarge right by=-5pt,%
  boxsep=5pt,%
  boxrule=0pt,
  left=0pt,right=0pt,top=0pt,bottom=0pt,
  sharp corners,
  before skip=0.5\baselineskip, %
  after skip=0.5\baselineskip,  %
  fonttitle=\bfseries, %
  coltitle=black %
}
\crefname{hypothesis}{Hypothesis}{Hypothesises}

\newtheorem{fact}{Fact}
\tcolorboxenvironment{fact}{
  breakable,
  colback=black!10,
  colframe=white,%
  width=\dimexpr\linewidth+10pt\relax,%
  enlarge left by=-5pt,%
  enlarge right by=-5pt,%
  boxsep=5pt,%
  boxrule=0pt,
  left=0pt,right=0pt,top=0pt,bottom=0pt,
  sharp corners,
  before skip=0.5\baselineskip, %
  after skip=0.5\baselineskip,  %
  fonttitle=\bfseries, %
  coltitle=black %
}
\crefname{fact}{Fact}{Facts}

\tcolorboxenvironment{example}{
  breakable,
  colback=black!10,
  colframe=white,%
  width=\dimexpr\linewidth+10pt\relax,%
  enlarge left by=-5pt,%
  enlarge right by=-5pt,%
  boxsep=5pt,%
  boxrule=0pt,
  left=0pt,right=0pt,top=0pt,bottom=0pt,
  sharp corners,
  before skip=0.5\baselineskip, %
  after skip=0.5\baselineskip,  %
  fonttitle=\bfseries, %
  coltitle=black %
}
\crefname{example}{Example}{Examples}

\tcolorboxenvironment{question}{
  breakable,
  colback=black!10,
  colframe=white,%
  width=\dimexpr\linewidth+10pt\relax,%
  enlarge left by=-5pt,%
  enlarge right by=-5pt,%
  boxsep=5pt,%
  boxrule=0pt,
  left=0pt,right=0pt,top=0pt,bottom=0pt,
  sharp corners,
  before skip=0.5\baselineskip, %
  after skip=0.5\baselineskip,  %
  fonttitle=\bfseries, %
  coltitle=black %
}

\crefname{question}{Question}{Questions}
\numberwithin{equation}{section}
\numberwithin{theorem}{section}
\numberwithin{proposition}{section}
\numberwithin{definition}{section}
\numberwithin{lemma}{section}
\numberwithin{assumption}{section}
\numberwithin{remark}{section}

\usepackage{lipsum}

\newcommand*{\annot}[1]{\tag*{\footnotesize{\textcolor{black!50}{\big(#1\big)}}}}

\makeatletter
\let\save@mathaccent\mathaccent
\newcommand*\if@single[3]{%
    \setbox0\hbox{${\mathaccent"0362{#1}}^H$}%
    \setbox2\hbox{${\mathaccent"0362{\kern0pt#1}}^H$}%
    \ifdim\ht0=\ht2 #3\else #2\fi
}
\newcommand*\rel@kern[1]{\kern#1\dimexpr\macc@kerna}
\newcommand*\widebar[1]{\@ifnextchar^{{\wide@bar{#1}{0}}}{\wide@bar{#1}{1}}}
\newcommand*\wide@bar[2]{\if@single{#1}{\wide@bar@{#1}{#2}{1}}{\wide@bar@{#1}{#2}{2}}}
\newcommand*\wide@bar@[3]{%
    \begingroup
    \def\mathaccent##1##2{%
        \let\mathaccent\save@mathaccent
        \if#32 \let\macc@nucleus\first@char \fi
        \setbox\z@\hbox{$\macc@style{\macc@nucleus}_{}$}%
        \setbox\tw@\hbox{$\macc@style{\macc@nucleus}{}_{}$}%
        \dimen@\wd\tw@
        \advance\dimen@-\wd\z@
        \divide\dimen@ 3
        \@tempdima\wd\tw@
        \advance\@tempdima-\scriptspace
        \divide\@tempdima 10
        \advance\dimen@-\@tempdima
        \ifdim\dimen@>\z@ \dimen@0pt\fi
        \rel@kern{0.6}\kern-\dimen@
        \if#31
        \overline{\rel@kern{-0.6}\kern\dimen@\macc@nucleus\rel@kern{0.4}\kern\dimen@}%
        \advance\dimen@0.4\dimexpr\macc@kerna
        \let\final@kern#2%
        \ifdim\dimen@<\z@ \let\final@kern1\fi
        \if\final@kern1 \kern-\dimen@\fi
        \else
        \overline{\rel@kern{-0.6}\kern\dimen@#1}%
        \fi
    }%
    \macc@depth\@ne
    \let\math@bgroup\@empty \let\math@egroup\macc@set@skewchar
    \mathsurround\z@ \frozen@everymath{\mathgroup\macc@group\relax}%
    \macc@set@skewchar\relax
    \let\mathaccentV\macc@nested@a
    \if#31
    \macc@nested@a\relax111{#1}%
    \else
    \def\gobble@till@marker##1\endmarker{}%
    \futurelet\first@char\gobble@till@marker#1\endmarker
    \ifcat\noexpand\first@char A\else
    \def\first@char{}%
    \fi
    \macc@nested@a\relax111{\first@char}%
    \fi
    \endgroup
    }
\makeatother
\let\bar\widebar

\newcommand*{\email}[1]{\footnote{\href{mailto:#1}{\texttt{#1}}}}

\setlist[itemize,enumerate]{
  parsep=\parskip,                                   %
  itemsep=\dimexpr .3em - \parskip\relax plus 2pt,   %
  topsep=\dimexpr 6pt - \parskip\relax plus 1pt minus 1pt,
  partopsep=0pt,
  listparindent=\parindent
}

\begin{document}
\begin{titlepage}

\begin{flushright}
Last Update: \today
\end{flushright}

\vskip 2.5em
\begin{center}

{
\LARGE \bfseries %
\begin{spacing}{1.15} %
On the Statistical Mechanics of Attention
\end{spacing}
}

\vskip 1em
Jerry Yao-Chieh Hu$^{\dagger\ddag *}$\email{jhu@u.northwestern.edu}
\quad
Yu-An Chen$^{\S*}$\email{b09202059@ntu.edu.tw}
\quad
Maojiang Su$^{\dagger}$\email{smj@u.northwestern.edu}
\quad
Han Liu$^{\dagger\natural}$\email{hanliu@northwestern.edu}

\def\thefootnote{*}
\footnotetext{These authors contributed equally to this work.
}

\vskip 1em

{\small
\begin{tabular}{ll}
 $^\dagger\;$Center for Foundation Models and Generative AI, Northwestern University\\
 \hphantom{$^\ddag\;$}Department of Computer Science, Northwestern University\\

 $^\ddag\;$Department of Computer Science, College of Connected Computing, Vanderbilt University\\
 $^\S\;$Department of Physics, National Taiwan University\\
 $^\natural\;$Department of Statistics and Data Science, Northwestern University
\end{tabular}}

\end{center}

\noindent
We revisit the statistical mechanics of Transformer attention via modern Hopfield models. 
We present a free-energy formalism to characterize parameter regimes and critical thresholds for robust attention inference. 
The resulting free-energy landscape incorporates randomness in query states and attention weights. 
Attention then emerges as the equation of motion of the corresponding statistical-mechanical system (i.e., modern Hopfield model) through the standard free-energy minimization principle. 
This allows us to identify first-order phase transitions in attention between glassy and robust-inference phases. 
We further extend to attentions with stochastic queries and keys. 
Together, our results provide a rigorous statistical-mechanics framework for characterizing phase structure and inference reliability across a family of Transformer attention mechanisms.

\end{titlepage}

{
\setlength{\parskip}{0em}
\setcounter{tocdepth}{2}
\tableofcontents
}
\setcounter{footnote}{0}

\section{Introduction}
\label{sec:intro}

We study the statistical-mechanics of Transformer attention through modern Hopfield models and characterize parameter regimes and critical threshold for robust attention inference.
Our key contribution is a statistical-mechanical formulation in which Transformer attention emerges from the standard free-energy minimization principle naturally.
Through this perspective, 
we give a clear characterization of first-order phase transitions of attention inference between glassy and robust-inference phases.
Consequently, our framework provides a rigorous way to characterize reliable inference of Transformer attention, and further allows extension to the full attention family.

During a Transformer forward pass, an input query selects and combines information from the entire context.
Attention mechanism captures the core function of this operation.
Concretely, for a query $q\in\R^{d_k}$ and stored key-value pairs $\{k_\mu,v_\mu\}_{\mu=1}^M$,  we denote $K$ and $V$ as the matrices with rows $k_\mu\in\R^{d_k}$ and $v_\mu\in\R^{d_v}$, respectively.
Then attention \cite{vaswani2017attention} computes
\begin{align*}
    {\rm Attn}(q,K,V)=\Softmax\Big(\beta\frac{qK^\top}{\sqrt{d_k}}\Big)V
    = \sum_{\mu=1}^M
    (\frac{e^{\beta z_\mu}}
    {\sum_{\nu=1}^M e^{\beta z_\nu}})v_\mu,
    \quad\text{with}\quad
    z_\mu := \frac{q k_\mu^\top}{\sqrt{d_k}}.
\end{align*}
Here $\beta$ is the inverse-temperature.
${\rm Attn}(q,K,V)$ specifies which key $k_\mu$ does the input query $q$ retrieve and contribute to the output via the value projection $V$. 

Recent works link this \textit{inference} of Transformer attention to the memory \textit{retrieval} of modern Hopfield models \cite{ramsauer2020hopfield,krotov2020large,hu2023sparse}, i.e., selecting or combining stored patterns based on the input query.
This connection is decomposable into two aspects:
\begin{itemize}
    \item First, \citet{hu2023sparse,wu2023stanhop} bridge between modern Hopfield models and attention family including softmax, sparsemax, and beyond \cite{martins2016softmax}.

    \item On the other hand, \citet{hu2023sparse} reveal that softmax solves an entropy-regularized optimization problem via variation \cite{peters2019sparse};
see \cite{santos2025hopfield} for a review.
\end{itemize}
Together, they view Transformer attentions as physical models and lead to useful methods \cite{widrich2020modern,hoover2023energy,auer2023conformal,hu2024outlier,wu2024uniform,dehmamy2025nrgpt}.
However, the prior work only investigate the retrieval rule itself. 
The parameter regimes for robust inference in these models remains unclear.
This raises a natural question in physics, with beneficial implications for Transformer inference.

In statistical mechanics, phase structure characterizes how a system's behavior changes as its parameters vary.
\citet{amit1985,AMIT198730,Nishimori2001} analyze the phase structure of classical Hopfield models \cite{classicalhopfield} via mean-field replica method.
However, they consider i.i.d. and infinite-dimension assumptions.
That is, how finite size and realistic structured patterns affect retrieval behavior remains open.
Modern Hopfield models offer a natural setting for extending this study of phase transitions: connecting the physics of associative memory to Transformer attention.
Yet their phase structure and the conditions governing these transitions remain poorly understood.

To address those issues, we utilize the free-energy minimization principle from statistical mechanics \cite{Landau:1980mil,binder1987,Zinn-Justinqft} to study phase structures of modern Hopfield models and hence of Transformer attention.
Within this framework, different entropy induces different %
sparsity and hence different attention mechanism \cite{martins2016softmax,peters2019sparse,ramsauer2020hopfield}.
Through this perspective, we explore first-order phase transitions and hence abrupt reorganization of attention weight beyond classical Hopfield models \cite{classicalhopfield,mezard_parisi1986}.
This leads to an %
parameter-dependent mechanism %
behind modern Hopfield models and hence the phase structure of Transformer attention.

\textbf{Contributions.} 
We develop a statistical-mechanical framework for Transformer attention and use it to understand its inference behavior. 
Our main contributions are:

\begin{itemize}

    \item \textbf{A Statistical-Mechanical Theory of Transformer Attention.}
    We formulate Transformer attention as a statistical-mechanical model of associative memory through the free energy of modern Hopfield networks.
    In this formulation, attention emerges naturally from the equations of motion (\cref{prp:eom_smhm,prp:eom_hopfield}) and its forward pass describes how a query retrieves stored information.
    This framework allows the unification of attention's temperature, entropy, sparsity,  fluctuations, and retrieval dynamics into one coherent physical picture.

    \item \textbf{Phase Transitions in Attention Inference.}
    We use the resulting free-energy landscape to study the reliability of information retrieval during a Transformer attention forward pass.
    Our analysis reveals two retrieval regimes: a glassy-like regime with competing metastable states and a robust-inference regime with stable, target-aligned retrieval.
    We characterize first-order phase transitions\footnote{A first-order phase transition occurs when the equilibrium
    state changes discontinuously as a control parameter varies.
    In our case, varying attention temperature can induce such
    transitions in memory retrieval.
    Memory load and key correlations further shape the phase structure.}
    between these regimes as attention parameters vary.  

    \item \textbf{Critical Conditions for Robust Attention Inference.}
    We further analyze what governs the stability of memory selection in attention inference.
    We characterize critical conditions for
    the coexistence of competing states (\cref{prp:coex_point}) and their spectral stability (\cref{prp:spinodal_point}) from correlations among learned keys.
    These conditions explain how learned key correlations and attention temperature determine whether target-memory selection remains stable or becomes vulnerable to interference from competing keys.

\end{itemize}

\textbf{Organization.}
\cref{sec:background} presents preliminaries. 
\cref{sec:stas_att_hopfield} presents physical formulation of attention through modern Hopfield models.
\cref{sec:retrieval_success} presents our phase transition results. 
\cref{sec:numerical_results} presents numerical validations.
\cref{sec:conclusion} presents concluding remarks and practical implications.
We defer related work discussion and statistical-mechanics tutorial in \cref{sec:related_works,sec:stat_mech_tutorial} due to space limit.
\cref{sec:fluctuated_version} extends to the stochastic version, and \cref{sec:dynamical_stability} analyzes the dynamical stability.
\cref{sec:numerical_studies,sec:add_exp} show experiment setup and additional experiments.

\begin{table}[t]
\centering
\caption{
Roadmap from the technical terms to their role in
attention.
}
\label{tab:framework-roadmap}
\small
\setlength{\tabcolsep}{4pt}
\renewcommand{\arraystretch}{1.15}
\resizebox{\textwidth}{!}{%
\begin{tabularx}{\linewidth}{
    >{\raggedright\arraybackslash}p{0.23\linewidth}
    >{\raggedright\arraybackslash}p{0.34\linewidth}
    >{\raggedright\arraybackslash}X
}
\toprule
\textbf{Concept}
&
\textbf{Meaning for Attention}
&
\textbf{Dictionary}
\\
\midrule

Fenchel Conjugate
&
It transforms a microscopic quantity into a macroscopic one, and constructs the free-energy objective.
&

\cref{def:fenchel_conjugation,prp:fenchel_young_duality_relation,thm:free_energy_tsallis_entropy}
\\

\addlinespace

Tsallis Entropy $\Psi_\alpha$
&
Attention Family:
softmax $(\alpha=1)$, entmax $(1<\alpha<2)$, and sparsemax
$(\alpha=2)$.
&
\cref{sec:tsallis_entropy,thm:free_energy_tsallis_entropy,prp:eom_attention}
\\

\addlinespace

Query-Key Similarity (Overlap $z=\frac{1}{\sqrt{d}}\xi^\top x$)
&
It measures the alignment between queries and memories.
A larger overpal indicates a stronger retrieval preference for that memory.
&
\cref{def:map_deterministic_variable}
\\

\addlinespace

Free-energy $\Gamma_\alpha$
&
It defines free-energy landscape and enables assigning energy value for each local minimum.
It encodes equation of motion as attention rule.
&
\cref{def:partition_function}, \cref{prp:eom_smhm,prp:eom_attention}
\\

\addlinespace

Coexistence (Energy Crossing)
&
It occurs when two distinct stable retrieval states have the same free energy $\Gamma_\alpha$ simultaneously.
&
\cref{def:coexist,prp:coex_point}
\\

\addlinespace

Spinodal
&
It occurs when the energy barrier of a local mimimum disappers.
This drives a retrieval state to the other.
&
\cref{def:spinodal,prp:spinodal_point}
\\

\addlinespace

Exponential Memory Load $\kappa=\log(M)/d$
&
It quantifies the growth rate of the number of memories.
This affects retrieval interference and reliability.
&
\cref{sec:memory_load_dependence}
\\

\bottomrule
\end{tabularx}
}
\end{table}

\section{Preliminaries}
\label{sec:background}
We first discuss convex analysis and geometric properties in statistical mechanics.
This induces entropy and sparsity via the free energy framework.
The purpose of this section is to identify the conjugate variables that characterize the retrieval dynamics and attention mechanism.

\subsection{Statistical Mechanics and Energy Formalism}
\label{sec:thermo_energy_formalism}
We introduce partition function characterizing a system with all possible configurations, and adopt it to discuss free energy through the statistical mechanics \cite{Zinn-Justinqft,Srednicki:2007qs}.

\begin{definition}[Partition Function]
\label{def:partition_function}
    Let $\mathcal{H}:\mathcal{X}\to\R$ be a given function called an energy function.
    We define the partition function as the Gibbs average over all configurations. 
    \begin{align*}
        Z[J]\coloneq\int_{\mathcal{X}}\exp\{-\beta\mathcal{H}[x]+\sum_{i=1}^dJ_ix_i\}~\dd\mu(x),
    \end{align*}
    where $\beta\coloneq1/T$ is the inverse temperature, $J\in\R^d$ is the generating source, and $\dd\mu(x)$ denotes the Lebesgue measure in $\mathcal{X}$.
    The partition function is positive.
    For simplicity, we consider $\mathcal{X}=\R^d$.
\end{definition}

\begin{remark}
    The generating source term represents the external magnetic field in a spin system.
    Furthermore, if the model is discrete, then we replace $\int_{\mathcal{X}}~\dd\mu(x)$ by $\sum_{x\in\mathcal{X}}$ with $\mathcal{X}=\{1,...,K\}$.
\end{remark}

\begin{definition}[Stationary Equations and Stationary Points]
\label{def:stationary_eq}
    The stationary equations are 
    \begin{align}
    \label{eq:stationary_eq}
        \nabla_{x}\{\beta\mathcal{H}[x]-\sum_{i=1}^dJ_ix_i\}|_{x=x_J}=0~,~\nabla^2_{x}\{\beta\mathcal{H}[x]-\sum_{i=1}^dJ_ix_i\}|_{x=x_J}\succ0~,
    \end{align}
    where $\succ$ means positive definite.
    We say that the solution $x_J$ is a stationary or (local) stable point.
\end{definition}

We obtain stable points in \cref{def:stationary_eq} through the energy function $\mathcal{H}[x]$.
However, it contains microscopic noise components.
Hence, $x_J$ is not a well-behaved stable point.
To extract the aggregated stability from the original energy function $\mathcal{H}[x]$, we should introduce the free energy (effective action) $\Gamma[\varphi]$. It later helps characterize the macroscopic averaged energy of the system.

\begin{definition}[Log-Partition Function and Free Energy]
\label{def:log_partition_effective_action}
    We have two different free energies that describe the partition function.
    To distinguish them, we define them carefully.
    \begin{itemize}
        \item \textbf{Log-Partition Function}: It is defined as 
        \begin{align}
        \label{eq:log_partition_function}
            W[J]\coloneq \log Z[J]=\log(\int_{\mathcal{X}}\exp\{-\beta\mathcal{H}[x]+\sum_{i=1}^dJ_ix_i\}~\dd\mu(x)).
        \end{align}
        
        \item \textbf{Free Energy (Effective Action)}: It is the Fenchel conjugate of the log-partition function.
        \begin{align}
        \label{eq:effective_action}
            \Gamma[\varphi]\coloneq\sup_{J\in\R^d}\{\sum_{i=1}^dJ_i\varphi_i-W[J]\},\qquad\text{with }~\varphi\in\R^d.
        \end{align}
        
    \end{itemize}
    
\end{definition}

\begin{remark}
\label{rmk:non_bijective_energy}
    From \cref{sec:basic_tool}, $W[J]$ and $\Gamma[\varphi]$ are both convex, $W[J]$ is smooth, $\nabla_J W[J]$ is single-valued, and $\Gamma[\varphi]$ is closed. 
    However, we do not assume strict convexity. 
    Consequently, the gradient maps $\nabla_JW[J]$ and $\nabla_\varphi\Gamma[\varphi]$ are not guaranteed to be bijective according to \cref{prp:properties_fenchel_conjugates}.
    In this way, distinct microscopic configurations are able to represent the same macroscopic state.
\end{remark}

The optimization problem becomes minimizing the free energy $\Gamma[\varphi]$ with respect to $\varphi$.
Next, we describe the physical interpretation/derivation of the Fenchel-Young duality.
\begin{proposition}[Fenchel-Young Duality Relation \cite{Srednicki:2007qs,blondel2020learning}]
\label{prp:fenchel_young_duality_relation}
    Let $p_J$ be the Gibbs measure with external source $J$.
    From the standard statistical field theory, the induced macroscopic field $\varphi=\mathbb{E}_{p_J}[x]$ reveals the following source-field relation:
    \begin{align*}
        \varphi=\mathbb{E}_{x\sim p_J}[x]\in\partial W[J]~\Leftrightarrow~J\in\partial\Gamma[\varphi]~,
    \end{align*}
    where $\partial$ is the subdifferential from \cref{def:subgradient}.
    If $W$ and $\Gamma$ are differentiable, then $\partial$ reduce to gradients $\nabla$.
    We call them equations of motion (EOMs) in the physics literature. 
    It also refers to the Fenchel-Young duality relation. 
    For $\varphi$ and $J$, we call them the classical solutions or saddle points.
\end{proposition}

\begin{proof}
    Please see \cref{sec:proof_fenchel_young} for a detailed proof.
\end{proof} 

From \cref{prp:fenchel_young_duality_relation}, the Fenchel conjugate variable $\varphi$ is macroscopic $\mathbb{E}_{x\sim p_J}[x]$ exactly.
$\varphi$ eliminates fluctuations and induces deterministic landscape $\Gamma[\varphi]$.
Please see \cref{fig:microscopic-to-effective-action} for illustration.

\begingroup
\setlength{\textfloatsep}{0em}
\setlength{\floatsep}{0em}
\begin{figure}[t]
    \centering
    \resizebox{\columnwidth}{!}{%
    \begin{tikzpicture}[scale=1, every node/.style={font=\small}]

    \begin{scope}[shift={(0,0)}]
        \draw[rounded corners=4pt, thick] (-0.5,-2.5) rectangle (4.8,2.8);
    
        \node[font=\bfseries] at (2.15,2.35) {Microscopic Configuration};
    
        \draw[->] (0,-1.8) -- (0,1.8) node[right] {$\beta \mathcal{H}[x]-Jx$};
        \draw[->] (0,-1.8) -- (4.2,-1.8) node[below] {$x$};

        \draw[thick, domain=0.5:3.8, smooth, variable=\x]
        plot ({\x},{0.85*(\x-2.1)^2 - 1.2});

        \fill (2.1,-1.2) circle (1.4pt);
        \node[below] at (2.1,-1.25) {$x_J$};

        \draw[dashed] (2.1,-0.55) ellipse (1.1 and 0.6);
        \fill[opacity=0.18] (2.1,-0.55) ellipse (1.1 and 0.6);
        \node at (2.1,0.35) {fluctuations};

        \fill (1.55,-0.94) circle (1.2pt);
        \fill (1.35,-0.72) circle (1.2pt);
        \fill (2.55,-1.03) circle (1.2pt);
        \fill (2.85,-0.72) circle (1.2pt);
        \fill (3.05,-0.43) circle (1.2pt);

        \node[align=center] at (2.15,1.3)
        {$p_J(x)\propto e^{-\beta \mathcal{H}[x]+Jx}$};

    \end{scope}

    \draw[->, thick] (5,0.1) -- (5.5,0.1);

    \begin{scope}[shift={(6.5,0)}]
        \draw[rounded corners=4pt, thick] (-0.5,-2.5) rectangle (4.8,2.8);

        \node[font=\bfseries] at (2.15,2.35) {Gibbs Average};

        \draw[->] (0,-2.0) -- (0,1.8) node[right] {$W[J]$};
        \draw[->] (0,-2.0) -- (4.2,-2.0) node[below] {$J$};

        \draw[thick, domain=0.6:3.7, smooth, variable=\x]
        plot ({\x},{0.23*(\x-0.8)^2 - 1.65});

        \fill (2.4,-1.06) circle (1.3pt);

        \draw[dashed] (1.35,-1.83) -- (3.45,-0.287);
        \node[rotate=22] at (3.35,-1) {\small slope $=\varphi$};

        \node[align=center] at (2.35,1.1)
        {$\varphi=\mathbb{E}_{x\sim p_J}[x]=\nabla_J W[J]$};
        
    \end{scope}

    \draw[->, thick] (11.5,0.1) -- (12,0.1);

    \begin{scope}[shift={(13.0,0)}]
        \draw[rounded corners=4pt, thick] (-0.5,-2.5) rectangle (4.8,2.8);

        \node[font=\bfseries] at (2.15,2.35) {Fenchel Conjugate};

        \draw[->] (0,-2.0) -- (0,1.8) node[right] {$\Gamma[\varphi]$};
        \draw[->] (0,-2.0) -- (4.2,-2.0) node[below] {$\varphi$};

        \draw[thick, domain=0.5:3.8, smooth, variable=\x]
        plot ({\x},{0.55*(\x-2.15)^2 - 1.25});

        \fill (2.15,-1.25) circle (1.4pt);
        \node[below] at (2.15,-1.3) {$\varphi^\star$};

        \node[align=center] at (2.15,1.15)
        {$J=\nabla_\varphi \Gamma[\varphi]$};

\end{scope}

\end{tikzpicture}%
}
\caption{
\textbf{Illustration of the procedure from the microscopic to macroscopic energy landscapes.}
{\small \textbf{Left.} ``(Microscopic) Fluctuations'' means $x$ is random, has nonzero variance, and follows Gibbs distribution $p_J(x)$.
Built from microscopic $x$, the local stable point $x_J$ exhibits randomness.  
\textbf{Middle.} $W[J]=\log(\int_\mathcal{X}\exp\{-\beta\mathcal{H}[x]+Jx\})$ encodes macroscopic responses, i.e., 
$\varphi=\mathbb{E}_{x\sim p_J}[x]=\nabla_J W[J]$.
Hence, the slope of $W[J]$ gives a deterministic average.
\textbf{Right.} $\Gamma[\varphi]=\sup_J\{J\varphi-W[J]\}$ is the Fenchel conjugate of $W[J]$. 
Its response is $J=\nabla_\varphi\Gamma[\varphi]$, so local stable macroscopic points occurs at $J=0$.
$\Gamma[\varphi]$ describes the energy landscape built from macroscopic $\varphi$ after integrating out microscopic fluctuations.
}}
\label{fig:microscopic-to-effective-action}
\end{figure}
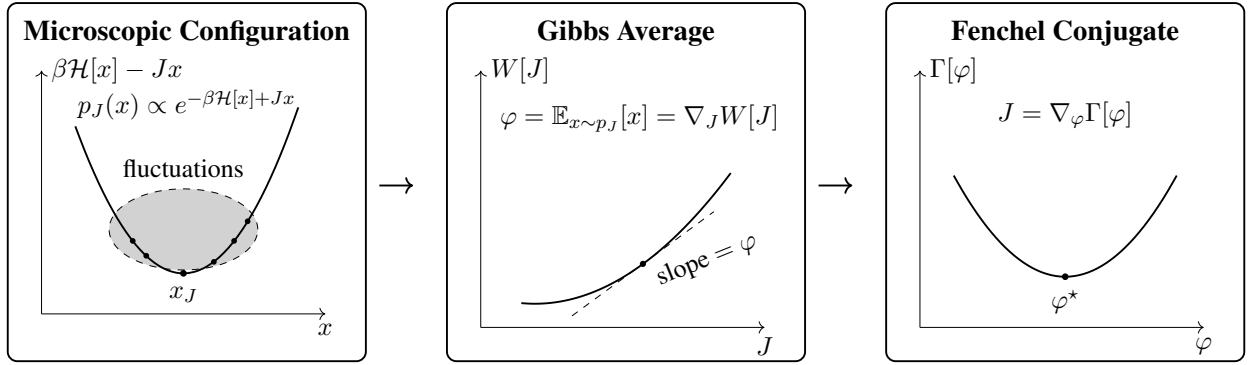
\endgroup

\subsection{Connection between Free Energy and Tsallis Entropy}
\label{sec:entropy_function}

In this section, we connect our free energy to Tsallis entropy.
The value of $\alpha$
determines the sparsity structure of the retrieval dynamics.
As we shall see in \cref{sec:stas_att_hopfield}, this choice covers the general attention mechanisms, ranging from dense ($\alpha=1$) to the sparse case ($\infty>\alpha>1$).

For preliminaries on Tsallis entropy, $\alpha$-deformed exponential and logarithm family, please see \cref{sec:tsallis_entropy} for completeness.
They generalize free energy to the $\alpha$-deformed version.

\begin{theorem}[Tsallis Entropy and Deformed Free Energy]
\label{thm:free_energy_tsallis_entropy}
    If we replace the exponential functions and logarithm functions by $\alpha$-deformed functions, then we have the deformed log-partition function
    \begin{align*}
        -W_\alpha[J]=\beta\mathbb{E}^{(\alpha)}_{x\sim p_J}[\mathcal{H}[x]]-\frac{\Psi_\alpha[p_J]}{Z_\alpha^{\alpha-1}[J]\int_{\mathcal{X}}p_J^\alpha(x)~\dd\mu(x)}-\sum_{i=1}^dJ_i\mathbb{E}^{(\alpha)}_{x\sim p_J}[x_i],
    \end{align*}
    where $\mathbb{E}^{(\alpha)}_{x\sim p_J}[\cdot]\coloneq\int_{\mathcal{X}} p^\alpha_J(x)[\cdot]\dd \mu(x)/\int_{\mathcal{X}} p^\alpha_J(x)\dd \mu(x)$ is the escort expectation value, and
    \begin{align}
    \label{eq:free_energy_tsallis_entropy}
        \Gamma_\alpha=\beta\mathbb{E}^{(\alpha)}_{x\sim p_{J^\star}}[\mathcal{H}[x]]-\frac{\Psi_\alpha[p_{J^\star}]}{Z_\alpha^{\alpha-1}[J^\star]\int_{\mathcal{X}}p_{J^\star}^\alpha(x)~\dd\mu(x)},\qquad\text{with }~\varphi=\mathbb{E}^{(\alpha)}_{x\sim p_{J^\star}}[x],
    \end{align}
    where $J^\star=\arg\max_J\{\sum_{i=1}^dJ_i\varphi_i-W_\alpha[J]\}$.
    We assume $\mathcal{X}$ is finite-volume to avoid divergence.

\end{theorem}

\begin{proof}
    Please see \cref{sec:proof_free_energy_tsallis_entropy} for a detailed proof.
\end{proof}

\begin{remark}
    If $\alpha=1$, an escort expectation value becomes an ordinary expectation value.
    Hence, if we further consider the zero-temperature limit $\beta\to\infty$, then minimizing the free energy $\Gamma_\alpha$ is equivalent to minimizing the energy function $\mathcal{H}$.
    In the high-temperature limit, minimizing $\Gamma_\alpha$ is equivalent to maximizing the Tsallis entropy $\Psi_\alpha$.
    The latter case indicates a disorder regime.
\end{remark}

\begin{remark}
    \cref{thm:free_energy_tsallis_entropy} relates to the evidence lower bound (ELBO) in variational autoencoders (VAE), \cite{kingma2013auto}.
    Specifically, with a posterior $q_\phi(z|x)$, we have %
    \begin{align*}
        \log p_\theta(x)\geq\mathbb{E}_{z\sim q_\phi(z|x)}[\log p_\theta(x,z)]-\mathbb{E}_{z\sim q_\phi(z|x)}[\log q_\phi(z|x)].
    \end{align*}
    We recognize that the second term on the right hand side is Shannon entropy, \eqref{eq:shannon_entropy}.
    Then,
    \begin{align*}
        \text{ELBO}\coloneq\mathbb{E}_{z\sim q_\phi(z|x)}[\log p_\theta(x,z)]+S[q_\phi(\cdot|x)].
    \end{align*}
    It is an energy-entropy decomposition, where the expected log joint probability corresponds to a negative
    energy. 
    Hence, maximizing ELBO is equivalent to minimizing a free energy \eqref{eq:free_energy_tsallis_entropy}. 
\end{remark}

\section{Statistical Physics of Attention: Modern Hopfield Models}
\label{sec:stas_att_hopfield}
In this section, we define an energy function of MHMs in \cite{hu2023sparse,wu2023stanhop}.
Then we express the free energy in terms of entropy and derive the emergent attention mechanism.

\subsection{Modern Hopfield Model (MHM)}
Inspired by the optimization objectives in \cite{ramsauer2020hopfield,krotov2020large,hu2023sparse,wu2023stanhop}, we introduce the explicit joint energy function of MHM as follows.

\begin{definition}[Modern Hopfield Model (MHM)]
\label{def:mhm_hamiltonian}
    The joint energy function of the MHM is
    \begin{align}
    \label{eq:mhm_joint_hamiltonian}
        \tilde{\mathcal{H}}_{\text{MHM}}[x,p,\xi]\coloneq\frac{1}{2}\sum_{i=1}^d x_i^2-\sum_{\mu=1}^Mz_\mu p_\mu,
    \end{align}
    where $p\in\Delta^M$ lives on a probability simplex, i.e., $p$ is non-negative and satisfies $\sum_{\mu=1}^Mp_\mu=1$, and $z_\mu\coloneq\sum_{i=1}^d\xi_i^\mu x_i/\sqrt{d}$.
    Moreover, we denote $\mathcal{H}_{\text{MHM}}[x,\xi]\coloneq\inf_{p\in\Delta^M}\tilde{\mathcal{H}}_{\text{MHM}}[x,p,\xi]$.
\end{definition}

In our setting, query states, the probability distribution, and the stored patterns are all fixed, i.e., the distributions of $x$ and $p$ collapse to their escort average.
Specifically, a Taylor expansion of $f(x)$ gives $\mathbb{E}^{(\alpha)}_{(x,p)\sim \rho_J}[f(x)]=f(\varphi)+O(\text{Var}^{(\alpha)}(x))$, where $\rho_J(x,p)$ is the joint probability distribution, 
$\varphi=\mathbb{E}^{(\alpha)}_{(x,p)\sim\rho_J}[x]$ is the expectation value, and $\text{Var}^{(\alpha)}$ is the variance.
Then the fixed-parameter setting corresponds to the zero variance.
Hence, we are able to evaluate all quantities $f(x)$ at $\varphi$ under this fixed-parameter setting exactly.
Thus, we identify $\varphi$ with $x$, and similarly for $p$ and $z$.
\footnote{
Here we restrict attention to the deterministic regime.
This allows us to isolate the physics most directly relevant to the attention mechanism.
In \cref{sec:fluctuated_version}, we consider the more general setting in which both $x$ and $p$ are random variables with nonzero variance.
This leads to a richer and more complete statistical-mechanics structure, and allows us to analyze fluctuations and response properties.
}

\begin{definition}[Statistics-MHM Mapping]
\label{def:map_deterministic_variable}
    The mapping is:
    
    \begin{itemize}
        \item \textbf{Query State Vector.} $x\in\R^d$: It represents continuous query states. 
        Due to our fixed-parameter setting, it is macroscopic and parametrizes the free energy $\Gamma_\alpha[x]=\Gamma_\alpha[\varphi]$.
        
        \item \textbf{External Source.} $J\in\R^d$: It represents an external bias of the query states.
        In particular, $J=0$ identifies local minima of $\Gamma_\alpha[x]$ due to the Fenchel-Young relation (\cref{prp:fenchel_young_duality_relation}).
        
        \item \textbf{Probability Function.} $p\in\Delta^M$: It represents the retrieval distribution over memories.
        
        \item \textbf{Entropy.} $\Psi_\alpha$: It acts as a regularizer controlling how sparse the memory weights $p$ are. 
        
        \item \textbf{Score.} $z\in\R^M$: It represents query-key similarity between the current $x$ and each memory pattern $\xi^\mu$ via the relation: $z_\mu\coloneq\sum_{i=1}^d\xi_i^\mu x_i/\sqrt{d}$.
        Higher $z_\mu$ suggests higher similarity.
    \end{itemize}
    
\end{definition}

In fact, the query state vector $x$ in \cref{def:map_deterministic_variable} relates to the probability weight $p$ as follows. 

\begin{proposition}[Query-Probability Relation]
\label{prp:expectation_values_and_macroscopic_states}
    A query state relates to a probability function via
    \begin{align}
        x_i=\frac{1}{\sqrt{d}}\sum_{\mu=1}^M\xi_i^\mu p_\mu+\frac{1}{\beta}J_i.
        \label{eq:exptation_x}
    \end{align}
\end{proposition}

\begin{proof}
    Please see \cref{sec:proof_expectation_values_and_macroscopic_states} for a detailed proof.
\end{proof}

In \cref{sec:entropy_function}, we relate free energy to Tsallis entropy via \eqref{eq:free_energy_tsallis_entropy}.
Inspired by that relation, we are capable of expressing the free energy in terms of the variables of MHMs as follows.

\begin{theorem}[Free Energy of MHM]
\label{thm:free_energy_mhm}
    The joint free energy of the MHM is
    \begin{align}
    \label{eq:smhm_joint_effective_action}
        \tilde{\Gamma}_\alpha[x,p,\xi] = \frac{\beta}{2}\sum_{i=1}^dx_i^2-\beta\sum_{\mu=1}^Mz_\mu p_\mu.
    \end{align}
    Hence, after taking the infimum with respect to $p$, we obtain the free energy
    \begin{align}
    \label{eq:mhm_effective_action}
        \Gamma_\alpha[x,\xi]=\frac{\beta}{2}\sum_{i=1}^dx_i^2-\beta\max_{\mu\in[1,M]}~z_\mu.
    \end{align}
\end{theorem}

\begin{proof}
    Please see \cref{sec:proof_free_energy_mhm} for a detailed proof.
\end{proof}

This free energy selects the maximal component of the overlap $\max_\mu z_\mu$.
The hard selection mechanism arises from zero entropy due to the collapse of $p$.
This captures a supremum over simplex vertices rather than the full simplex $\Delta^M$.
The relaxation from a hard selection to a distribution over multiple memories becomes a natural extension.
Motivated by the sparsity of the entropy regularizer in \cite{hu2023sparse}, we consider the relaxation through Fenchel conjugate as follows. 

\begin{tcolorbox}[
  breakable,
  colback=black!10,
  colframe=white,
  width=\dimexpr\linewidth+10pt\relax,
  enlarge left by=-5pt,
  enlarge right by=-5pt,
  boxsep=5pt,
  boxrule=0pt,
  left=0pt,
  right=0pt,
  top=0pt,
  bottom=0pt,
  sharp corners,
  before skip=0.5\baselineskip,
  after skip=0.5\baselineskip
]
    We introduce the entropy regularizer  in \eqref{eq:mhm_effective_action}.
    This yields a generalized (relaxed) model:
    \begin{align}
    \label{eq:mod_eff_action}
        \Gamma^{\text{(mod)}}_\alpha[x,\xi]=\frac{\beta}{2}\sum_{i=1}^dx_i^2-\Omega_\alpha^\star[\beta z],\qquad\Omega_\alpha\coloneq-\Psi_\alpha,
    \end{align}
    where $\Omega_\alpha^\star[\beta z]=\sup_{p\in\Delta^M}\{\sum_{\mu=1}^M\beta z_\mu p_\mu-\Omega_\alpha[p]\}$ is the Fenchel conjugate of the negative Tsallis entropy \eqref{eq:fenchel_conjugate_tsallis_entropy}.
    It recovers the energy form in \cite{hu2023sparse,wu2023stanhop} exactly.
\end{tcolorbox}

\begin{remark}
\label{rmk:generalization}
    This generalization allows us to relax $p$ to have a distribution over multiple memories on the full simplex $\Delta^M$.
    This leads to several theoretical and practical advantages:
    \begin{enumerate}
        \item \textbf{Equilibrium Selection vs. Averaging.}
        The MHM performs a hard maximization, selecting a single dominant component of $\Gamma_\alpha^{(\text{mod})}[x,\xi]$.
        In contrast, expectation-based rules compute averages over configurations. 
        The generalization \eqref{eq:mod_eff_action} interpolates between these regimes, i.e., it reduces to the original MHM in $\beta\to\infty$ limit.
        Please see \cref{sec:fluctuated_version} for details.

        \item \textbf{Robustness to Noise.}
        As we will see in \cref{prp:eom_attention}, $\Gamma_\alpha^{(\text{mod})}$ emphasizes dominant components of $z$ for $\alpha>1$.
        Hence, it suppresses contributions from noisy patterns. 
        In contrast, expectation-based rules average all components, making them more sensitive to noise.
        
        \item \textbf{Sparsity.}
        The entropy function $\Psi_\alpha$ controls the support of sparsity. 
        In particular, for $\alpha>1$, the maximizer $p^\star$ is sparse, so some of the components are zero.
        This is also a consequence of robustness to noise, which does not appear in expectation-based formulations.

        \item \textbf{Computational Accessibility.}
        Fenchel conjugation leads to an explicit variational object and is more accessible with gradient-based methods.
        In contrast, expectation-based formulations require sampling or approximation of Gibbs averages, which is more demanding. 
    \end{enumerate}
    
\end{remark}

\subsection{Retrieval Dynamics and Attention Mechanism}
In \cref{sec:background}, we introduce EOMs from \cref{prp:fenchel_young_duality_relation}.
These equations describe the trajectory that the flow should follow.
In this section, we investigate the physical meaning of the trajectories in associative memory systems.
In particular, the attention emerges as an EOM naturally.

\begin{proposition}[EOMs of Free Energies in MHMs]
\label{prp:eom_smhm}
    Let $Z_\alpha$ be the partition function with $\alpha$-deformed Gibbs average.
    Then the EOM in the $J$-space corresponds to the Hopfield retrieval
    \begin{align}
    \label{eq:eom_w_x}
        \nabla_J W_\alpha[J,\xi]=x,\qquad\text{with }~W_\alpha[J,\xi]\coloneq\log_\alpha Z_\alpha[J,\xi].
    \end{align}
    The EOM of the joint log-partition function in the $z$-space corresponds to 
    \begin{align}
    \label{eq:eom_w_p}
        \nabla_z \tilde{W}[z,p,J]=\beta p,\qquad\text{with }~ \tilde{W}[z,p,J]\coloneq-\beta\tilde{\mathcal{H}}_{\text{MHM}}[x,p,\xi]+\sum_{i=1}^dJ_ix_i.
    \end{align}
    The EOMs of the joint free energy and the generalized one are 
    \begin{align}
    \label{eq:eom_gamma_p}
        \nabla_z\tilde{\Gamma}_\alpha[x,p,\xi]=\nabla_z\tilde{\Gamma}_\alpha^{(\text{mod})}[x,p,\xi]=-\beta p~.
    \end{align}
\end{proposition}

\begin{proof}
    Under fixed $x$ and $p$, the log-partition function is $W_\alpha[J,\xi]=-\beta/2\sum_{i=1}^dx_i^2+\beta\sum_{\mu=1}^Mz_\mu p_\mu+\sum_{i=1}^dJ_ix_i$ following from the proof of \cref{thm:free_energy_mhm}.
    After taking the derivative of $W_\alpha[J,\xi]$ with respect to $J$, we obtain the first equation.
    For the second and the last ones, we derive them by taking the derivative of \eqref{eq:mhm_joint_hamiltonian}, \eqref{eq:mhm_effective_action}, and \eqref{eq:mod_eff_action} with respect to $z_\mu$ directly.
\end{proof}

We observe the emergent attention mechanism from the EOMs in \cref{prp:eom_smhm}. 
Specifically,

\begin{proposition}[Emergent Attention Mechanism as EOMs in MHMs]
\label{prp:eom_attention}
    Attention emerges as:
    
    \begin{itemize}
        \item If $\alpha\to\infty$, it corresponds to a hardmax function:
        \begin{align*}
            p_\mu=\text{hardmax}(z_\mu)=\begin{cases}1~,~&\text{for maximal }z_\mu\\0~,~&\text{otherwise}\end{cases}.
        \end{align*}
        
        \item If $\alpha=2$, it corresponds to a sparsemax function \cite{martins2016softmax}:
        \begin{align*}
            p_\mu= \text{Sparsemax}_\beta(z_\mu)=\frac{\beta}{2}[z_\mu-\theta]_+,\qquad\text{with }~\frac{\beta}{2}\sum_{\mu=1}^M[z_\mu-\theta]_+=1.
        \end{align*}

        \item If $1<\alpha<2$, it corresponds to an entmax family \cite{peters2019sparse}.
        
        \item If $\alpha\to1$, it corresponds to a softmax function \cite{blondel2020learning}:
        \begin{align*}
            p_\mu= \Softmax_\beta(z_\mu)= \frac{\exp\{\beta z_\mu\}}{\sum_{\nu=1}^M\exp\{\beta z_\nu\}}.
        \end{align*}
        
    \end{itemize}
\end{proposition}

\begin{proof}
    Please see \cref{sec:proof_eom_attention} for a detailed proof.
\end{proof}

\cref{prp:eom_attention} provides a re-derivation via EOMs of a free energy landscape, revealing the underlying physical interpretation of the attention mechanism.
In Transformer language, \cite{vaswani2017attention}, we denote $\text{Attn}(Q,K,V)\coloneq\sum_{\mu=1}^Mp_\mu(Q,K)V_\mu$
as the representation of $x$, where $Q$ represents the microscopic states, and $K$ and $V$ represent the stored memory patterns.
$Q,K,V$ are all fixed inputs. 
This corresponds to our zero variance or deterministic setting.

\section{Phase Transitions and Critical Points of Retrieval Success}
\label{sec:retrieval_success}
In this section, we discuss the spectral stability and thresholds for retrieval success.
For the dynamical stability, please see \cref{sec:dynamical_stability}.
$\beta,M,d$ and $\kappa\coloneq\exp(M)/d$ determine the free energy and allow multiple local minima in some ranges.
This leads to our study of the phase structure.
In our setting, the stored memories $\{\xi^\mu\}$ remain fixed, while the query state $x$ evolves dynamically.

\begin{definition}[Branch]
    Choose an initial parameter $\theta_0$ that determines $\Gamma_\alpha^{\text{(mod)}}[x(\theta_0),\xi]$.
    As $\theta\in\mathcal{I}_{\theta_0}$ varies from $\theta_0$, we define a branch as a continuous family of the evolution from a specific local stable point $\{x^\star_a(\theta;\theta_0)\}_{\theta\in \mathcal{I}_{\theta_0}}$ satisfying the stationary condition (\cref{def:stationary_eq})ao68 
    \begin{align*}
        \nabla_x\Gamma^{\text{(mod)}}_\alpha[x_a^\star(\theta;\theta_0),\xi]=0~,~\text{ and }~\nabla^2_x\Gamma^{(\text{mod})}_\alpha[x^\star_a(\theta;\theta_0),\xi]\succ0.
    \end{align*}
    
\end{definition}

Branches help us analyze the evolution of stable points and their loss of stability.
Adopting this setup, we study the phase transitions of the generalized model \eqref{eq:mod_eff_action} following stability analysis in statistical mechanics \cite{binder1987,Zinn-Justinqft}.
We start from defining two critical points.

\begin{definition}[Coexistence Point]
\label{def:coexist}
    Given two distinct branches, $a,b$, with their stable points, $\{x_a^\star(\theta;\theta_0)\}$, $\{x_b^\star(\theta;\theta_0)\}$.
    A coexistence point is a parameter value $\theta_{\text{coexist}}$ such that
    \begin{align}
    \label{eq:energy_coexist}
        \Gamma_\alpha^{\text{(mod)}}[x^\star_a(\theta_{\text{coexist}};\theta_0),\xi]=\Gamma_\alpha^{(\text{mod})}[x_b^\star(\theta_{\text{coexist}};\theta_0),\xi].
    \end{align}
\end{definition}

However, the coexistence condition is not a necessary condition for a real jump of local minima.
Even if the equality holds, there is still an energy barrier that forbids the transition.

\begin{definition}[Spinodal Point]
\label{def:spinodal}
    Given a smooth branch $a$, a spinodal point occurs when the branch loses its local stability.
    That is, the point where the stationary condition fails:
    \begin{align}
    \label{eq:energy_jump}
        \lambda_{\text{min}}(\nabla^2_x\Gamma_\alpha^{(\text{mod})}[x_a^\star(\theta_{\text{jump}};\theta_0),\xi])=0.
    \end{align}
    This is also the point where the continuation along that branch breaks down.
\end{definition}

\begin{remark}
\label{rmk:sparsity_nonsmooth_branch}
    If a branch is non-smooth, e.g., there is a sudden change in the effective support size, then jumps also occur at that discontinuous point.
    In this case, \eqref{eq:energy_jump} does not hold necessarily.
\end{remark}

We control the system by the parameters, $\theta=(\beta,M,d,\xi)$.
For $\alpha=1$, the retrieval map is the softmax map (\cref{prp:eom_attention}), so $p_\mu^\star$ is positive for all memories and the wights vary with $z^\star$ smoothly.
Hence, %
discontinuous change of retrieval only occurs if the Hessian vanishes, i.e., \eqref{eq:energy_jump} holds.
To investigate the transitions, we focus on the critical value of $\beta$ in the following.

\begin{proposition}[Temperature Coexistence Point for Sparse Model]
\label{prp:coex_point}
    Let $a,b$ be two distinct branches and $\alpha=2$.
    Denote $k_a$ and $k_b$ as the exact support sizes (see \cref{lem:threshold_p_mu}) of branches $a$ and $b$, respectively.
    Based on the statistical mechanics framework, the coexistence point satisfies
    \begin{align*}
        \sum_{\mu,\nu\in a}[z^\star_\mu-\theta]_+[z^\star_\nu-\theta]_+K_{\mu\nu}(\beta_{\text{coexist}};\beta_0)=\sum_{\mu,\nu\in b}[z^\star_\mu-\theta]_+[z^\star_\nu-\theta]_+K_{\mu\nu}(\beta_{\text{coexist}};\beta_0),
    \end{align*}
    where $K_{\mu\nu}(\beta)\coloneq\one_{\{\mu,\nu\}}-\frac{\beta}{2}G_{\mu\nu}$, and $G_{\mu\nu}\coloneq\sum_{i=1}^d\xi_i^\mu\xi_i^\nu/d$ is the Gram matrix.
    $\beta_{\text{coexist}}$ does not admit a closed-form expression. 
    However, if the distribution is uniform on the simplex and the system satisfies the well-separation condition, i.e.,
    \begin{align}
    \label{eq:well-separation}
        G_{\mu\nu}=\begin{cases}
            1~&,~\mu=\nu\\\epsilon\ll1~&,~\mu\neq\nu
        \end{cases},
    \end{align}
    then for $k_a=1$ and $k_b=2$, we have the analytic solution $\beta_{\text{coexist}}=2/(1-\epsilon)$.
    At this point, $\beta=\beta_{\text{coexist}}$, $p^\star$ is degenerate and every $p^\star$ is a global minimum of the joint free energy $\tilde{\Gamma}_{\alpha=2}^{(\text{mod})}$.
    
\end{proposition}

\begin{proof}
    Please see \cref{sec:proof_coex_point} for a detailed proof.
\end{proof}

\begin{proposition}[Temperature Spinodal Point for Sparse Model]
\label{prp:spinodal_point}
    Consider the sparse case $\alpha=2$ and a branch $a$ with exact support size $k$.
    Then the spinodal critical value is
    \begin{align*}
        \beta_{\text{jump}}=\frac{2}{\lambda_{\text{max}}(PGP)},
    \end{align*}
    where $P_{\mu\nu}\coloneq\one_{\{\mu,\nu\}}-k^{-1}$, and $G$ is the Gram matrix.
    If the distribution is uniform and the system satisfies the well-separation condition, \eqref{eq:well-separation}, then we have the analytic solution $\beta_{\text{jump}}=2/(1-\epsilon)$.
\end{proposition}

\begin{proof}
    Please see \cref{sec:proof_spinodal_point} for a detailed proof.
\end{proof}

\begin{remark}
\label{rmk:non_uniform}
    Under the uniform distribution condition and the well-separation condition, we conclude that $\beta_{\text{coexist}}=\beta_{\text{jump}}=2/(1-\epsilon)$. 
    However, we do not have these assumptions in general.
    Then the real jump and the coexistence point do not occur simultaneously, i.e., $\beta_{\text{jump}}\neq\beta_{\text{coexist}}$.
\end{remark}

\section{Numerical Results}
\label{sec:numerical_results}

We provide numerical evidence supporting the ideas in \cref{sec:retrieval_success}. 
Specifically, we display the phase structure of MHMs under temperature and memory-load scans, including both $\alpha=1,2$.

\subsection{Temperature Dependence}

We examine the temperature dependence of the phase structure in MHMs, including both the sparse ($\alpha=2$) and dense ($\alpha=1$) regimes.
Specifically, we characterize the temperature effect on the attention map (\cref{prp:eom_attention}) and the branches of free energy \eqref{eq:mod_eff_action}.
For the experimental setup, please see \cref{sec:numerical_studies}. 
We first start from the temperature analysis in the sparse case.

\begin{figure}[!htb]
    \centering
    \includegraphics[width=0.9\linewidth]{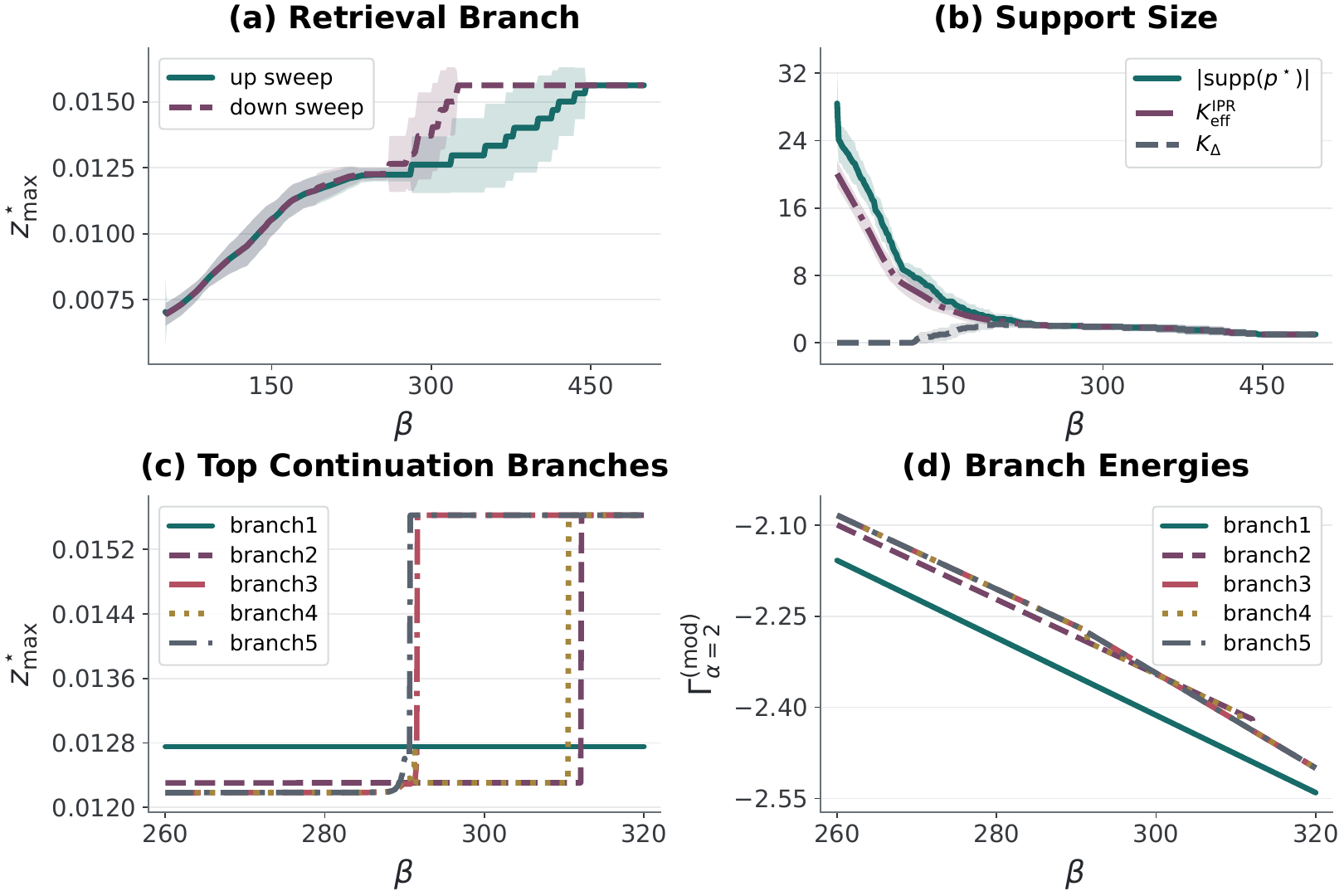}
    \vspace{-0.5em}
    \caption{\textbf{Temperature analysis of the retrieval structure of MHM in the sparse regime ($\alpha=2$).}
    {\small $(M=100,d=64)$.
    Results are averaged over 10 random pattern realizations, and shaded regions indicate one standard deviation across different seeds.
    \textbf{(a) Retrieval branches.}
    The maximal overlap $z_{\text{max}}^\star$ under up- and down-sweeps of $\beta$ exhibits hysteresis.
    \textbf{(b) Support size.} 
    We compare the exact support size (active components of the probability distribution $p$), the inverse participation ratio (IPR), and the fixed score-threshold count (\cref{sec:exp_settings}). 
    \textbf{(c) Top branches.} 
    We track the top five stable branches ranked by $z_{\text{max}}^\star$.
    \textbf{(d) Branch energies.} 
    We show the corresponding free energies of those top five branches.
    }}
    \label{fig:alpha2_phase_transition_beta}
    \vspace{-1em}
\end{figure}

\paragraph{Sparse case ($\alpha=2$).}
Here, we select four plots to represent the retrieval structure in \cref{fig:alpha2_phase_transition_beta}.
    
\begin{enumerate}[label=(\alph*)]
    \item The up- and down-sweep branches do not coincide, which shows a hysteresis window in $\beta$. 
    This indicates that the coexistence of multiple local stable retrieval states (\cref{def:coexist,prp:coex_point}).
    The selected branch depends on initialization, analogous to first-order transitions in freezing-melting hysteresis of water \cite{enninful2019cooling}.
    As $\beta$ increases, the maximal overlap $z_{\text{max}}^\star$ grows stepwise and the shaded regions collapse at large
    $\beta$.
    This indicates a seed-independent saturation.
    Specifically, it saturates at $z_{\text{max}}^\star\approx 0.0156$.
    Indeed, this is the consequence of one pattern retrieval, i.e.,
    \begin{align*}
        z^\star_\mu=\frac{1}{\sqrt{d}}\sum_{i=1}^d\xi_i^\mu x^\star_i=\frac{1}{\sqrt{d}}\times d\times\frac{1}{\sqrt{d}}\times \frac{1}{d}=\frac{1}{d}=\frac{1}{64}\approx0.0156.
    \end{align*}

    \item As $\beta$ increases, the exact support size and IPR decreases (\cref{sec:exp_settings}).
    This indicates a transition from distributed to sparse retrieval.
    The discrepancy of $|\text{supp}(p^\star)|$ and IPR arises since IPR counts the components non-uniformly.
    However, $K_\Delta$ is zero in the high-temperature regime, where noise dominates and no overlap exceeds the threshold $\Delta=0.01$.
    This suggests a disordered regime.
    In the low-temperature regime, all support sizes approach to one, so the system reaches single-pattern retrieval with large overlap.
    The collapse of the shaded regions again indicates a perfect retrieval and seed-independence.

    \item Initialization at $\beta_0=290$.
    Multi-branch continuation indicates several competing stable retrieval branches in the intermediate-$\beta$ regime. 
    The presence of these branches indicates multiple coexisting local minima.
    Their crossings and discontinuous jumps in $z_{\text{max}}^\star$ show that the sparse model admits branch switching between distinct local minima (\cref{prp:spinodal_point,rmk:sparsity_nonsmooth_branch}). 
    Hence, the selected state correspond to either stronger or weaker alignment depending on the branch we are tracking.

    \item Initialization at $\beta_0=290$.
    The branch exhibits coexist point from the plot (\cref{def:coexist,prp:coex_point}). 
    As $\beta$ varies, the global minimizer switches between branches.
    However, under continuation (warm-start dynamics), the system follows a given local branch rather than the instantaneous global minimizer.
    Hence, the transition between branches occur before reaching the exact crossing point.
    This leads to discontinuous changes along a branch. %
    Hence, we interpret this as a transition within the retrieval phase rather than a transition between retrieval and a disordered phase.
    This simulation provides a more precise phase structure than earlier analyses, \cite{amit1985,AMIT198730,Nishimori2001}.
\end{enumerate}

\begin{figure}[!htb]
    \centering
    \includegraphics[width=0.9\linewidth]{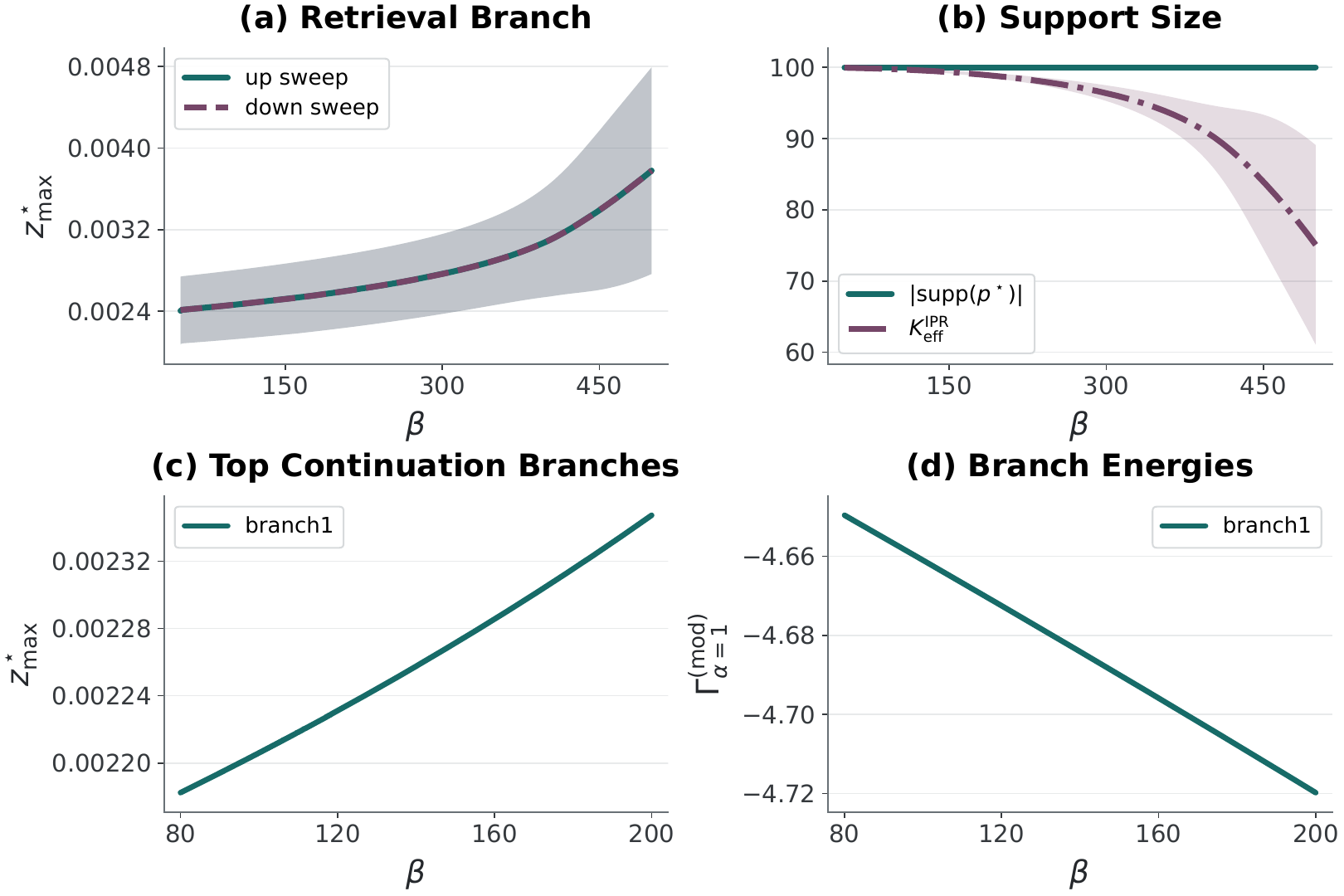}
    \vspace{-0.5em}
    \caption{\textbf{Temperature analysis of the retrieval structure of MHM in the dense regime ($\alpha=1$).}
    {\small $(M=100,d=64)$.
    Results are averaged over 10 random pattern realizations, and shaded regions indicate one standard deviation across different seeds.
    \textbf{(a) Retrieval branches.}
    The maximal overlap $z_{\text{max}}^\star$ under up- and down-sweeps of $\beta$ shows no hysteresis.
    \textbf{(b) Support size.} 
    We compare the exact support size (active components of the probability distribution $p$) and the inverse participation ratio (IPR) (\cref{sec:exp_settings}). 
    \textbf{(c) Top branches.} 
    We observe one single stable branch across the entire range of $\beta$.
    \textbf{(d) Branch energies.}
    We show the corresponding free energy of the branch.
    }}
    \label{fig:alpha1_phase_transition_beta}
    \vspace{-1em}
\end{figure}

\paragraph{Dense case ($\alpha=1$).}
Here, we select four plots to represent the retrieval structure in \cref{fig:alpha1_phase_transition_beta}.

\begin{enumerate}[label=(\alph*)]

    \item The up- and down-sweep branches coincide, which indicates the absence of hysteresis.
    Hence, the retrieval dynamics admits one stable solution for each $\beta$. 
    There is no metastable branch or discontinuous transition.
    The shaded region becomes broader as $\beta$ increases.
    Hence, the alignment with stored pattern increases, but the variance increases as well.

    \item The support size remains full for all $\beta$ (\cref{sec:exp_settings}), since $p_\mu^\star\propto\exp(\beta z_\mu^\star)$ assigns positive weight to every component $\mu$.
    As $\beta$ increases, $K_{\text{eff}}^{\text{IPR}}$ decreases, so the distribution becomes concentrated.
    However, this does not induce sparsity.
    The shaded region of $K_{\text{eff}}^{\text{IPR}}$ becomes broader as $\beta$ increases.
    This is due to the amplification of small differences from seeds.

    \item Initialization at $\beta_0=140$.
    The maximal overlap $z_{\text{max}}^\star$ grows with $\beta$.
    This indicates a continuous sharpening of alignment with the target pattern. 
    There is no discontinuity or spinodal instability due to the vanishing Hessian.
    Hence, it shows the absence of phase-transition-like behavior, based on our softmax experiment setup here.
    However, it does not indicate a unique local minimum necessarily.
    That is, other local minima may locate far away from our selected initial point so that the nearby search only displays one branch.

    \item Initialization at $\beta_0=140$.
    The energy decreases along the branch, and the retrieval state becomes more concentrated. 
    This reflects a continuous sharpening of the energy basin, rather than a restructuring of competing minima as in the sparse case.
    
\end{enumerate}

\begin{remark}
    \citet{koulischer2023exploring} analyze temperature-dependent behavior in MHM under softmax attention ($\alpha=1$).
    They show that as $\beta$ increases, the retrieval weights concentrate smoothly.
    Our results here are consistent with their findings.
    Besides, we further provide a characterization of this behavior by tracking the maximal overlap $z_{\text{max}}^\star$ and the support sizes.%
\end{remark}

\subsection{Memory-Load Dependence}
\label{sec:memory_load_dependence}

Here we study the exponential memory load $M=\exp(\kappa d)$, where $M$ is the pattern number and $d$ is the dimension of each pattern.
Specifically, we characterize the memory-load effect on the attention map (\cref{prp:eom_attention}) and the branches of free energy \eqref{eq:mod_eff_action}.
For the experimental setup, please see \cref{sec:numerical_studies}. 
We start from the memory-load analysis of the sparse case.

\begin{figure}[!htb]
    \centering
    \includegraphics[width=0.9\linewidth]{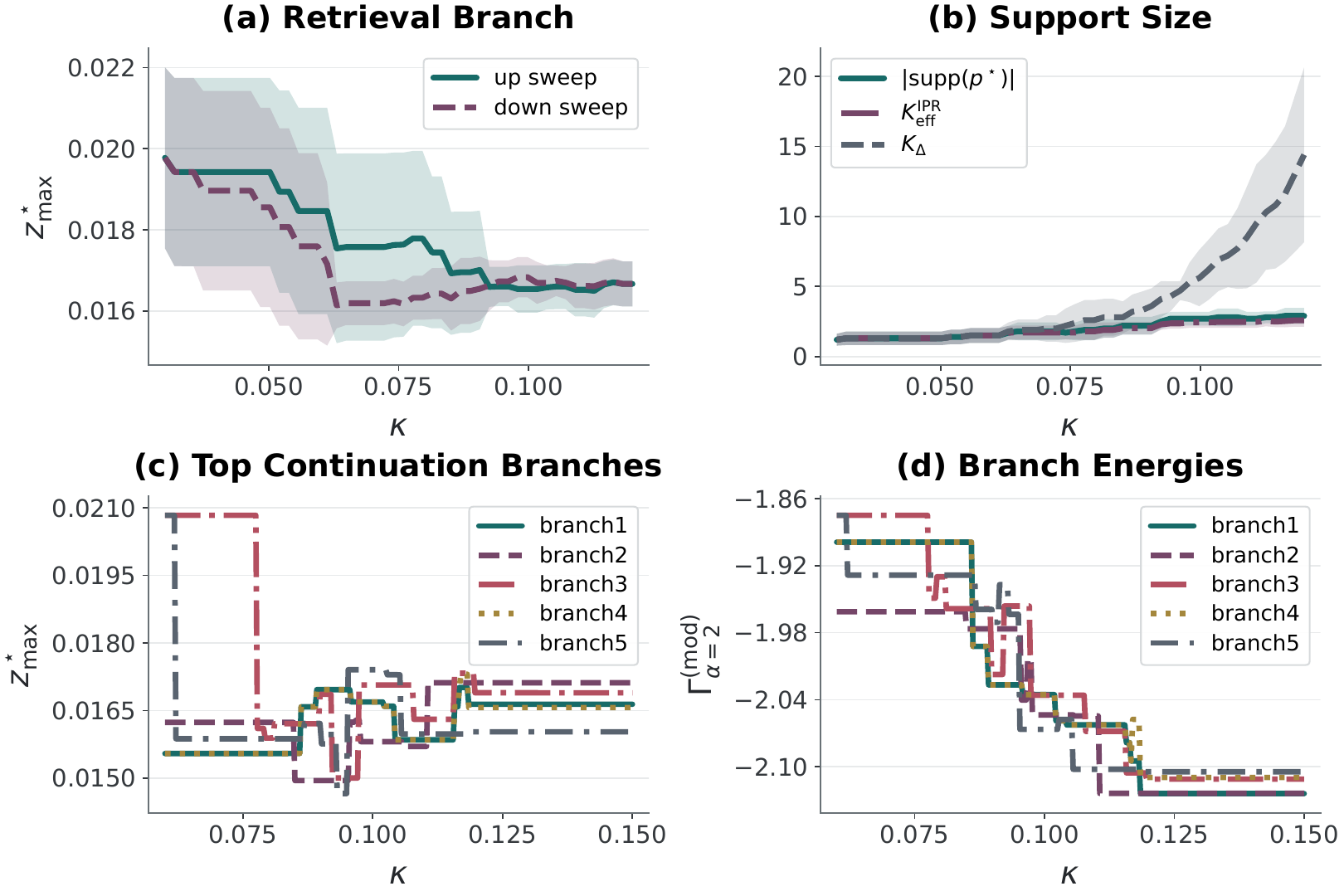}
    \vspace{-0.5em}
    \caption{\textbf{Memory-load analysis of the retrieval structure of MHM in the sparse regime ($\alpha=2$).}
    {\small ($d=48,\beta=180$).
    Results are averaged over 10 random pattern realizations, and shaded regions indicate one standard deviation across different seeds.
    \textbf{(a) Retrieval branches.}
    The maximal overlap $z_{\text{max}}^\star$ under up- and down-sweeps of $\kappa$ shows hysteresis.
    \textbf{(b) Support size.}
    We compare the exact support size $|\text{supp}(p^\star)|$ (active components of the probability distribution $p$) and the inverse participation ratio (IPR) (\cref{sec:exp_settings}).
    \textbf{(c) Top branches.}
    We track the top five stable branches ranked by $z^\star_{\text{max}}$. 
    \textbf{(d) Branch energies.} 
    We plot the corresponding free energy of the five branches.
    }}
    \label{fig:alpha2_phase_transition_kappa}
    \vspace{-1em}
\end{figure}

\paragraph{Sparse case ($\alpha=2$).}
Here, we select four plots to represent the retrieval structure in \cref{fig:alpha2_phase_transition_kappa}.

\begin{enumerate}[label=(\alph*)]

    \item The up- and down-sweep branches do not coincide, which shows a hysteresis window in $\kappa$. 
    This indicates the coexistence of multiple local stable retrieval states and the presence of first-order phase transitions.
    As $\kappa$ increases, the maximal overlap $z_{\text{max}}^\star$ decreases stepwise.
    This suggests the delocalization when the stored patterns are too many.
    
    \item The system remains in a sparse retrieval regime: although additional overlaps enter the active set as $\kappa$ increases, IPR remains close to one (\cref{sec:exp_settings}). 
    This indicates that a single dominant memory continues to control the retrieval. 
    Besides, the variance of fixed scored-threshold support size (\cref{sec:exp_settings}) grows with $\kappa$ due to the delocalization.

    \item Initialization at $\kappa_0=0.12$.
    Multiple branches coexist and represent distinct local stable retrieval states in this $\kappa$-window.
    The selected state depends on initialization and the continuation path.
    This reflects metastability in the energy landscape, i.e., multiple stable points coexist and the dynamics remains trapped in a given basin until stability is lost.

    \item Initialization at $\kappa_0=0.12$.
    We plot the corresponding free energies of the five branches.
    As $\kappa$ grows, all the branch energies decrease.
    However, it does not indicate that the retrieval is better.
    Instead, panel (a) suggests the delocalization (low score) at high $\kappa$.

\end{enumerate}

\begin{figure}[!htb]
    \centering
    \includegraphics[width=0.9\linewidth]{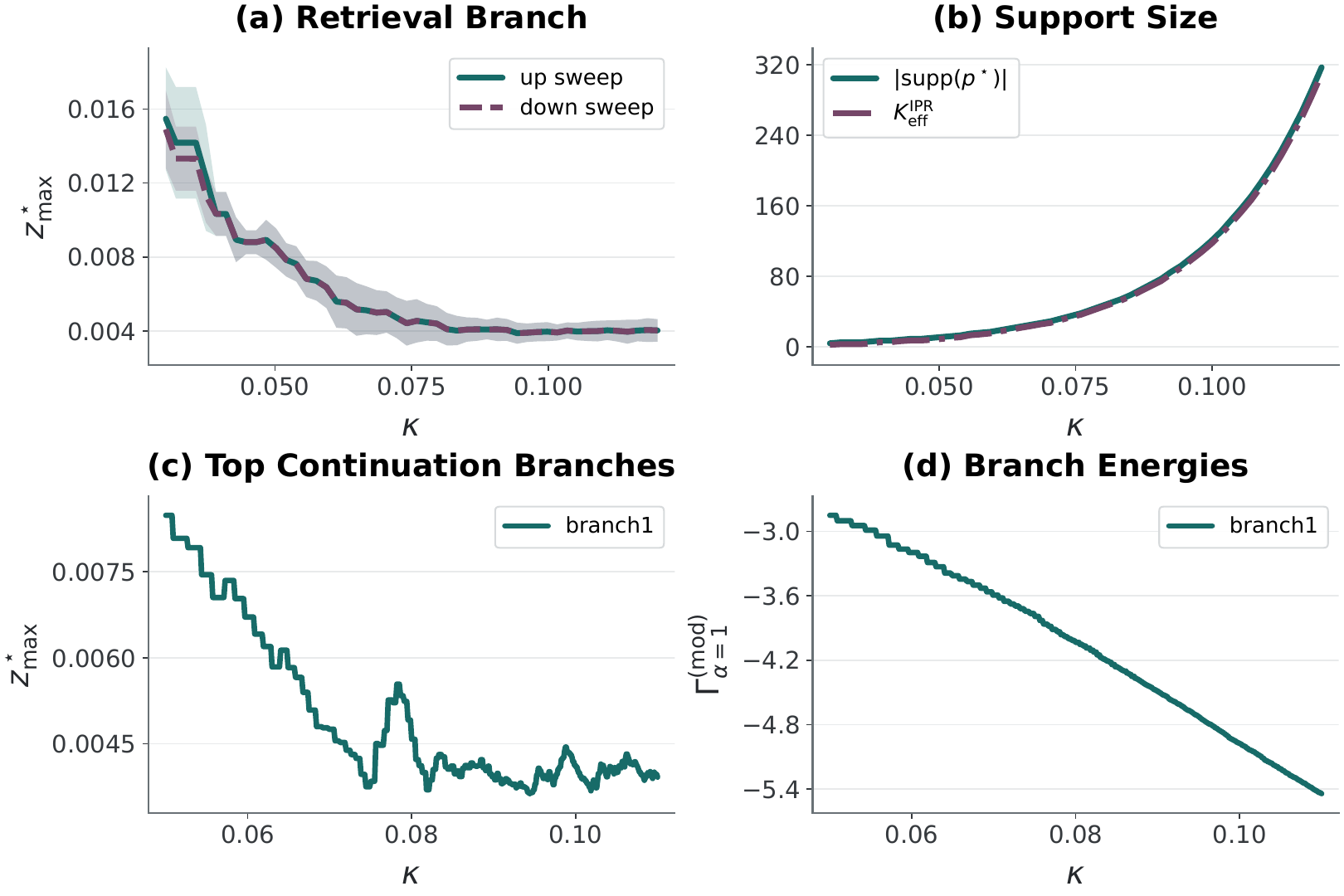}
    \vspace{-0.5em}
    \caption{\textbf{Memory-load analysis of the retrieval structure of MHM in the dense regime ($\alpha=1$).}
    {\small ($M=\exp(\kappa d),d=48,\beta=180$). 
    Results are averaged over random pattern realizations, and shaded regions indicate one standard deviation across different seeds.
    \textbf{(a) Retrieval branches.}
    The maximal overlap $z_{\text{max}}^\star$ under up- and down-sweeps of $\kappa$ shows no hysteresis.
    \textbf{(b) Support size.} 
    We compare the exact support size (active components of the probability distribution $p$) and the inverse participation ratio (IPR) (\cref{sec:exp_settings}).
    \textbf{(c) Top branches.}
    We track and observe only a single stable branch across the entire range of $\kappa$. 
    \textbf{(d) Branch energies.}
    We plot the corresponding free energy of the branch.
    }}
    \label{fig:alpha1_phase_transition_kappa}
    \vspace{-1em}
\end{figure}

\paragraph{Dense case ($\alpha=1$).}
Here, we select four plots to represent the retrieval structure in \cref{fig:alpha1_phase_transition_kappa}.

\begin{enumerate}[label=(\alph*)]
    
    \item The up- and down-sweep branches coincide, which shows the absence of hysteresis in $\kappa$. 
    That is, there is no coexistence window of multiple local stable retrieval states.
    As $\kappa$ increases, the maximal overlap $z_{\text{max}}^\star$ decreases.
    This suggests the delocalization at large $\kappa$.
    
    \item Both the exact support size and IPR (\cref{sec:exp_settings}) grow with $\kappa$. 
    This reflects the dense nature of softmax, i.e., 
    \begin{align*}
        p^\star_\mu=\frac{\exp\{\beta z^\star_\mu\}}{\sum_{\nu=1}^M\exp\{\beta z^\star_\nu\}}>0,\quad\forall\mu,\qquad\text{with overlap }~z^\star_\mu.
    \end{align*}
    All patterns have strictly positive weight. %
    Hence, the support size $|\text{supp}|$ remains full, and the delocalization drives IPR to grow exponentially with $\kappa$ as well.

    \item Initialization at $\kappa_0=0.08$.
    Only a single continuation branch appear.
    The energy decreases with $\kappa$ monotonically.
    This indicates a smooth deformation of the solution rather than transitions between competing states.
    Again, this does not indicate the uniqueness of this local minimum.
    Other local minima may locate far away from our initial point.

    \item Initialization at $\kappa_0=0.08$.
    We plot the corresponding free energies of the branch.
    As $\kappa$ grows, the branch energy decreases monotonically.
    However, it does not indicate that the retrieval is better.
    Instead, panel (a) suggests the delocalization (low $\kappa$) at high $\kappa$.
    
\end{enumerate}

\begin{remark}
    The curves show the non-smoothness as functions of $\kappa$.
    This is the consequence of the fact that we require $M$ to be an integer in our simulation.
\end{remark}

\section{Discussion and Conclusion}
\label{sec:conclusion}
We present a free energy framework for Transformer attention via modern Hopfield models. 
This viewpoint identifies attention as the equation of motion on an energy landscape. 
Our work shows the empirical reliable retrieval regime with respect to entropy, temperature, and exponential memory load (i.e., $\alpha/\beta/\kappa$).
We study the temperature and exponential memory load effects on phase structure.
In the sparse regime, discontinuous transitions occur due to support change in memories (keys).
The accompanying hysteresis indicates the coexistence of multiple stable states.
In contrast, the softmax regime exhibits smooth evolution without hysteresis.
Consequently, the softmax attention is a mixture of all patterns, while the query-key alignment grows with inverse-temperature smoothly.
These results reveal the entropy-dependent phase behavior of modern Hopfield models and hence in Transformers.

\paragraph{Practical Implications.}
We sketch some design principles from our results:
\begin{itemize}
    \item First, the attention should match the task: sparse one is support-efficient, while dense one is more accurate when many relevant values contribute (\cref{prp:eom_attention,sec:cap_reliability_memory_load,sec:support_tem_reliability}).

    \item Second, the retrieval depends on path under temperature changes (\cref{prp:coex_point,fig:alpha2_phase_transition_beta,fig:alpha1_phase_transition_beta}).
    Specifically, increasing versus decreasing $\beta$ lead to different different memory selection in the sparse case.
    The outcome also depends on the initial query (i.e., input context or prompt) supplied to the retrieval update during inference.

    \item Third, increasing the inverse-temperature strengthens the maximal query-key similarity $z_{\max}=\max_\mu\sum_{i=1}^dx_i\xi^\mu_i/\sqrt{d}$ (\cref{fig:alpha2_phase_transition_beta,fig:alpha1_phase_transition_beta}). 
    In particular, sparse attention reaches single-memory retrieval saturation, while dense one sharpens smoothly with full support.

    \item Finally, increasing memory load reduces retrieval accuracy (\cref{fig:alpha2_phase_transition_kappa,fig:alpha1_phase_transition_kappa}). 
    This further reveals a capacity-reliability trade-off in realistic Transformer architectures (\cref{sec:cap_reliability_memory_load}).
    
\end{itemize}

Together, these provide concrete guidance for selecting attention, temperature ane memory-load.

\paragraph{Limitations.}
Our analysis relies on simplified assumptions, such as prescribed correlation and noise structures, and the fixed memory setting.
In contrast, realistic Transformer architectures allow interactions across heads and layers.
This thereby leads to complicated noise.
Furthermore, the keys and queries co-evolve through learned projection matrices and upstream layers.
Hence, the theoretical analysis on phase transitions and retrieval reliability of full Transformers remain less explored.
Extending our framework along these directions is crucial for future work.

We defer \textbf{more experiments} to \cref{sec:numerical_studies,sec:add_exp} due to page limit.

\section*{Impact Statement}
By the theoretical nature of this work, we do not anticipate any negative social impact.

\section*{AI Usage Disclosure}
We used large language models (LLMs) to aid, polish writing and debug codes, such as improving clarity, grammar, and conciseness. We also used LLMs for retrieval and discovery, for example, exhausting literature to identify potential missing related work. Finally, we used LLMs to brainstorm ideas for selected experiments in \cref{sec:add_exp}, including suggestions for the model architecture, the number of attention heads, and the hidden dimensions \cite{vaswani2017attention,xiong2020layer}. All technical content, proofs, and results are original contributions by the authors. The authors take full responsibility for the paper's content and conclusions.

\section*{Acknowledgments}

JH is partially supported by Northwestern University’s Walter P. Murphy Fellowship and Terminal Year Fellowship (Paul K. Richter Memorial Award).
Han Liu is partially supported by NIH R01LM1372201, NSF
AST-2421845, Simons Foundation
MPS-AI-00010513, AbbVie , Dolby and Chan Zuckerberg Biohub Chicago Spoke Award.
The content is solely the responsibility of the authors and does not necessarily represent the official
views of the funding agencies.

\newpage
\appendix
\label{sec:append}
\part*{Appendix}
{
\setlength{\parskip}{-0em}
\startcontents[sections]
\printcontents[sections]{ }{1}{}
}

\clearpage

\section{Related Work}
\label{sec:related_works}

\textbf{Classical and Modern Hopfield Models.}
Hopfield-type models originate from the physical interpretation of associative memory in spin systems, including the Ising model \cite{ising}, the Sherrington-Kirkpatrick model \cite{skmodel}, and the classical Hopfield model \cite{classicalhopfield}.
Modern Hopfield models modify the energy function so that stored associative memories become local minima on the energy landscape, leading to a higher memory capacity \cite{krotov2020large}.

Attention mechanisms \cite{vaswani2017attention} have been highly successful in large-scale architectures.
\citet{ramsauer2020hopfield} establishes the connection between softmax modern Hopfield retrieval and Transformer attention.
\citet{hu2023sparse} further extends the softmax to sparse retrieval and sparse attention by applying convex conjugate.
\citet{peters2019sparse} introduce the entmax family of attention maps, which interpolates between dense and sparse attention.
Complementarily, \citet{pmlr-v162-millidge22a} develop universal Hopfield networks.
It formulates a retrieval dynamics through three operations: similarity, separation and projection.
\citet{santos2025hopfield} systematize this line of work into a Hopfield-Fenchel-Young framework, covering Fenchel-Young losses and the variational formulation of retrieval.

\textbf{Free Energy.}
Our formulation adopts the variational principles used in physics.
Specifically, this method appears in quantum field theory \cite{Zinn-Justinqft}, large-$N$ mean-field theory \cite{Moshe_2003}, and low-temperature theory \cite{Landau:1980mil}.
This viewpoint describes the retrieval process as a variational problem over macroscopic parameters, rather than microscopic states.
Related optimization-based interpretations of Transformer architectures are also explored in \cite{NEURIPS2022_efd1e27a}, where Transformer layers are viewed as unfolding steps of an energy minimization process.
Unlike traditional analyses of Hopfield networks that focus on energy minimization alone, the free energy framework incorporates entropy, leading to a richer landscape with multiple metastable states.

\textbf{Phase Transition.} 
Physicists often characterize phase transitions in energy-based systems through spontaneous symmetry breaking \cite{coleman_weinberg_ssb} and the emergence of metastable states. 
In classical Hopfield networks, \citet{amit1985,AMIT198730} and \citet{mezard_parisi1986} analyze this behavior using replica methods.
This reveals transitions between the retrieval and spin-glass phases, which contain numerous spurious states that lead to retrieval failure.
Free energy provides a description of phase transitions through the deformation of energy landscapes and the emergence of local minima \cite{binder1987}.
Recent work from \citet{koulischer2023exploring} discusses the phase transition in the softmax case of modern Hopfield models.
They identify a critical inverse temperature at which the mean minima are close to individual stored patterns, rather than one common averaged minimum.
However, their analysis does not include sparse retrieval map and the entmax family.

\paragraph{Dense Associative Memory.}
Dense associative memory models extend classical Hopfield models by introducing higher-order non-linear energy functions, which improved memory load storage in energy-based models \cite{Demircigil_2017,krotov2023frontier}.
These interactions further reshape the energy landscape and enhance robustness to noise \cite{krotov2018robustadversial}.

\textbf{Fuzzy Memory.}
Modern Hopfield models admit distributed retrieval over multiple memories.
Recent work studies how such retrieval behavior relates to memory capacity.
Specifically, \citet{wu2024uniform} introduce a learnable feature map that separates stored patterns in a kernel space.
This makes the distribution of memory minima more uniform.
On the other hand, \citet{hu2024provably} consider spherical codes, i.e., stored memories in the kernel feature space as points on a hypersphere.
This induces a tight capacity bound by analyzing how many feature vectors maintain angular separation on a hypersphere.
These works specify regimes where multiple patterns are activated simultaneously.
Such distributed states interpolate between sparse (single-pattern) and dense (multi-pattern) retrieval regimes.

Despite these prior works, a unified statistical framework of modern Hopfield models remains limited. 
Existing analyses focus on either retrieval capacity or dynamical behavior, without a systematic treatment of entropy, sparsity, and metastability within a single framework.
Furthermore, replica analyses characterize transitions between retrieval, spin-glass, and disordered phases, but not capture the detailed finite-dimensional geometry for modern Hopfield retrieval.
In this work, we introduce a free energy formulation that unifies these aspects and study the induced finite-dimensional retrieval landscape.
In particular, we focus on phase behavior inside the retrieval regime, rather than only transitions between retrieval and disordered phases.
This allows us to compare dense and sparse regime, analyze metastable states, and verify retrieval success.
In addition, we analyze how exponential memory load $\kappa$ changes empirical retrieval behaviors.
This probes memory-dependent phase regimes which are not addressed by existing Fenchel-Young or sparse Hopfield analyses.

\section{A Minimal Introduction of Statistical Mechanics}
\label{sec:stat_mech_tutorial}

This appendix gives a physical route from interacting spins to associative memory and then to attention.  
We set Boltzmann's constant to one, so the inverse temperature is \(\beta=1/T\).
\subsection{From Interacting Spins to Associative Memory}
\label{sec:tutorial_spins_to_memory}

A Hamiltonian assigns an energy to each microscopic configuration; a choice of dynamics specifies how the system moves through that landscape. 

\begin{definition}[Classical Hopfield Model \cite{classicalhopfield}]
    Consider binary spins \(x_i\in\{-1,+1\}\) with couplings $J_{ij}=d^{-1}\sum_{\mu=1}^M\xi_i^\mu\xi_j^\mu$ and $J_{ii}=0$.
    This gives the classical Hopfield model
    \begin{align}
        \mathcal H_{\text{Hop}}(x)=-\frac{1}{2d}\sum_{i\ne j}^d\sum_{\mu=1}^M\xi_i^\mu\xi_j^\mu x_i x_j,\label{eq:tutorial_pairwise_hamiltonian}
    \end{align}
    where memories are denoted as $\xi^1,\ldots,\xi^M\in\{-1,+1\}^d$.
\end{definition} 

Thus, zero-temperature Hopfield dynamics is a dissipative descent toward a local energy minimum.  At nonzero temperature the system instead fluctuates among configurations, as described in \cref{sec:tutorial_concrete_free_energy}.

\begin{definition}[Overlap]
    The overlap between the current state and memory $\mu$ is
    \begin{align}
    \label{eq:tutorial_classical_overlap}
        z_\mu(x)\coloneq \frac{1}{\sqrt{d}}\sum_{i=1}^d\xi_i^\mu x_i.
    \end{align}
    
\end{definition}

Therefore, low energy means large overlap with one or more stored patterns. 
\begin{definition}[Frustration \cite{kirkpatrick1977frustration}]
    A system is frustrated if it does not satisfy $J_{ij}x_ix_j>0$ for every pair with $J_{ij}\neq0$.
\end{definition}

A one-pattern Hopfield model is equivalent to an ordinary ferromagnet and retrieves without frustration.
However, with several patterns, the pairwise alignments become incompatible.
This creates frustration and additional metastable minima.  
Hence, stored patterns are attractors, but not every local minimum must be a stored memory.
This is the elementary mechanism of associative retrieval shown in \cref{fig:tutorial_retrieval_frustration}.

\begingroup
\setlength{\textfloatsep}{0.7em}
\begin{figure}[h]
\centering
\begin{tikzpicture}[every node/.style={font=\small}]
  \newcommand{\MemoryPattern}[2]{%
    \begin{scope}[shift={(#1,#2)},scale=0.3]
      \fill[white] (0,0) rectangle (7,7);
      \foreach \yy in {0,...,6}{
        \fill[blue!68!black] (0,\yy) rectangle ++(1,1);
        \fill[blue!68!black] (6,\yy) rectangle ++(1,1);
      }
      \foreach \xx in {0,...,6}{
        \fill[blue!68!black] (\xx,3) rectangle ++(1,1);
      }
      \draw[black!45,line width=0.35pt] (0,0) grid (7,7);
      \draw[black!65,line width=0.55pt] (0,0) rectangle (7,7);
    \end{scope}}
  \newcommand{\NoisyPattern}[2]{%
    \MemoryPattern{#1}{#2}
    \begin{scope}[shift={(#1,#2)},scale=0.3]
      \foreach \xx/\yy in {0/1,0/5,6/2,6/6,3/3}{
        \fill[white] (\xx,\yy) rectangle ++(1,1);
      }
      \foreach \xx/\yy in {2/0,4/1,1/5,3/6}{
        \fill[blue!68!black] (\xx,\yy) rectangle ++(1,1);
      }
      \draw[black!45,line width=0.35pt] (0,0) grid (7,7);
      \draw[black!65,line width=0.55pt] (0,0) rectangle (7,7);
    \end{scope}}

  \node[font=\bfseries] at (4.45,2.45) {(a) Associative retrieval};
  \MemoryPattern{0.15}{0.00}
  \NoisyPattern{3.25}{0.00}
  \MemoryPattern{6.55}{0.00}
  \node[align=center] at (1.05,-0.47) {stored memory\\$\xi^1$};
  \node[align=center] at (4.15,-0.47) {noisy query\\$x(t=0)$};
  \node[align=center] at (7.45,-0.47) {retrieved state\\$x^\star\simeq\xi^1$};
  \draw[->,very thick,blue!65!black] (5.45,0.92) -- (6.45,0.92)
    node[midway,above,align=center,font=\small] {$\mathcal H\downarrow$\\$z\uparrow$};

  \draw[black!25] (9.05,-1) -- (9.05,2.15);

  \node[font=\bfseries] at (12.05,2.45) {(b) Frustrated bonds};
  \coordinate (A) at (10.25,-0.15);
  \coordinate (B) at (13.85,-0.15);
  \coordinate (C) at (12.05,1.55);
  \draw[very thick] (A)--(C) node[midway,above left] {$J_{12}>0$};
  \draw[very thick] (C)--(B) node[midway,above right] {$J_{23}>0$};
  \draw[very thick,red!70!black,dashed] (A)--(B)
    node[midway,below] {$J_{31}<0$};
  \foreach \P in {A,B,C}{
    \filldraw[fill=white,thick] (\P) circle (0.27);
    \node at (\P) {$+$};
  }
  \node at (9.6,-0.15) {$x_1$};
  \node at (14.5,-0.15) {$x_3$};
  \node at (12.05,2) {$x_2$};
  \node[red!70!black,font=\bfseries] at (12.05,-0.15) {$\times$};

\end{tikzpicture}
\caption{\textbf{Retrieval and frustration in the classical Hopfield picture.}
\textbf{(a)} From a noisy query set, we retrieve a stored pattern by minimizing $\mathcal{H}$.
Equivalently, the overlap increases during this associative retrieval.
\textbf{(b)} Two ferromagnetic bonds and one antiferromagnetic bond on a triangle is impossible to satisfy pairwise energy minimization simultaneously.  This is the simplest illustration of frustration.}
\label{fig:tutorial_retrieval_frustration}
\end{figure}
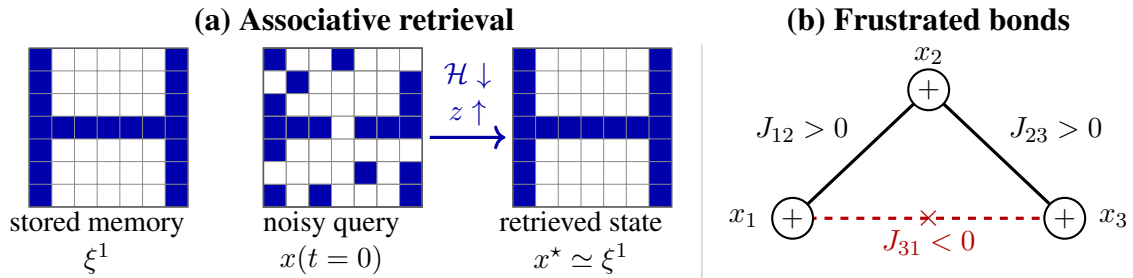
\endgroup

\subsection{A Concrete Hopfield Free Energy}
\label{sec:tutorial_concrete_free_energy}

In this section, we take one-pattern Hopfield model for a self-contained example of the statistical mechanics construction.
That is, we choose $\mathcal{H}_{\text{1pat}}(x)=-\frac{1}{2}z^2(x)$.
At temperature $T=1/\beta$, the Gibbs distribution and partition function (\cref{def:partition_function}) are 
\begin{align}
\label{eq:tutorial_one_pattern_gibbs}
    p_J(x)=\frac{e^{-\beta\mathcal H_{\mathrm{1pat}}(x)+J^\top x}}{Z(J)},\qquad Z(J)=\sum_{x\in\{-1,+1\}^d}e^{-\beta\mathcal H_{\mathrm{1pat}}(x)+J^\top x}.
\end{align}
By grouping states via Gibbs distribution and Fenchel conjugate (\cref{def:fenchel_conjugation,def:log_partition_effective_action}),
\begin{align*}
    \Gamma(z)=\sup_{J}\big\{J^\top\varphi-\log Z(J)\big\},
\end{align*}
we obtain the exact form of one-pattern Hopfield free energy.
\begin{proposition}[Free Energy of One-Pattern Hopfield Model \cite{Friedli_Velenik_2017}]
\label{eq:tutorial_one_pattern_free_energy}
    The free energy of a one-pattern Hopfield model is
    \begin{align*}
        \Gamma(z)=-\frac{1}{2}z^2-\frac{1}{\beta}\log\begin{pmatrix}d\\\frac{d+\sqrt{d}z}{2}\end{pmatrix}.
    \end{align*}
    In the large-$d$ limit, it approaches to an explicit expression
    \begin{align}
        \Gamma(z)=\beta\big(U-T S\big)\xrightarrow{~d\to\infty~}\Gamma_{\infty}(q)=-\frac{\beta}{2}q^2-s(q),
    \end{align}
    where $U=-\frac{1}{2}\mathbb{E}_{p_J}[z^2]$ is the internal energy, $S$ is the entropy (\cref{def:entropy_function}) in the statistical mechanics language, $q\coloneq\lim_{d\to\infty}\mathbb{E}_{p_J}[z]$ and
    \begin{align}
    \label{eq:tutorial_binary_entropy}
        s(q)\coloneq-\frac{d+q\sqrt{d}}{2}\log\frac{1+q/\sqrt{d}}{2}-\frac{d-q\sqrt{d}}{2}\log\frac{1-q/\sqrt{d}}{2}.
\end{align}
\end{proposition}

\begin{proof}
    Please see \cite{Friedli_Velenik_2017} for the detailed proof.
\end{proof}

\begin{proposition}[Stationary Equation of One-Pattern Hopfield Model \cite{Friedli_Velenik_2017}]
    The stationary condition of the large-$d$ one-pattern Hopfield free energy is
    \begin{align}
    \label{eq:tutorial_curie_weiss_equation}
        q^\star=\sqrt{d}\tanh\Big(\frac{\beta q^\star}{\sqrt{d}}\Big).
    \end{align}
    This is an implicit equation. 
    The solutions depend on inverse-temperature $\beta$.
\end{proposition}

\begin{proof}
    Taking the derivative , we have
    \begin{align*}
        0=\Gamma_\infty'(q^\star)=-\beta q^\star+\frac{\sqrt{d}}{2}\log\frac{1+q^\star/\sqrt{d}}{1-q^\star/\sqrt{d}}.
    \end{align*}
    This hence gives the stationary points $q^\star$ and proves the proposition.
\end{proof}

At the origin, $\Gamma_\infty''(q^\star=0)=-\beta+1$.
Therefore, as illustrated in \cref{fig:tutorial_temperature_landscape}:

\begin{itemize}

    \item for \(\beta<1\) (high temperature), \(q^\star=0\) is the unique minimum.
    However, $q^\star=\lim_{d\to\infty}\mathbb{E}_{p_J}[z]=0$ indicates disorderness, i.e., zero overlap.

    \item at \(\beta_c=1\), the Hessian at the origin vanishes.

    \item for \(\beta>1\) (low temperature), \(q^\star=0\) becomes unstable and two minima \(q^\star_{\pm}\neq0\) appear.
    Since the local minima are symmetric, the retrieval is still unique and stable.
    
\end{itemize}

\begingroup
\setlength{\textfloatsep}{1em}
\begin{figure}[H]
\centering
\resizebox{\linewidth}{!}{%
\begin{tikzpicture}[every node/.style={font=\small}]
  \node[rotate=90] at (-2.95,0.82) {$\Gamma_\infty(q)-\Gamma_\infty(0)$};

  \begin{scope}[shift={(0,0)}]
    \draw[rounded corners=2pt,black!45] (-2.65,-0.58) rectangle (2.65,2.30);
    \draw[->] (-2.38,0) -- (2.42,0) node[below] {$q$};
    \draw[->] (-2.35,-0.42) -- (-2.35,2.08);
    \draw[very thick,blue!68!black,smooth] plot coordinates {
(-2.248,1.940) (-2.173,1.600) (-2.098,1.378) (-2.023,1.202) (-1.948,1.055) (-1.873,0.930) (-1.798,0.821) (-1.723,0.725) (-1.648,0.640) (-1.573,0.564) (-1.498,0.496) (-1.424,0.435) (-1.349,0.380) (-1.274,0.331) (-1.199,0.287) (-1.124,0.247) (-1.049,0.211) (-0.974,0.179) (-0.899,0.150) (-0.824,0.124) (-0.749,0.101) (-0.674,0.081) (-0.599,0.063) (-0.524,0.048) (-0.450,0.035) (-0.375,0.024) (-0.300,0.015) (-0.225,0.009) (-0.150,0.004) (-0.075,0.001) (0.000,0.000) (0.075,0.001) (0.150,0.004) (0.225,0.009) (0.300,0.015) (0.375,0.024) (0.450,0.035) (0.524,0.048) (0.599,0.063) (0.674,0.081) (0.749,0.101) (0.824,0.124) (0.899,0.150) (0.974,0.179) (1.049,0.211) (1.124,0.247) (1.199,0.287) (1.274,0.331) (1.349,0.380) (1.424,0.435) (1.499,0.496) (1.573,0.564) (1.648,0.640) (1.723,0.725) (1.798,0.821) (1.873,0.930) (1.948,1.055) (2.023,1.202) (2.098,1.378) (2.173,1.600) (2.248,1.940)};
    \fill (0,0) circle (1.7pt);
    \node[above] at (0,0.05) {$q^\star=0$};
    \node at (0,2.5) {(a) $\beta=0.7$: high $T$};
  \end{scope}

  \begin{scope}[shift={(5.55,0)}]
    \draw[rounded corners=2pt,black!45] (-2.65,-0.58) rectangle (2.65,2.30);
    \draw[->] (-2.38,0) -- (2.42,0) node[below] {$q$};
    \draw[->] (-2.35,-0.42) -- (-2.35,2.08);
    \draw[very thick,blue!68!black,smooth] plot coordinates {
(-2.248,0.759) (-2.173,0.561) (-2.098,0.443) (-2.023,0.356) (-1.948,0.289) (-1.873,0.235) (-1.798,0.191) (-1.723,0.155) (-1.648,0.126) (-1.573,0.101) (-1.498,0.081) (-1.424,0.064) (-1.349,0.051) (-1.274,0.040) (-1.199,0.030) (-1.124,0.023) (-1.049,0.017) (-0.974,0.013) (-0.899,0.009) (-0.824,0.006) (-0.749,0.004) (-0.674,0.003) (-0.599,0.002) (-0.524,0.001) (-0.450,0.001) (-0.375,0.000) (-0.300,0.000) (-0.225,0.000) (-0.150,0.000) (-0.075,0.000) (0.000,0.000) (0.075,0.000) (0.150,0.000) (0.225,0.000) (0.300,0.000) (0.375,0.000) (0.450,0.001) (0.524,0.001) (0.599,0.002) (0.674,0.003) (0.749,0.004) (0.824,0.006) (0.899,0.009) (0.974,0.013) (1.049,0.017) (1.124,0.023) (1.199,0.030) (1.274,0.040) (1.349,0.051) (1.424,0.064) (1.499,0.081) (1.573,0.101) (1.648,0.126) (1.723,0.155) (1.798,0.191) (1.873,0.235) (1.948,0.289) (2.023,0.356) (2.098,0.443) (2.173,0.561) (2.248,0.759)};
    \fill (0,0) circle (1.7pt);
    \draw[<->,thin] (-0.65,0.18) -- (0.65,0.18);
    \node[above] at (0,0.20) {$\Gamma_\infty''(0)=0$};
    \node at (0,2.5) {(b) $\beta_c=1$: critical};
  \end{scope}

  \begin{scope}[shift={(11.10,0)}]
    \draw[rounded corners=2pt,black!45] (-2.65,-0.58) rectangle (2.65,2.30);
    \draw[->] (-2.38,0) -- (2.42,0) node[below] {$q$};
    \draw[->] (-2.35,-0.42) -- (-2.35,2.08);
    \draw[very thick,blue!68!black,smooth] plot coordinates {
(-2.248,-0.159) (-2.173,-0.248) (-2.098,-0.284) (-2.023,-0.301) (-1.948,-0.307) (-1.873,-0.305) (-1.798,-0.298) (-1.723,-0.287) (-1.648,-0.274) (-1.573,-0.258) (-1.498,-0.242) (-1.424,-0.224) (-1.349,-0.206) (-1.274,-0.187) (-1.199,-0.169) (-1.124,-0.151) (-1.049,-0.133) (-0.974,-0.116) (-0.899,-0.100) (-0.824,-0.085) (-0.749,-0.071) (-0.674,-0.058) (-0.599,-0.046) (-0.524,-0.036) (-0.450,-0.026) (-0.375,-0.018) (-0.300,-0.012) (-0.225,-0.007) (-0.150,-0.003) (-0.075,-0.001) (0.000,0.000) (0.075,-0.001) (0.150,-0.003) (0.225,-0.007) (0.300,-0.012) (0.375,-0.018) (0.450,-0.026) (0.524,-0.036) (0.599,-0.046) (0.674,-0.058) (0.749,-0.071) (0.824,-0.085) (0.899,-0.100) (0.974,-0.116) (1.049,-0.133) (1.124,-0.151) (1.199,-0.169) (1.274,-0.187) (1.349,-0.206) (1.424,-0.224) (1.499,-0.242) (1.573,-0.258) (1.648,-0.274) (1.723,-0.287) (1.798,-0.298) (1.873,-0.305) (1.948,-0.307) (2.023,-0.301) (2.098,-0.284) (2.173,-0.248) (2.248,-0.159)};
    \fill (-1.932,-0.307) circle (1.7pt);
    \fill (1.932,-0.307) circle (1.7pt);
    \fill[white] (0,0) circle (1.7pt);
    \draw (0,0) circle (1.7pt);
    \node[above left] at (-1.6,0.1) {$q_-^\star$};
    \node[above right] at (1.6,0.1) {$q_+^\star$};
    \node at (0,2.5) {(c) $\beta=1.5$: low $T$};
  \end{scope}
\end{tikzpicture}%
}
\caption{\textbf{One-pattern free-energy landscape in the large-$d$ limit.}
The curves show the one-pattern Hopfield model free energy and its local minima.
Cooling the system turns the minimum at $q^\star=0$ into two retrieval minima.}
\label{fig:tutorial_temperature_landscape}
\end{figure}
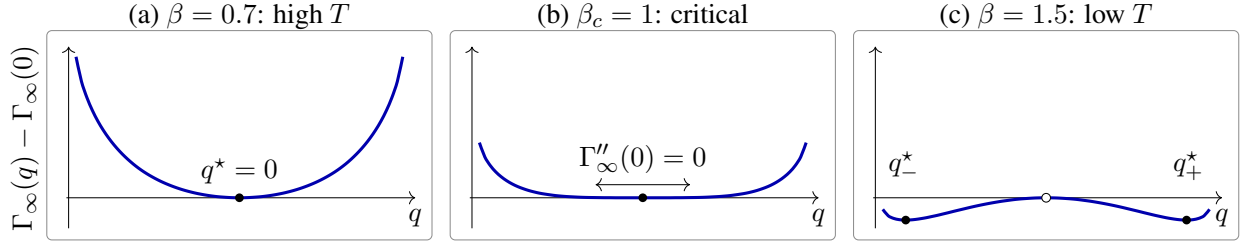
\endgroup

The deformation of the energy landscape moves the local minimum.
It leads to a phase transition. 

\subsection{Landscapes, Metastability, and Retrieval Jumps}
\label{sec:tutorial_landscapes}

The one-pattern transition above is continuous: the nonzero minima emerge from $q^\star=0$ continuously. 
More complicated retrieval landscapes contain several separated local minima.
Specifically, two minima coexist when they have equal free energy. 
If one is higher, then the minimum is metastable since a barrier prevents movement to the lower state.  
A system in one basin remains there until that basin loses local stability.  
It moves only when the energy barrier vanishes, which is called a \textit{spinodal} (\cref{def:spinodal}).
We illustrate this mechanism in \cref{fig:tutorial_jump_mechanisms}.

\begingroup
\setlength{\textfloatsep}{1em}
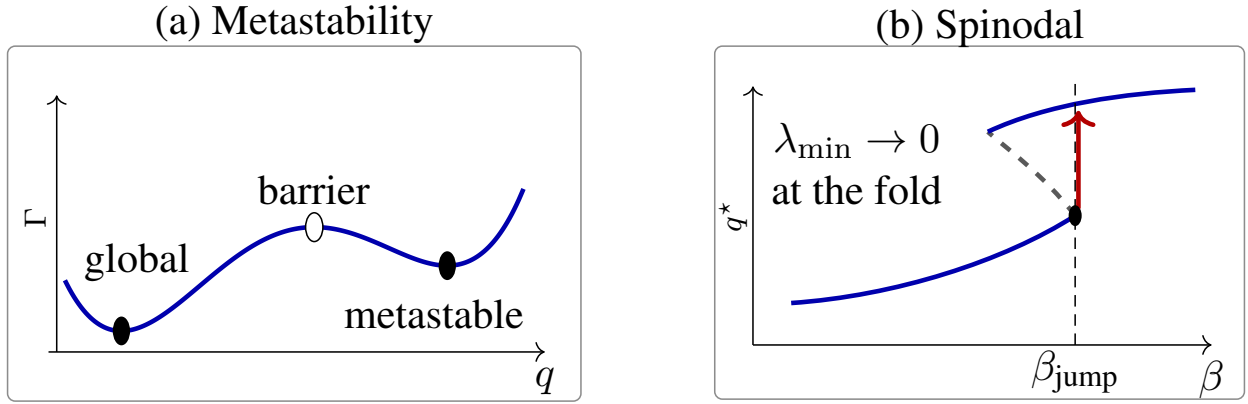
\begin{figure}[H]
\centering
\resizebox{\linewidth}{!}{%
\begin{tikzpicture}[every node/.style={font=\small}]
  \begin{scope}[shift={(0,0)},xscale=1.5,yscale=2.5]
    \draw[rounded corners=2pt,black!45] (-1.75,-0.18) rectangle (1.75,1.12);
    \draw[->] (-1.50,0) -- (1.53,0) node[below] {$q$};
    \draw[->] (-1.45,0) -- (-1.45,0.94);
    \node[rotate=90,font=\scriptsize] at (-1.58,0.49) {$\Gamma$};
    \draw[very thick,domain=-1.40:1.40,samples=100,smooth,variable=\x,blue!68!black]
      plot ({\x},{0.25*\x*\x*\x*\x-0.5*\x*\x+0.12*\x+0.45});
    \fill (-1.055,0.077) circle (1.5pt);
    \fill (0.934,0.316) circle (1.5pt);
    \fill[white] (0.122,0.457) circle (1.5pt);
    \draw (0.122,0.457) circle (1.5pt);
    \node[above right] at (-1.38,0.2) {global};
    \node[below left] at (1.5,0.25) {metastable};
    \node[above] at (0.12,0.48) {barrier};
    \node at (0,1.20) {(a) Metastability};
  \end{scope}

  \begin{scope}[shift={(4.2,0)},xscale=1,yscale=1.6]
    \draw[rounded corners=2pt,black!45] (-0.35,-0.27) rectangle (4.50,1.75);
    \draw[->] (0,0.04) -- (4.20,0.04) node[below] {$\beta$};
    \draw[->] (0,0.04) -- (0,1.52);
    \node[rotate=90,font=\scriptsize] at (-0.18,0.78) {$q^\star$};
    \draw[very thick,blue!68!black]
      (0.35,0.28) .. controls (1.30,0.32) and (2.30,0.52) .. (2.95,0.78);
    \draw[very thick,dashed,black!65]
      (2.95,0.78) .. controls (2.65,1.02) and (2.35,1.15) .. (2.15,1.26);
    \draw[very thick,blue!68!black]
      (2.15,1.26) .. controls (2.70,1.42) and (3.40,1.48) .. (4.05,1.50);
    \draw[densely dashed] (2.95,0.04) -- (2.95,1.56);
    \draw[->,very thick,red!70!black] (2.98,0.82) -- (2.98,1.38);
    \fill (2.95,0.78) circle (1.7pt);
    \node[below] at (2.95,0.1) {$\beta_{\text{jump}}$};
    \node[align=center,font=\small] at (0.95,1.08)
      {$\lambda_{\min}\to0$\\at the fold};
    \node at (2.08,1.85) {(b) Spinodal};
  \end{scope}
\end{tikzpicture}%
}
\caption{\textbf{Phase-transition picture.}
\small{
\textbf{(a)} A metastable minimum is protected by a barrier.
\textbf{(b)} We plot two smooth branches. In one smooth branch, the minimal Hessian eigenvalue vanishes at a spinodal. This leads to a jump onto another branch.}}
\label{fig:tutorial_jump_mechanisms}
\end{figure}
\endgroup

\clearpage
\section{Definitions and Proofs in the Main Text}
\label{sec:def_proof}

\subsection{Basic Tools for Convex Analysis}
\label{sec:basic_tool}

Here we define several fundamental terminologies from standard convex analysis \cite{rockafellar1997convex}.

\begin{definition}[Proper]
\label{def:proper}
    A function $f$ is proper if the support domain of $f$ satisfies 
    \begin{align*}
        \text{dom}(f)\coloneq\{x\in\R^d:f(x)<\infty\}\neq \emptyset,\qquad\text{ and }~ f\neq+\infty.
    \end{align*}
\end{definition}

\begin{definition}[Lower Semi-Continuous (lsc)]
    A function $f$ is lower semi-continuous (lsc) if 
    \begin{align*}
        f(x)\leq\lim\inf_{\tau\to\infty}f(x_\tau),
    \end{align*}
    for every sequence $(x_i)_\tau\to x_i$.
\end{definition}

\begin{definition}[Closed Convex]
    A function $f$ is closed convex if $f$ is proper, lsc and convex.
\end{definition}

\begin{definition}[Essentially Smooth]
\label{def:essentially_smooth}
    A function $f$ is essentially smooth if $f$ is differentiable on $\text{int~dom}(f)$, and $\|\nabla f(x)\|\to\infty$ whenever $x_i\rightarrow\bar{x}_i\in\partial(\text{dom}(f))$.
\end{definition}

\begin{definition}[Subgradient and Subdifferential]
\label{def:subgradient}
    Let $f:\R^d\to[-\infty,+\infty)$ be a proper convex function.
    A vector $g\in\R^d$ is called a subgradient of $f$ at $x\in\text{dom}(f)$ if
    \begin{align*}
        f(y)\geq f(x)+\sum_{i=1}^dg_i(y_i-x_i)\qquad\forall~y\in\R^d.
    \end{align*}
    The set of all subgradients of $f$ at $x$ is called the subdifferential of $f$ at $x$, and is denoted by
    \begin{align*}
        \partial f(x)=\{g\in\R^d:f(y)\geq f(x)+\sum_{i=1}^dg_i(y_i-x_i)~\forall~y\in\R^d\}.
    \end{align*}
    If $f$ is differentiable at $x$, then the subdifferential reduces to $\partial f(x)=\{\nabla f(x)\}$.
\end{definition}

\begin{definition}[Fenchel Conjugation \cite{rockafellar1997convex}]
\label{def:fenchel_conjugation}
    Consider a function $f:\R^d\to\R$ and an input $x\in\R^d$. 
    Then we define the Fenchel conjugate of $f(x)$ by the following variational formula,
    \begin{align*}
        f^\star(k)\coloneq\sup_{x\in\R^d}\{\sum_{i=1}^d k_ix_i-f(x)\}\in\R~,\qquad k\in\R^d.
    \end{align*}
\end{definition}

\begin{fact}[Fenchel Conjugation Properties \cite{rockafellar1997convex}]
\label{prp:properties_fenchel_conjugates}
\cref{def:fenchel_conjugation} has several properties:

    \begin{enumerate}
        \item \textbf{Convexity.} A Fenchel conjugate $f^\star(x)$ is by definition a convex function.
        
        \item  \textbf{Involutivity.} If a function $f$ is closed convex, then applying Fenchel conjugation twice yields the identity map, i.e., $f^{\star\star}=f$ for all $x\in\R^d$.
        This is the well-known Fenchel-Moreau Theorem.
        
        \item \textbf{Bijectiveness.} If a function $f$ is proper, lsc, strictly convex and essentially smooth, then $\nabla f^{\star}:\text{int~dom}(f^\star)\leftrightarrow\text{int~dom}(f)$ is a bijection, where int and dom denote interior and domain.
    \end{enumerate}
\end{fact}

\begin{definition}[Legendre Transformation]
\label{def:legendre}
    A function $f$ is called a Legendre function if it is proper, lsc, strictly convex and essentially smooth.
    In other words, Legendre functions satisfy involutivity and bijectiveness.
    In this case, we refer to the Fenchel conjugate $f^\star$ as the Legendre transformation.
\end{definition}

\begin{remark}
    A Legendre function $f$ satisfies the Fenchel-Moreau involutivity $f^{\star\star}=f$, and its gradient map $\nabla f$ is a bijection between the interior domains from \cref{prp:properties_fenchel_conjugates} automatically.
\end{remark}

\begin{remark}
    Here we adopt the definition of the Legendre transformation from \cite{rockafellar1997convex}.
    However, in the physics and machine learning literature, for example, in \cite{krotov2020large} and \cite{ramsauer2020hopfield}, people denote the same conjugation without verifying the conditions.
    Specifically, in many cases, the domain is $\R^d$, so $\partial\R^d=\emptyset$ leads to the essential smoothness condition being vacuous.
    The real issue is the non-bijectiveness of gradient maps.
    We should understand the transformations in these works as the Fenchel conjugations rather than the true Legendre transforms.

\end{remark}

\subsection{Proof of \texorpdfstring{\cref{prp:fenchel_young_duality_relation}}{}}
\label{sec:proof_fenchel_young}

\begin{proposition}[\cref{prp:fenchel_young_duality_relation} Restated: Fenchel-Young Duality Relation \cite{Srednicki:2007qs,blondel2020learning}]
    Let $p_J$ be the Gibbs measure with external source $J$.
    From the standard statistical field theory, the induced macroscopic field $\varphi=\mathbb{E}_{p_J}[x]$ reveals the following source-field relation:
    \begin{align*}
        \varphi=\mathbb{E}_{x\sim p_J}[x]\in\partial W[J]~\Leftrightarrow~J\in\partial\Gamma[\varphi]~,
    \end{align*}
    where $\partial$ is the subdifferential from \cref{def:subgradient}.
    If $W$ and $\Gamma$ are differentiable, then $\partial$ reduce to gradients $\nabla$.
    We call them equations of motion (EOMs) in the physics literature. 
    It also refers to the Fenchel-Young duality relation. 
    For $\varphi$ and $J$, we call them the classical solutions or saddle points.
\end{proposition}

\begin{proof}
    We repeat the proof from \cite{blondel2020learning,rockafellar1997convex} for completeness.
    By definition, $\Gamma[\varphi]=\sup_{J\in\R^d}\{\sum_{i=1}^dJ_i\varphi_i-W[J]\}$.
    Therefore, for any $J$, $\Gamma[\varphi]\geq\sum_{i=1}^d J_i\varphi_i-W[J]$.    
    By the definition of subgradient \cref{def:subgradient}, we have
    \begin{align*}
        W[J]\geq W[J^\star]+\sum_{i=1}^d\varphi_i(J_i-J^\star_i),\qquad\forall J\in\R^d~\forall\varphi\in\partial W[J^\star].
    \end{align*}
    Then 
    \begin{align*}
        \sum_{i=1}^dJ^\star_i\varphi_i-W[J^\star]\geq\sum_{i=1}^dJ_i\varphi_i-W[J]\qquad\forall J\in\R^d.
    \end{align*}
    Hence, $J^\star$ is a maximizer, and therefore $J^\star\in\partial \Gamma[\varphi]$.
    Conversely, if we assume that $J^\star\in\partial\Gamma[\varphi^\star]$, then
    \begin{align*}
        \Gamma[\varphi]\geq\Gamma[\varphi^\star]+\sum_{i=1}^dJ^\star_i(\varphi_i-\varphi_i^\star)\qquad\forall\varphi\in\R^d.
    \end{align*}
    In this way,
    \begin{align*}
        \sum_{i=1}^dJ^\star_i\varphi_i^\star-\Gamma[\varphi^\star]\geq\sum_{i=1}^dJ^\star_i\varphi_i-\Gamma[\varphi]\qquad\forall\varphi\in\R^d.
    \end{align*}
    Note that $W[J]$ is closed due to the form \eqref{eq:log_partition_function}. 
    Then by involutivity in \cref{prp:properties_fenchel_conjugates}, we obtain
    \begin{align*}
        W[J^\star]=W^{\star\star}[J^\star]=\sup_{\varphi\in\R^d}\{\sum_{i=1}^d\varphi_iJ^\star_i-\Gamma[\varphi]\}\geq\sum_{i=1}^d\varphi_iJ^\star_i-\Gamma[\varphi]~~\forall\varphi\in\R^d.  
    \end{align*}
    Hence, $\varphi^\star$ is a maximizer, and therefore $\varphi^\star\in\partial W[J^\star]$.
    After removing the redundant notation $\star$, we arrive the Fenchel-Young duality relation.
    \end{proof}

\begin{remark}
    The non-zero equation, $\nabla_J W(J)$, reflects a symmetry breaking of a free theory due to the external input. 
    For example, consider a spin system. Every spin has a random orientation, $S\in\{-1,+1\}^d$, and there exists a $\mathbb{Z}_2$ symmetry that flips the entire system upside down. 
    Then,
    \begin{align}
    \label{eq:nabla_W_borken_vaccua}
        \nabla_J W[J]|_{J=0}=\frac{S\exp\{-\beta\mathcal{H}[S]+\sum_{i=1}^dJ_iS_i\}}{Z[J]}|_{J=0}=\mathbb{E}_{S\sim p(S)}[S]=0~,
    \end{align}
    where $p(S)\propto\exp\{-\beta\mathcal{H}[S]\}$ up to a normalization factor.
    This gives the macroscopic variable $\varphi=0$.
    However, if we input the external source, $J\neq0$, then the spin system favors one orientation. 
    This corresponds to a response in a ferromagnet. 
    After this, we approach $J\to0^+$ as a stationary condition for $\Gamma[\varphi]$.
    The system stays at one specific local minima.
    See \cref{fig:branch-selection} for the illustration.
    We shall see that in the sparse modern Hopfield models (\cref{sec:stas_att_hopfield}), we always require this response, $\varphi=\nabla_J W[J]\neq0$, to realize the retrieval states. 
\end{remark}

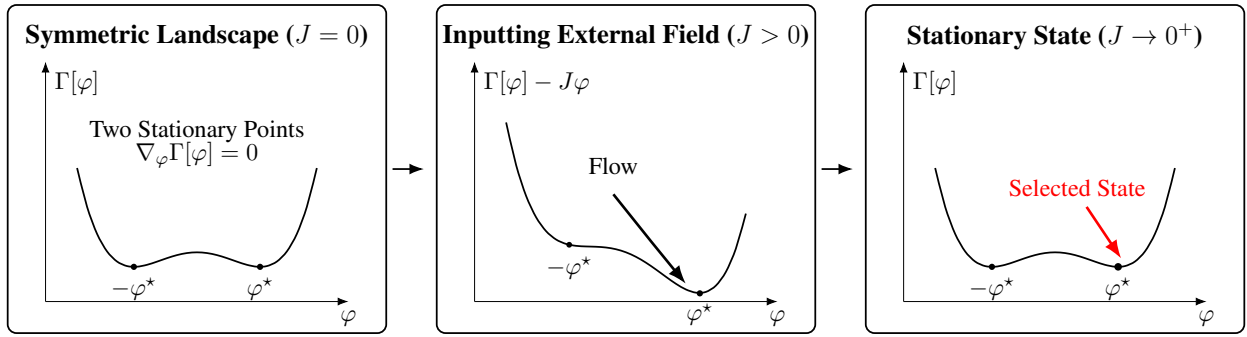
\begin{figure}[t]
    \centering
    \resizebox{1\columnwidth}{!}{%
    \begin{tikzpicture}[>=Latex, every node/.style={font=\small}]

    \begin{scope}[shift={(0,0)}]
        \draw[rounded corners=4pt, thick] (-0.6,-2.5) rectangle (5.4,2.7);

        \node[font=\bfseries] at (2.4,2.2) {Symmetric Landscape ($J=0$)};

        \draw[->] (0,-2.0) -- (0,1.8) node[below right] {$\Gamma[\varphi]$};
        \draw[->] (0,-2.0) -- (4.8,-2.0) node[below] {$\varphi$};

        \draw[thick, domain=0.5:4.3, smooth, variable=\x]
        plot ({\x},{0.23*(\x-1.4)^2*(\x-3.4)^2 - 1.45});

        \fill (1.4,-1.45) circle (1.4pt);
        \fill (3.4,-1.45) circle (1.4pt);

        \node[below] at (1.4,-1.48) {$-\varphi^\star$};
        \node[below] at (3.4,-1.48) {$\varphi^\star$};

        \node[align=center] at (2.4,0.5)
        {Two Stationary Points\\[-1mm] $\nabla_\varphi \Gamma[\varphi]=0$};
    \end{scope}

    \draw[->, thick] (5.5,0.1) -- (6,0.1);

    \begin{scope}[shift={(6.8,0)}]
        \draw[rounded corners=4pt, thick] (-0.6,-2.5) rectangle (5.4,2.7);

        \node[font=\bfseries] at (2.4,2.2) {Inputting External Field ($J>0$)};

        \draw[->] (0,-2.0) -- (0,1.8) node[below right] {$\Gamma[\varphi]-J\varphi$};
        \draw[->] (0,-2.0) -- (4.8,-2.0) node[below] {$\varphi$};

        \draw[thick, domain=0.5:4.3, smooth, variable=\x]
        plot ({\x},{0.23*(\x-1.4)^2*(\x-3.4)^2 - 1.45 - 0.38*(\x-2.4)});

        \fill (1.5,-1.1) circle (1.3pt);
        \fill (3.57,-1.87) circle (1.4pt);

        \node[below] at (1.5,-1.15) {$-\varphi^\star$};
        \node[below] at (3.57,-1.88) {$\varphi^\star$};

        \draw[->, very thick] (2.2,-0.3) .. controls (2.6,-0.8) and (3.0,-1.3) .. (3.35,-1.72);
        \node at (2.2,0.15) {Flow};

    \end{scope}
    \draw[->, thick] (12.3,0.1) -- (12.8,0.1);

    \begin{scope}[shift={(13.6,0)}]
        \draw[rounded corners=4pt, thick] (-0.6,-2.5) rectangle (5.4,2.7);

        \node[font=\bfseries] at (2.4,2.2) {Stationary State ($J\to0^+$)};

        \draw[->] (0,-2.0) -- (0,1.8) node[below right] {$\Gamma[\varphi]$};
        \draw[->] (0,-2.0) -- (4.8,-2.0) node[below] {$\varphi$};

        \draw[thick, domain=0.5:4.3, smooth, variable=\x]
        plot ({\x},{0.23*(\x-1.4)^2*(\x-3.4)^2 - 1.45});

        \fill (1.4,-1.45) circle (1.4pt);
        \fill (3.4,-1.45) circle (1.7pt);

        \node[below] at (1.4,-1.48) {$-\varphi^\star$};
        \node[below] at (3.4,-1.48) {$\varphi^\star$};

        \draw[red, very thick, ->] (2.9,-0.5) -- (3.45,-1.32);
        \node[red] at (2.75,-0.2) {Selected State};

    \end{scope}

    \end{tikzpicture}%
    }
    \caption{\small
        \textbf{Selection of a stationary state by an external source.}
        {\small This figure illustrates how a source $J$ selects one state from multiple degenerate stationary states of the free energy $\Gamma[\varphi]$.
        \textbf{Left.} When $J=0$, $\Gamma[\varphi]$ is symmetric and has two degenerate local minima $\pm\varphi^\star$.
        \textbf{Middle.} Turning on a small external source $J>0$ tilts the landscape to $\Gamma[\varphi]-J\varphi$.
        This breaks the symmetry so that $+\varphi^\star$ becomes preferred.
        Then ,the system flows toward the selected basin.
        \textbf{Right.} After the selection, we take the limit $J\to0^+$.
        Although the explicit source is removed, the system remains in the chosen local minimum $+\varphi^\star$ satisfying $\nabla_\varphi\Gamma[+\varphi^\star]=0$.
        This source-selection procedure resolves the degeneracy between degenerate stationary states and specifies the retrieval state.
    }}
    \label{fig:branch-selection}
\end{figure}

\begin{remark}
\label{rmk:gradient_induced_not_gradient}
    From the Fenchel-Young duality relation, \cref{prp:fenchel_young_duality_relation}, we know that a gradient descent trajectory in $J$-space has one-to-one mapping to a unique trajectory in $\varphi$-space.
    However, the induced unique trajectory is not a gradient flow.
    Specifically, suppose we define gradient descent as $\varphi=\dot{J}=-\nabla_J W[J]$, where $\dot{J}\coloneq dJ/dt$ if we view update steps as infinitesimal.
    Then
    \begin{align*}
        \dot{\varphi}=\nabla^2_JW[J](-\nabla_JW[J])=-\varphi\nabla^2_JW[J]~,
    \end{align*}
    which is not the gradient flow, $\dot{\varphi}=-\nabla_\varphi\Gamma[\varphi]$.
\end{remark}

\subsection{Tsallis Entropy}
\label{sec:tsallis_entropy}

In this section, we introduce the Tsallis entropy and its role in statistical mechanics.

\begin{definition}[Tsallis Entropy \cite{Tsallis:1987eu}]
\label{def:entropy_function}
    We define the Tsallis entropy, \cite{Tsallis:1987eu}, of the form
    \begin{align}
    \label{eq:tsallis_entropy_function}
        \Psi_\alpha[p]\coloneq\frac{1}{\alpha-1}(1-\sum_{\mu=1}^Mp_\mu^\alpha)~,
    \end{align}
    where $\sum_{\mu=1}^Mp_\mu=1$, $\alpha\in\R_+/\{1\}$. 
\end{definition}

\begin{remark}
    The maximal Tsallis entropy occurs when the probability function is uniform, i.e., $p_\mu=M^{-1}$ for every $\mu$.
    Hence, if the system is random, then the entropy becomes maximal.
\end{remark}

\begin{proposition}[Types of Entropy \cite{Tsallis:1987eu}]
\label{prp:tsallis_properties}
    Specific values of $\alpha$ in Tsallis entropy reduce to other types of entropy.
    
    \begin{itemize}
        \item If $\alpha\to1$, it reduces to the Shannon entropy \cite{shannon1948mathematical}.
        \begin{align}
        \label{eq:shannon_entropy}
            S[p]\coloneq-\sum_{\mu=1}^Mp_\mu\log(p_\mu).
        \end{align}
        
        \item If $\alpha\to2$, it reduces to the Gini entropy \cite{gini1912variabilita}.
        \begin{align}
        \label{eq:gini_entropy}
            G[p]\coloneq1-\sum_{\mu=1}^Mp_\mu^2.
        \end{align}
        
        \item If $\alpha\to\infty$, it is zero.
    \end{itemize}
    
\end{proposition}

\begin{proof}
    We repeat the proof from \cite{Tsallis:1987eu} for completeness.
    We prove them one-by-one as follows.
    
    \begin{itemize}
        \item Applying the l’Hôpital’s rule for $0/0$-form, we have
        \begin{align*}
            \Psi_{\alpha\to1}[p]=-\lim_{\alpha\to1}\sum_{\mu=1}^M\log(p_\mu)p_\mu^\alpha=-\sum_{\mu=1}^Mp_\mu\log(p_\mu)\coloneq S[p]~.
        \end{align*}
        This is the form of Shannon entropy.
        
        \item For $\alpha\to2$, the Tsallis entropy function is by definition the Gini entropy.
        
        \item Let $p_{\text{max}}\coloneq\max_{\mu}~p_\mu$ and $\mu_{\text{max}}$ be the corresponding maximal index. Then we separate the Tsallis entropy as
        \begin{align*}
            \Psi_\alpha[p]=\frac{1}{\alpha-1}[1-p_{\text{max}}^\alpha(1+\sum_{\mu\neq\mu_{\text{max}}}^M\frac{p_\mu^\alpha}{p_{\text{max}}^\alpha})]~.
        \end{align*}
        After taking $\alpha\to\infty$ and applying the l’Hôpital’s rule for $\infty/\infty$-form, we obtain
        \begin{align*}
            \Psi_{\alpha\to\infty}[p]=-\lim_{\alpha\to\infty}\log(p_{\text{max}})p_{\text{max}}^\alpha~,
        \end{align*}
        where $(p_\mu/p_{\text{max}})^\alpha\to0$ as $\alpha\to\infty$. 
        If $p_{\text{max}}<1$, then $\Psi_{\alpha\to\infty}=0$.
        If $p_{\text{max}}=1$, then $\Psi_{\alpha\to\infty}\to0$ due to the logarithm. Hence, $\Psi_{\alpha\to\infty}[p]$ for all $p$.
    \end{itemize}
    
    This completes the proof.
\end{proof}

\begin{remark}
\label{rmk:convexity_tsallis_entropy}
    If $\alpha>0$, then the Tsallis entropy is strictly concave since its second derivative respect to $p$ is negative on the simplex interior.
    This includes Shannon entropy ($\alpha\to1$), which is also strictly concave.
\end{remark}

\begin{remark}
    Although the Tsallis entropy is not defined at $\alpha=1$ from \cref{def:entropy_function}, we use the continuous extension $\Psi_1[p]\coloneq\lim_{\alpha\to1}\Psi_\alpha[p]$.
    It is equal to the Shannon entropy according to \cref{prp:tsallis_properties}.
    Therefore, throughout the paper, the $\alpha=1$ refers to this Shannon/softmax limit.
\end{remark}

\begin{proposition}[Fenchel Conjugate of Entropy \cite{blondel2020learning}]
\label{prp:fenchel_conjugate_tsallis_entropy}
    The Fenchel conjugate of the negative Tsallis entropy is
    \begin{align}
    \label{eq:fenchel_conjugate_tsallis_entropy}
        \Omega_\alpha^\star[z]\coloneq(-\Psi_\alpha)^\star[z]=\sup_{p\in\Delta^M}\{\sum_{\mu=1}^Mz_\mu p_\mu+\Psi_\alpha[p]\}~,
    \end{align}
    where $\Delta^M\coloneq\{p_\mu\in\R_+^M|\sum_{\mu=1}^Mp_\mu=1\}$ is a $(M-1)$-dimensional unit simplex.
    Then 
    \begin{itemize}
        \item In the Shannon case, $\alpha\to1$, the negative entropy $\Omega_{\alpha}$ is a Legendre function.

        \item For $\alpha>1$, $\Omega_\alpha$ satisfies involutivity, but neither Legendre condition nor bijectiveness.

        \item For $1>\alpha>0$, $\Omega_\alpha$ is a Legendre function.
    \end{itemize} 
    
\end{proposition}

\begin{proof}
    First, we look at the Shannon case.
    
    \begin{itemize}
        \item Due to \eqref{eq:shannon_entropy} and \cref{rmk:convexity_tsallis_entropy}, the Shannon entropy is closed and strictly concave.
        By \cref{prp:properties_fenchel_conjugates}, it satisfies involutivity, $S^{\star\star}[p]=S[p]$.
        Furthermore, $\|\nabla_pS[p]\|_{p\to0}=\infty$.
        Thus, it is essentially smooth.
        The negative Shannon entropy $\Omega_\alpha$ is thereby a Legendre function and satisfies bijectiveness.

        \item For $\alpha>1$, the Tsallis entropy is closed and strictly concave on the simplex.
        Therefore, $\Omega_\alpha=-\Psi_\alpha$ satisfies involutivity condition $\Omega^{\star\star}_{\alpha>1}[p]=\Omega_{\alpha>1}[p]$.
        However, $\|\nabla_p\Omega_{\alpha>1}[p]\|_{p\to0}$ does not diverge, so it is not essentially smooth.
        As a result, it is not a Legendre function nor satisfies bijectiveness.

        \item For $1>\alpha>0$, the Tsallis entropy is strictly concave.
        Hence, $\Omega_\alpha$ is strictly convex.
        Furthermore, $\partial\Omega_\alpha/\partial p_\mu\to-\infty$ near the boundary, so it satisfies the essential smoothness.
        Therefore, it is a Legendre function satisfying the invlolutivity and bijectiveness properties.

    \end{itemize}
    
    This completes the proof.
\end{proof}

\begin{remark}
    We are able to express unit simplex condition in two different ways: $1=\sum_{\mu=1}^Mp_\mu=\int_{\mathcal{X}}p_\mu(x)~\dd\mu(x)$.
    That is, we transform the $x$ configuration space, $\mathcal{X}$, onto a simplex, $\Delta^M$.
\end{remark}

\begin{definition}[$\alpha$-Deformed Exponential Family and Logarithm \cite{Tsallis:1987eu}]
\label{def:alpha_deformed_exponential}
    Let $\alpha>0$.
    We define the $\alpha$-deformed exponential function as 
    \begin{align*}
        \exp_\alpha(x)\coloneq\begin{cases}
            \exp(x) & , ~ \text{if}~\alpha=1\\
            [1+(1-\alpha)x]_+^{\frac{1}{1-\alpha}} & , ~ \text{if}~\alpha\neq1
        \end{cases}~,
    \end{align*}
    and its inverse as
    \begin{align*}
        \log_\alpha(y)\coloneq\begin{cases}
            \log(y) & , ~ \text{if}~\alpha=1\\
            \frac{y^{1-\alpha}-1}{1-\alpha} & , ~ \text{if}~\alpha\neq1
        \end{cases}~,
    \end{align*}
    where $x\in\R$, $y\in\R_+$ for well-definedness, and $[z]_+\coloneq \max\{z,0\}$ for every $z\in\R$.
\end{definition}

This $\alpha$-deformed family has several properties, and some of which differ from the ordinary ones.

\begin{lemma}[Properties of $\alpha$-Deformed Functions \cite{Tsallis:1987eu}]
\label{prp:alpha_deformed_functions_properties}
    $\alpha$-deformed functions satisfy
    
    \begin{itemize}
        \item $\lim_{\alpha\to1}\exp_\alpha(x)=\exp(x)$ and $\lim_{\alpha\to1}\log_\alpha(x)=\log(x)$.
        \item $\exp_\alpha(x)/\log_\alpha(y)$ are monotonic and convex/concave.
        \item $\alpha$-deformed exponential/logarithm does not have multiplicity/additivity unless $\alpha=1$.
        \item $\frac{d}{dx}\exp_\alpha(x)=[\exp_\alpha(x)]^\alpha$ in the domain $[1+(1-\alpha)x]>0$ and $\frac{d}{dy}\log_\alpha(y)=y^{-\alpha}$.
    \end{itemize}
    
\end{lemma}

\begin{proof}
    We repeat the proof from \cite{Tsallis:1987eu} for completeness.
    We prove those properties one-by-one as follows.
    
    \begin{itemize}
        \item Let $u=1-\alpha$.
        By \cref{def:alpha_deformed_exponential}, 
        \begin{align*}
            \lim_{\alpha\to1}\exp_\alpha(x)=\lim_{\alpha\to1}[1+(1-\alpha)x]_+^{\frac{1}{1-\alpha}}=\lim_{u\to0}[1+ux]^{1/u}=\exp(x)~.
        \end{align*}
        For the logarithm, we apply the l’Hôpital’s rule for $0/0$-form.
        \begin{align*}
            \lim_{\alpha\to1}\log_\alpha(y)=\lim_{\alpha\to1}\frac{y^{1-\alpha}-1}{1-\alpha}=-\lim_{\alpha\to1}\log(y)y^{1-\alpha}=\log(y)~.
        \end{align*}
        \item If $[1+(1-\alpha)x]>0$, by differentiating $\exp_\alpha(x)$ with respect to $x$, we have
        \begin{align*}
            \frac{d}{dx}\exp_\alpha(x)=[1+(1-\alpha)x]^{\frac{\alpha}{1-\alpha}}>0~,~\frac{d^2}{dx^2}\exp_\alpha(x)=\alpha[1+(1-\alpha)x]^{\frac{2\alpha-1}{1-\alpha}}>0~.
        \end{align*}
        Hence, it is strictly increasing (monotonic) and strictly convex. 
        For $\log_\alpha$,
        \begin{align*}
            \frac{d}{dy}\log_\alpha(y)=y^{-\alpha}>0~,~\frac{d^2}{dy^2}\log_\alpha(y)=-\alpha y^{-\alpha-1}<0~.
        \end{align*}
        Hence, it is strictly increasing (monotonic) and strictly concave.
        \item Consider $x_1,x_2\in\R$.
        \begin{align}
        \label{eq:exp_alpha_add}
            \exp_\alpha(x_1)\exp_\alpha(x_2) = & ~[1+(1-\alpha)(x_1+x_2)+(1-\alpha)^2x_1x_2]_+^{\frac{1}{1-\alpha}}\notag\\
            = & ~ \exp_\alpha(x_1+x_2)\cdot[1+\frac{(1-\alpha)^2x_1x_2}{1+(1-\alpha)(x_1+x_2)}]_+^{\frac{1}{1-\alpha}}~.
        \end{align}
        If $\alpha\to1$, then $[\cdot\cdot\cdot]$ in \eqref{eq:exp_alpha_add} is $1$.
        Consider $y_1,y_2\in\R_+$.
        \begin{align}
        \label{eq:log_alpha_addition}
            \log_\alpha(y_1y_2) = & ~ \frac{y_1^{1-\alpha}y_2^{1-\alpha}-1}{1-\alpha}=\frac{(y_1^{1-\alpha}-1)+(y_2^{1-\alpha}-1)+(y_1^{1-\alpha}-1)(y_2^{1-\alpha}-1)}{1-\alpha}\notag\\
            = & ~\log_\alpha(y_1)+\log_\alpha(y_2)+(1-\alpha)\log_\alpha(y_1)\log(y_2)~.
        \end{align}
        If $\alpha\to1$, the last term in \eqref{eq:log_alpha_addition} is zero.
        \item For $[1+(1-\alpha)x]>0$, we have
        \begin{align*}
            \frac{d}{dx}\exp_\alpha(x)=[1+(1-\alpha)x]^{\frac{1}{1-\alpha}-1}=[\exp_\alpha(x)]^\alpha~.
        \end{align*}
        For $\log_\alpha(y)$, we obtain the result by differentiation easily.
        For $\alpha\to1$, they hold trivially.
    \end{itemize}
    This completes the proof.
\end{proof}

\begin{remark}
    Since $p_\mu^\alpha\log_\alpha(p_\mu)=(1-\alpha)^{-1}(p_\mu-p_\mu^\alpha)$, using $\sum_{\mu=1}^Mp_\mu=1$, we obtain
    \begin{align}
    \label{eq:tsallis_entropy_alpha_deformed}
        \Psi_\alpha[p]=-\sum_{\mu=1}^Mp_\mu^\alpha\log_\alpha(p_\mu).
    \end{align}
\end{remark}

\subsection{Proof of \texorpdfstring{\cref{thm:free_energy_tsallis_entropy}}{}}
\label{sec:proof_free_energy_tsallis_entropy}

\begin{theorem}[\cref{thm:free_energy_tsallis_entropy} Restated: Tsallis Entropy and Deformed Free Energy]
    If we replace the exponential functions and logarithm functions by $\alpha$-deformed functions, then we have the deformed log-partition function
    \begin{align*}
        -W_\alpha[J]=\beta\mathbb{E}^{(\alpha)}_{x\sim p_J}[\mathcal{H}[x]]-\frac{\Psi_\alpha[p_J]}{Z_\alpha^{\alpha-1}[J]\int_{\mathcal{X}}p_J^\alpha(x)~\dd\mu(x)}-\sum_{i=1}^dJ_i\mathbb{E}^{(\alpha)}_{x\sim p_J}[x_i],
    \end{align*}
    where $\mathbb{E}^{(\alpha)}_{x\sim p_J}[\cdot]\coloneq\int_{\mathcal{X}} p^\alpha_J(x)[\cdot]\dd \mu(x)/\int_{\mathcal{X}} p^\alpha_J(x)\dd \mu(x)$ is the escort expectation value, and
    \begin{align}
        \Gamma_\alpha=\beta\mathbb{E}^{(\alpha)}_{x\sim p_{J^\star}}[\mathcal{H}[x]]-\frac{\Psi_\alpha[p_{J^\star}]}{Z_\alpha^{\alpha-1}[J^\star]\int_{\mathcal{X}}p_{J^\star}^\alpha(x)~\dd\mu(x)},\qquad\text{with }~\varphi=\mathbb{E}^{(\alpha)}_{x\sim p_{J^\star}}[x],
    \end{align}
    where $J^\star=\arg\max_J\{\sum_{i=1}^dJ_i\varphi_i-W_\alpha[J]\}$.
    We assume $\mathcal{X}$ is finite-volume to avoid divergence.
\end{theorem}

\begin{proof}
    By \cref{def:alpha_deformed_exponential}, we express $W_\alpha[J]=\log_\alpha Z_\alpha[J]$ and
    \begin{align*}
        p_J(x)=\frac{\exp_\alpha\{-\beta\mathcal{H}[x]+\sum_{i=1}^dJ_ix_i\}}{Z_\alpha[J]}~,~Z_\alpha[J]=\int_{\mathcal{X}}\exp_\alpha\{-\beta\mathcal{H}[x]+\sum_{i=1}^dJ_ix_i\}~\dd\mu(x)~.
    \end{align*}
    Notice that since $\mathcal{X}$ is finite-volume, $Z_\alpha[J]$ is finite for appropriate $J$ and the following computation is reliable.
    Then we have
    \begin{align*}
        \log_\alpha p_J(x) = & ~
        \log_\alpha(\frac{\exp_\alpha\{-\beta\mathcal{H}[x]+\sum_{i=1}^dJ_ix_i\}}{Z_\alpha[J]})\\
        = & ~
        \frac{Z_\alpha^{\alpha-1}[J]\exp_\alpha^{1-\alpha}\{-\beta\mathcal{H}[x]+\sum_{i=1}^dJ_ix_i\}-1}{1-\alpha}\annot{By \cref{def:alpha_deformed_exponential}}\\
        = & ~
        \frac{Z_\alpha^{\alpha-1}[J]\{1+(1-\alpha)(-\beta\mathcal{H}[x]+\sum_{i=1}^dJ_ix_i)\}-1}{1-\alpha}\annot{By \cref{def:alpha_deformed_exponential}}\\
        = & ~
        Z^{\alpha-1}_\alpha[J](-\beta\mathcal{H}[x]+\sum_{i=1}^dJ_ix_i)+\frac{Z_\alpha^{\alpha-1}[J]-1}{1-\alpha}\\
        = & ~
        Z^{\alpha-1}_\alpha[J](-\beta\mathcal{H}[x]+\sum_{i=1}^dJ_ix_i)-Z_\alpha^{\alpha-1}[J]\cdot(\frac{Z^{1-\alpha}_\alpha[J]-1}{1-\alpha})\\
        = & ~
        Z_\alpha^{\alpha-1}[J](-\beta\mathcal{H}[x]+\sum_{i=1}^dJ_ix_i-W_\alpha[J])\annot{By \cref{def:alpha_deformed_exponential} and $W_\alpha[J]=\log_\alpha Z_\alpha[J]$}.
    \end{align*}
    By \eqref{eq:tsallis_entropy_alpha_deformed}, we express the Tsaliis entropy as
    \begin{align}
        \Psi_\alpha[p_J] = & ~
        -\int_{\mathcal{X}}p_J^\alpha(x)\log_\alpha(p_J(x))~\dd\mu(x)\notag\\
        = & ~
        -\int_{\mathcal{X}}p_J^\alpha(x)Z_\alpha^{\alpha-1}[J](-\beta\mathcal{H}[x]+\sum_{i=1}^dJ_ix_i-W_\alpha[J])\label{eq:tsallis_w_escort}.
    \end{align}
    This gives
    \begin{align*}
        -W_\alpha[J] = & ~
        \frac{-\Psi_\alpha[p_J]+\beta\int_{\mathcal{X}}p_J^\alpha(x)Z_\alpha^{\alpha-1}[J]\mathcal{H}[x]~\dd\mu(x)-\int_{\mathcal{X}}\sum_{i=1}^dp_J^\alpha(x)Z_\alpha^{\alpha-1}[J]J_ix_i}{Z^{\alpha-1}_\alpha[J]\int_{\mathcal{X}}p_J^\alpha(x)~\dd\mu(x)}\\
        = & ~
        \beta~\frac{\int_{\mathcal{X}}p_J^\alpha(x)\mathcal{H}[x]~\dd\mu(x)}{\int_{\mathcal{X}} p_J^\alpha(x)~\dd\mu(x)}-\frac{\Psi_\alpha[p_J]}{Z_\alpha^{\alpha-1}[J]\cdot\int_{\mathcal{X}}p_J^\alpha(x)~\dd\mu(x)}-\sum_{i=1}^dJ_i~\frac{\int_{\mathcal{X}}p_J^\alpha(x)x_i~\dd\mu(x)}{\int_{\mathcal{X}}p_J^\alpha(x)~\dd\mu(x)}.
    \end{align*}
    The first and last term are the escort expectation values $\mathbb{E}_{x\sim p_J}^{(\alpha)}[\cdot]$.
    Hence, it is the form of our first equation.
    After applying the Fenchel conjugation, we have
    \begin{align*}
        \Gamma_\alpha = & ~
        \sup_{J}\{\sum_{i=1}^dJ_i\varphi_i-W_\alpha[J]\}\\
        = & ~
        \sup_{J}\{\sum_{i=1}^dJ_i(\varphi_i-\mathbb{E}_{x\sim p_J}^{(\alpha)}[x_i])+\beta\mathbb{E}^{(\alpha)}_{x\sim p_J}[\mathcal{H}[x]]-\frac{\Psi_\alpha[p_J]}{Z_\alpha^{\alpha-1}[J]\cdot\int_{\mathcal{X}}p_J^\alpha(x)~\dd\mu(x)}\}\\
        = & ~
        \beta\mathbb{E}^{(\alpha)}_{x\sim p_{J^\star}}[\mathcal{H}[x]]-\frac{\Psi_\alpha[p_{J^\star}]}{Z_\alpha^{\alpha-1}[J^\star]\cdot\int_{\mathcal{X}}p_{J^\star}^\alpha(x)~\dd\mu(x)},
    \end{align*}
    where $\varphi=\mathbb{E}_{x\sim p_{J^\star}}^{(\alpha)}[x]$.
    This proves the theorem.
\end{proof}

\begin{remark}
    The finite-volume assumption is only to ensure that the $\alpha$-deformed partition function is well-defined.
    Specifically, if we consider whole $\R^d$ with $\alpha=2$ and $a>0$, then
    \begin{align*}
        \int_{\R^d}\exp_2\{-a\|x\|_2^2\}~\dd\mu(x)=\int_{\R^d}\frac{\dd\mu(x)}{1+a\|x\|_2^2}=S_{d-1}\int_0^\infty \frac{r^{d-1}}{1+ar^2}~\dd r,
    \end{align*}
    where $S_{d-1}$ is the surface area of a unity sphere $S^{d-1}$.
    Hence, only if $d<2$ is the integral finite.
    This leads to the consideration of a finite-volume domain $\mathcal{X}$.
    Passing to a finite-volume domain $\mathcal{X}\subset\R^d$ introduces a boundary.
    It removes the divergence but does not guarantee strict convexity of $W_\alpha[J]$ nor the essential smoothness at the boundary of $\text{dom}W_\alpha$.
    Hence, according to \cref{def:legendre}, the transformation here should be understood as a Fencehl conjugate on $\text{dom}W_\alpha$, not as a Legendre transform necessarily.
\end{remark}

\subsection{Proof of \texorpdfstring{\cref{prp:expectation_values_and_macroscopic_states}}{}}
\label{sec:proof_expectation_values_and_macroscopic_states}

\begin{proposition}[\cref{prp:expectation_values_and_macroscopic_states} Restated: Query-Probability Relation]
    A query state vector relates to a probability function via
    \begin{align*}
        x_i=\frac{1}{\sqrt{d}}\sum_{\mu=1}^M\xi_i^\mu p_\mu+\frac{1}{\beta}J_i.
    \end{align*}
\end{proposition}

\begin{proof}
    In the deterministic limit, we have
    \begin{align*}
        Z_\alpha[J]=\int_{\mathcal{X}}\exp_\alpha\{-\beta\tilde{\mathcal{H}}_{\text{MHM}}[x,p,\xi]+\sum_{i=1}^dJ_ix_i\}~\dd\mu(x)=\exp_\alpha\{-\beta\tilde{\mathcal{H}}_{\text{MHM}}[x,p,\xi]+\sum_{i=1}^dJ_ix_i\},
    \end{align*}
    since we identify $x$ with the mean.
    The query state vector that maximizes $Z_\alpha[J]$ is the deterministic $x$ itself.
    Hence, we have
    \begin{align*}
        0 = & ~
        \frac{\partial}{\partial x_i}\{\beta\tilde{\mathcal{H}}_{\text{MHM}}[x,p,\xi]-\sum_{j=1}^dJ_jx_j\}\\
        = & ~
        \beta x_i-\frac{\beta}{\sqrt{d}}\sum_{\mu=1}^M\xi_i^\mu p_\mu-J_i.
    \end{align*}
    This proves the proposition.
\end{proof}

\subsection{Proof of \texorpdfstring{\cref{thm:free_energy_mhm}}{}}
\label{sec:proof_free_energy_mhm}

\begin{proposition}[\cref{thm:free_energy_mhm} Restated: Free Energy of MHM]
    The joint free energy of the MHM is
    \begin{align*}
        \tilde{\Gamma}_\alpha[x,p,\xi] = \frac{\beta}{2}\sum_{i=1}^dx_i^2-\beta\sum_{\mu=1}^Mz_\mu p_\mu.
    \end{align*}
    Hence,
    \begin{align*}
        \Gamma_\alpha[x,\xi]=\frac{\beta}{2}\sum_{i=1}^dx_i^2-\beta\max_{\mu\in[1,M]}~z_\mu.
    \end{align*}
\end{proposition}

\begin{proof}
    The form of the partition function is 
    \begin{align*}
        Z_\alpha[J,\xi]=\exp_\alpha\{W_\alpha[J,\xi]\}=\int_{\mathcal{X}}\int_{\Delta^M}\exp_\alpha\{-\frac{\beta}{2}\sum_{i=1}^dx_i^2+\beta\sum_{\mu=1}^Mz_\mu p_\mu+\sum_{i=1}^dJ_ix_i\}~\dd\mu(p)\dd\mu(x)~.
    \end{align*}
    According to the form of the Tsallis entropy \eqref{eq:tsallis_entropy_alpha_deformed} and \eqref{eq:tsallis_w_escort}, we have
    \begin{align*}
        \Psi_\alpha[\rho_J] = & ~ 
        -\int_{\mathcal{X}}\int_{\Delta^M}\rho_J^\alpha(x,p)\log_\alpha[\rho_J(x,p)]~\dd\mu(p)\dd\mu(x)\\
        = & ~ 
        \frac{\beta}{2}Z_\alpha^{\alpha-1}[J,\xi]\sum_{i=1}^d\int_{\mathcal{X}}\int_{\Delta^M}\rho_J^\alpha(x,p)x_i^2~\dd\mu(p)\dd\mu(x)\\
        & ~
        -\beta Z_\alpha^{\alpha-1}[J,\xi]\sum_{\mu=1}^M\int_{\mathcal{X}}\int_{\Delta^M}\rho_J^\alpha(x,p)p_\mu z_\mu(x)~\dd\mu(p)\dd\mu(x)\\
        & ~
        -Z_\alpha^{\alpha-1}[J,\xi]\sum_{i=1}^dJ_i\int_{\mathcal{X}}\int_{\Delta^M}\rho_J^\alpha(x,p)x_i~\dd\mu(p)\dd\mu(x)\\
        & ~
        +Z_\alpha^{\alpha-1}[J,\xi]W_\alpha[J,\xi]\int_{\mathcal{X}}\int_{\Delta^M}\rho_J^\alpha(x,p)~\dd\mu(p)\dd\mu(x).
    \end{align*}
    It yields
    \begin{align}
    \label{eq:w_mhm}
        -W_\alpha[J,\xi] = & ~
        \frac{\beta}{2}\sum_{i=1}^d\mathbb{E}^{(\alpha)}_{(x,p)\sim\rho_J}[x_i^2]-\beta\sum_{\mu=1}^M\mathbb{E}^{(\alpha)}_{(x,p)\sim \rho_J}[z_\mu p_\mu]\notag\\
        & ~
        -\frac{\Psi_\alpha[\rho_J]}{Z_\alpha^{\alpha-1}[J,\xi]\int_{\mathcal{X}}\int_{\Delta^M}\rho_J^\alpha(x,p)~\dd\mu(p)\dd\mu(x)}
        -\sum_{i=1}^dJ_i\mathbb{E}^{(\alpha)}_{(x,p)\sim\rho_J}[x_i].
    \end{align}
    Since $x$ and $p$ are fixed, so the probability density function concentrates to an indicator, i.e.,
    \begin{align*}
        \rho_J(x,p)\to\one_{\{x,\tilde{\mathbb{E}}^{(\alpha)}_{x\sim p_{J}}[x]\}}\one_{\{p,\tilde{\mathbb{E}}^{(\alpha)}_{p\sim p_{\Delta,J}}[p]\}},
    \end{align*}
    where
    \begin{align}
    \label{eq:tilde_exp_x}
        \tilde{\mathbb{E}}^{(\alpha)}_{x\sim p_{J}}[x]\coloneq\mathbb{E}^{(\alpha)}_{(x,p)\sim\rho_J}[x]=\frac{\int_{\mathcal{X}}C(x)xp^\alpha_J(x)~\dd\mu(x)}{\int_{\mathcal{X}}C(x)p^\alpha_J(x)~\dd\mu(x)},\qquad\text{with }~C(x)\coloneq\int_{\Delta^M}\rho_J^\alpha(p|x)~\dd\mu(p),
    \end{align}
    and
    \begin{align}
        \label{eq:tilde_exp_p}
        \tilde{\mathbb{E}}_{p\sim p_{\Delta,J}}^{(\alpha)}[p_\mu]\coloneq\mathbb{E}^{(\alpha)}_{(x,p)\sim\rho_J}[p_\mu]=\frac{\int_{\Delta^M}C(p)p_\mu p^\alpha_{\Delta,J}(p)~\dd\mu(p)}{\int_{\Delta^M}C(p)p^\alpha_{\Delta,J}(p)~\dd\mu(p)},\qquad\text{with }~C(p)\coloneq\int_{\mathcal{X}}\rho_J^\alpha(x|p)~\dd\mu(x).
    \end{align}
    Then
    \begin{align*}
        Z_\alpha[J,\xi]\to \exp_\alpha\{-\beta\tilde{\mathcal{H}}_{\text{MHM}}[\tilde{\mathbb{E}}^{(\alpha)}_{x\sim p_{J}}[x],\tilde{\mathbb{E}}^{(\alpha)}_{p\sim p_{\Delta,J}}[p],\xi]+\sum_{i=1}^dJ_i\tilde{\mathbb{E}}^{(\alpha)}_{x\sim p_{J}}[x_i]\}.
    \end{align*}
    Hence, by the Fenchel conjugate, we obtain
    \begin{align*}
        \tilde{\Gamma}_\alpha = & ~
        \sup_{J\in\text{dom}W_\alpha}\{\sum_{i=1}^dJ_i\varphi_i-W_\alpha[J,\xi]\}\\
        = & ~
        \sup_{J\in\text{dom}W_\alpha}\Big\{\sum_{i=1}^dJ_i(\varphi_i-\tilde{\mathbb{E}}_{x\sim p_J}^{(\alpha)}[x_i])+\frac{\beta}{2}\sum_{i=1}^d(\tilde{\mathbb{E}}_{x\sim p_J}^{(\alpha)}[x_i^2])-\beta\sum_{\mu=1}^M\tilde{\mathbb{E}}_{x\sim p_J}^{(\alpha)}[z_\mu]\tilde{\mathbb{E}}_{p\sim p_{\Delta,J}}^{(\alpha)}[p_\mu]\\
        & ~-\frac{\Psi_\alpha[\one]}{\exp^{\alpha-1}_\alpha\{-\beta\tilde{\mathcal{H}}_{\text{MHM}}[\tilde{\mathbb{E}}^{(\alpha)}_{x\sim p_{J}}[x],\tilde{\mathbb{E}}^{(\alpha)}_{p\sim p_{\Delta,J}}[p],\xi]+\sum_{i=1}^dJ_i\tilde{\mathbb{E}}^{(\alpha)}_{x\sim p_{J}}[x_i]\}}\Big\}\\
        = & ~ 
        \frac{\beta}{2}\sum_{i=1}^d(\tilde{\mathbb{E}}_{x\sim p_{J^\star}}^{(\alpha)}[x_i^2])-\beta\sum_{\mu=1}^M\tilde{\mathbb{E}}_{x\sim p_{J^\star}}^{(\alpha)}[z_\mu]\tilde{\mathbb{E}}_{p\sim p_{\Delta,J^\star}}^{(\alpha)}[p_\mu],\annot{By $\Psi_\alpha[\one]=0$}
    \end{align*}
    where $J^\star$ is the maximizing source.
    Since we identify $\tilde{\mathbb{E}}_{x\sim p_{J^\star}}^{(\alpha)}[x]$ with $x$, $\tilde{\mathbb{E}}_{x\sim p_{J^\star}}^{(\alpha)}[x^2]$ with $x^2$, and $\tilde{\mathbb{E}}_{p\sim p_{\Delta,J^\star}}^{(\alpha)}[p]$ with $p$, we have
    \begin{align*}
        \tilde{\Gamma}_\alpha[x,p,\xi]=\frac{\beta}{2}\sum_{i=1}^dx_i^2-\beta\sum_{\mu=1}^Mz_\mu p_\mu.
    \end{align*}
    Notice that
    \begin{align*}
        \inf_{p\in\Delta^M}\{-\sum_{\mu=1}^Mz_\mu p_\mu\}=-\max_{\mu\in[1,M]}~z_\mu.
    \end{align*}
    After substituting into the joint free energy to eliminate the dependence of $p$, we prove the second equation in this proposition.
\end{proof}

\subsection{Proof of \texorpdfstring{\cref{prp:eom_attention}}{}}
\label{sec:proof_eom_attention}

\begin{proposition}[\cref{prp:eom_attention} Revisited: Emergent Attention Mechanism as EOMs in MHMs]
    Attention emerges as follows.
    
    \begin{itemize}
        \item If $\alpha\to\infty$, it corresponds to a hardmax function.
        \begin{align*}
            p_\mu=\text{hardmax}(z_\mu)=\begin{cases}1~,~&\text{for maximal }z_\mu\\0~,~&\text{otherwise}\end{cases}.
        \end{align*}
        
        \item If $\alpha=2$, it corresponds to a sparsemax function \cite{martins2016softmax}.
        \begin{align*}
            p_\mu= \text{Sparsemax}_\beta(z_\mu)=\frac{\beta}{2}[z_\mu-\theta]_+,\qquad\text{with }~\frac{\beta}{2}\sum_{\mu=1}^M[z_\mu-\theta]_+=1.
        \end{align*}

        \item If $1<\alpha<2$, it corresponds to an entmax family \cite{peters2019sparse}.
        
        \item If $\alpha\to1$, it corresponds to a softmax function \cite{blondel2020learning}.
        \begin{align*}
            p_\mu= \Softmax_\beta(z_\mu)= \frac{\exp\{\beta z_\mu\}}{\sum_{\nu=1}^M\exp\{\beta z_\nu\}}.
        \end{align*}
        
    \end{itemize}
\end{proposition}

\begin{proof}
    We include the proof for completeness, adapting the standard argument from \cite{martins2016softmax,peters2019sparse,blondel2020learning} to our notation.
    The training trajectory follows $p_\mu=\arg\max_{p\in\Delta^M}\{\beta\sum_{\mu=1}^M z_\mu p_\mu+\Psi_\alpha[p]\}$.
    We derive the physical meanings of this optimization problem via the Lagrangian function:
    \begin{align*}
        \frac{\partial}{\partial p_\mu}\{\sum_{\nu=1}^Mz_\nu p_\nu+\frac{1}{\beta}\cdot\frac{1}{\alpha-1}(1-\sum_{\nu=1}^Mp_\nu^\alpha)+\theta(1-\sum_{\nu=1}^Mp_\nu)+\sum_{\nu=1}^M\lambda_\nu p_\nu\}=0,
    \end{align*}
    where we use the form of Tsallis entropy, \eqref{eq:tsallis_entropy_function}, and the Lagrangian multipliers, $\lambda$ and $\lambda_\nu\geq0$, to include the simplex constraints. 
    This gives the complementary slackness, $\lambda_\mu p_\mu=0$, and the stationary condition,
    \begin{align*}
        z_\mu-\frac{\alpha}{\beta(\alpha-1)}p_\mu^{\alpha-1}-\theta+\lambda_\mu=0~\implies~\begin{cases}p_\mu=\{\beta(1-\frac{1}{\alpha})[z_\mu-\theta]_+\}^{\frac{1}{\alpha-1}} & , ~ \text{if}~\lambda_\mu=0\\p_\mu=0 & , ~ \text{if}~\lambda_\mu>0\end{cases}~,
    \end{align*}
    where $[u]_+\coloneq\max(u,0)$.
    For $\alpha\to\infty$ and $\alpha=2$, we obtain the forms of $p_\mu$ by inserting $\alpha$ directly.
    For $\alpha=1$, we use the Shannon form, $\eqref{eq:shannon_entropy}$.
    Hence, for non-zero $p_\mu$, we have
    \begin{align*}
        z_\mu+\frac{1}{\beta}\frac{\partial}{\partial p_\mu}(-\sum_{\nu=1}^Mp_\nu\log p_\nu)-\theta=0.
    \end{align*}
    This gives
    \begin{align*}
        p_\mu=\exp\{\beta(z_\mu-\theta)-1\}.
    \end{align*}
    The normalization condition gives $1=\sum_{\mu=1}^Mp_\mu=\exp\{-\beta\theta\}\sum_{\mu=1}^M\exp\{\beta z_\mu-1\}$.
    This yields $\theta=\beta^{-1}\log\sum_{\mu=1}^M\exp\{\beta z_\mu-1\}$.
    After plugging this into $p_\mu$, we obtain the softmax function.
    For $2>\alpha>1$, $(\alpha-1)^{-1}>1$.
    Hence, $p_\mu$ is the entmax family interpolating between softmax and sparsemax.
\end{proof}

\subsection{Proof of \texorpdfstring{\cref{prp:coex_point}}{}}
\label{sec:proof_coex_point}

\begin{proposition}[\cref{prp:coex_point} Restated: Temperature Coexistence Point for Sparse Model]
    Let $a,b$ be two distinct branches and $\alpha=2$.
    Denote $k_a$ and $k_b$ as the exact support sizes (see \cref{lem:threshold_p_mu}) of branches $a$ and $b$, respectively.
    Based on the statistical mechanics framework, the coexistence point satisfies
    \begin{align*}
        \sum_{\mu,\nu\in a}[z^\star_\mu-\theta]_+[z^\star_\nu-\theta]_+K_{\mu\nu}(\beta_{\text{coexist}};\beta_0)=\sum_{\mu,\nu\in b}[z^\star_\mu-\theta]_+[z^\star_\nu-\theta]_+K_{\mu\nu}(\beta_{\text{coexist}};\beta_0),
    \end{align*}
    where $K_{\mu\nu}(\beta)\coloneq\one_{\{\mu,\nu\}}-\frac{\beta}{2}G_{\mu\nu}$, and $G_{\mu\nu}\coloneq\sum_{i=1}^d\xi_i^\mu\xi_i^\nu/d$ is the Gram matrix.
    $\beta_{\text{coexist}}$ does not admit a closed-form expression. 
    However, if the distribution is uniform on the simplex and the system satisfies the well-separation condition, i.e.,
    \begin{align*}
        G_{\mu\nu}=\begin{cases}
            1~&,~\mu=\nu\\\epsilon\ll1~&,~\mu\neq\nu
        \end{cases},
    \end{align*}
    then for $k_a=1$ and $k_b=2$, we have the analytic solution $\beta_{\text{coexist}}=2/(1-\epsilon)$.
    At this point, $\beta=\beta_{\text{coexist}}$, $p^\star$ is degenerate and every $p^\star$ is a global minimum of the joint free energy $\tilde{\Gamma}_{\alpha=2}^{(\text{mod})}$.
\end{proposition}

\begin{proof}
    Consider a system with the minimizer, $p^\star_\mu$.
    Using that at a stable point, $J_i=0$ for all $i=1,...,d$, we have
    \begin{align*}
        z^\star_\mu=\frac{1}{\sqrt{d}}\sum_{i=1}^d\xi_i^\mu x^\star_i=\frac{1}{\sqrt{d}}\sum_{i=1}^d\xi_i^\mu(\frac{1}{\sqrt{d}}\sum_{\nu=1}^M\xi_i^\nu p^\star_\nu+\frac{1}{\beta}J_i)=\sum_{\nu=1}^MG_{\mu\nu}p^\star_\nu.
    \end{align*}
    Moreover,
    \begin{align*}
        \sum_{i=1}^d(x^\star_i)^2=\frac{1}{d}\sum_{i=1}^d\sum_{\mu,\nu=1}^M\xi_i^\mu \xi_i^\nu p^\star_\mu p^\star_\nu=\sum_{\mu,\nu=1}^MG_{\mu\nu}p^\star_\mu p^\star_\nu.
    \end{align*}
    Suppose the support is $k$, we denote the indices of the non-zero components as $\tilde{\mu}=1,...,k$. 
    Then we have
    \begin{align*}
        \tilde{\Gamma}^{(\text{mod})}_\alpha[x^\star,p^\star,\xi] 
        = & ~ 
        \frac{\beta}{2}\sum_{i=1}^d(x^\star_i)^2-\beta\sum_{\tilde{\mu}=1}^kz^\star_{\tilde{\mu}} p^\star_{\tilde{\mu}}-(1-\sum_{\tilde{\mu}=1}^k(p^\star_{\tilde{\mu}})^2)\\
        = & ~ -1-\sum_{\tilde{\mu},\tilde{\nu}=1}^kp^\star_{\tilde{\mu}}(\frac{\beta}{2}G_{\tilde{\mu}\tilde{\nu}}
        -\one_{\{\tilde{\mu},\tilde{\nu}\}})p^\star_{\tilde{\nu}}
    \end{align*}
    From the uniform and well-separation condition, we write $G_{\tilde{\mu}\tilde{\nu}}=(1-\epsilon)\one_{\{\tilde{\mu},\tilde{\nu}\}}+\epsilon$. 
    Then this gives $\sum_{\tilde{\mu},\tilde{\nu}=1}^kp^\star_{\tilde{\mu}}G_{\tilde{\mu}\tilde{\nu}}p^\star_{\tilde{\nu}}=(1-\epsilon)\|p^\star\|_2^2+\epsilon$.
    In this way, we have
    \begin{align}
    \label{eq:energy_uniform_well_separation}
        \tilde{\Gamma}_\alpha^{(\text{mod})}[x^\star,p^\star,\xi] =\Big(1-\frac{\beta}{2}(1-\epsilon)\Big)\|p^\star\|^2-1-\frac{\beta}{2}\epsilon.
    \end{align}
    By \cref{lem:threshold_p_mu}, we express the minimizer as $p^\star_{\tilde{\mu}}=\frac{\beta}{2}(z^\star_{\tilde{\mu}}-\theta)$ for $\alpha=2$.
    By matching the free energy on branch $a$ and branch $b$, we prove the first equation.
    For the well-separation and uniform distribution conditions, we have
    \begin{align*}
        \tilde{\Gamma}^{(\text{mod})}_\alpha[x^\star_a,p^\star,\xi] 
        =-1-\frac{\beta_{\text{coexist}}}{2}\times 1+1=-\frac{\beta_{\text{coexist}}}{2}.
    \end{align*}
    For branch $b$,
    \begin{align*}
        \tilde{\Gamma}^{(\text{mod})}_\alpha[x^\star_b,p^\star,\xi] 
        =-1-\begin{pmatrix}\frac{1}{2}&\frac{1}{2}\end{pmatrix}\begin{pmatrix}\frac{\beta_{\text{coexist}}}{2}-1&\frac{\beta_{\text{coexist}}\epsilon}{2}\\\frac{\beta_{\text{coexist}}\epsilon}{2}&\frac{\beta_{\text{coexist}}}{2}-1\end{pmatrix}\begin{pmatrix}\frac{1}{2}\\\frac{1}{2}\end{pmatrix}=-\frac{1}{2}-\frac{\beta_{\text{coexist}}}{4}(1+\epsilon).
    \end{align*}
    After equating these two energies, we obtain $\beta_{\text{coexist}}=2/(1-\epsilon)$.
    Notice that when $\beta=\beta_{\text{coexist}}$, all $p^\star$ are degenerate, i.e., they are all global minima in free energy \eqref{eq:energy_uniform_well_separation}.
    This proves the proposition.
\end{proof}

\subsection{Proof of \texorpdfstring{\cref{prp:spinodal_point}}{}}
\label{sec:proof_spinodal_point}

\begin{proposition}[\cref{prp:spinodal_point} Restated: Temperature Spinodal Point for Sparse Model]
    Consider the sparse case $\alpha=2$ and a branch $a$ with exact support size $k$.
    Then the spinodal critical value is
    \begin{align*}
        \beta_{\text{jump}}=\frac{2}{\lambda_{\text{max}}(PGP)},
    \end{align*}
    where $P_{\mu\nu}\coloneq\one_{\{\mu,\nu\}}-k^{-1}$, and $G$ is the Gram matrix.
    If the distribution is uniform and the system satisfies the well-separation condition, \eqref{eq:well-separation}, then we have the analytic solution $\beta_{\text{jump}}=2/(1-\epsilon)$.
\end{proposition}

\begin{proof}
    Since for $\alpha=2$, the minimizer satisfies $p^\star=\frac{\beta}{2}(z^\star_\mu-\theta)$, where $\mu=1,...,k$, we have
    \begin{align*}
        \frac{\partial p^\star_\mu}{\partial z^\star_\nu}=\frac{\beta}{2}\one_{\{\mu,\nu\}}-\frac{\beta}{2}\frac{\partial}{\partial z^\star_\nu}(\frac{\sum_{\rho=1}^kz^\star_\rho-1}{k})=\frac{\beta}{2}(\one_{\{\mu,\nu\}}-\frac{1}{k})=\frac{\beta}{2}P_{\mu\nu}.
    \end{align*}
    Then the gradient of the free energy at the local stable point, $x^\star$, is
    \begin{align*}
        \frac{\partial}{\partial x_i}\Gamma_\alpha^{(\text{mod})}[x^\star,\xi] = & ~ \beta x^\star_i-\frac{\partial}{\partial x_i}\sup_{p\in\Delta^M}\{\beta\sum_{\mu=1}^kz_\mu p_\mu+1-\sum_{\mu=1}^kp_\mu^2\}|_{x=x^\star}\\
        = & ~ 
        \beta x^\star_i-\frac{\beta}{\sqrt{d}}\sum_{\mu=1}^k\xi_i^\mu p^\star_\mu.
    \end{align*}  
    The Hessian is
    \begin{align}
        \frac{\partial^2}{\partial x_i\partial x_j}\Gamma_\alpha^{(\text{mod})}[x^\star,\xi] = & ~ 
        \beta\one_{\{i,j\}}-\frac{\beta}{\sqrt{d}}\sum_{\mu,\nu=1}^k\xi_i^\mu\frac{\partial p^\star_\mu}{\partial z^\star_\nu}\frac{\partial z^\star_\nu}{\partial x_j}|_{x=x^\star}\notag\\
        = & ~ 
        \beta\one_{\{i,j\}}-\frac{\beta^2}{2d}\sum_{\mu,\nu=1}^k\xi_i^\mu \xi_j^\nu P_{\mu\nu}.\label{eq:hessian_x}
    \end{align}
    The smallest Hessian eigenvalue vanishes when $\frac{\beta_{\text{jump}}}{2d}\lambda_{\text{max}}(\sum_{\mu,\nu=1}^k\xi_i^\mu P_{\mu\nu}\xi_j^\nu)=1$.
    Notice that $P_{\mu\nu}$ is a projector, i.e.,
    \begin{align*}
        \sum_{\nu=1}^kP_{\mu\nu}P_{\nu\rho}=P_{\mu\rho}.
    \end{align*}
    Suppose $u$ is the eigenvector such that $\sum_{j=1}^d\sum_{\mu,\nu=1}^k\xi_i^\mu P_{\mu\nu}\xi_j^\nu u_j=\lambda u_i$.
    We construct 
    \begin{align*}
        C_{\mu\nu}\coloneq\sum_{\rho,\sigma=1}^kP_{\mu\rho}(\sum_{i=1}^d\xi_i^\rho\xi_i^\sigma)P_{\sigma\nu}=d\sum_{\rho,\sigma=1}^kP_{\mu\rho}G_{\rho\sigma}P_{\sigma\nu}.
    \end{align*}
    In fact, its eigenvector is $v_\mu=\sum_{\nu=1}^k\sum_{i=1}^dP_{\mu\nu}\xi_i^\nu u_i$.
    In this way,
    \begin{align*}
        (Cv)_\mu = & ~
        d\sum_{\nu,\rho,\sigma=1}^k\sum_{\kappa=1}^k\sum_{i=1}^dP_{\mu\rho}G_{\rho\sigma}P_{\sigma\nu}P_{\nu\kappa}\xi_i^\kappa u_i\\
        = & ~ 
        \sum_{\rho,\sigma=1}^k\sum_{\kappa=1}^k\sum_{i,j=1}^dP_{\mu\rho}\xi_j^\rho\xi_j^\sigma P_{\sigma\kappa}\xi_i^\kappa u_i\annot{By $P$ being a projector}\\
        = & ~
        \sum_{\rho=1}^k\sum_{j=1}^d P_{\mu\rho}\xi_j^\rho\lambda u_j\annot{By $\sum_{j=1}^d\sum_{\mu,\nu=1}^k\xi_i^\mu P_{\mu\nu}\xi_j^\nu u_j=\lambda u_i$}\\
        = & ~
        \lambda v_\mu.
    \end{align*}
    Hence, its eigenvalue is the same as which of $\sum_{\mu,\nu=1}^k\xi_i^\mu P_{\mu\nu}\xi_j^\nu$.
    That is, 
    \begin{align*}
        \lambda_{\text{max}}(\sum_{\mu,\nu=1}^k\xi_i^\mu P_{\mu\nu}\xi_j^\nu)=d\cdot\lambda_{\text{max}}(\sum_{\rho,\sigma=1}^kP_{\mu\rho}G_{\rho\sigma}P_{\sigma\nu}).
    \end{align*}
    Finally, we obtain
    \begin{align*}
        \beta_{\text{jump}}=\frac{2d}{\lambda_{\text{max}}(\sum_{\mu,\nu=1}^k\xi_i^\mu P_{\mu\nu}\xi_j^\nu)}=\frac{2}{\lambda_{\text{max}}(\sum_{\rho,\sigma=1}^kP_{\mu\rho}G_{\rho\sigma}P_{\sigma\nu})}.
    \end{align*}
    If the system satisfies the well-separation condition, $G_{\mu\nu}=(1-\epsilon)\one_{\{\mu,\nu\}}+\epsilon$, with $\epsilon\ll1$.
    Then
    \begin{align*}
        \sum_{\rho,\sigma=1}^kP_{\mu\rho}G_{\rho\sigma}P_{\sigma\nu} = & ~ 
        (1-\epsilon)\sum_{\rho,\sigma=1}^kP_{\mu\rho}\one_{\{\rho,\sigma\}}P_{\sigma\nu}+\epsilon\sum_{\rho,\sigma=1}^kP_{\mu\rho} P_{\sigma\nu}\\
        = & ~
        (1-\epsilon)P_{\mu\nu}\annot{By $P$ being a projector and $P(1~1~\cdot\cdot\cdot~1)^\top=0$}
    \end{align*}
    Let the vector space be $I\oplus I^\perp$, where $I\coloneq\text{span}\{(1,...,1)^\top\}$.
    Since the eigenvalue is $(1-\epsilon)$ on $I^\perp$ and zero on $I$, the largest eigenvalue is $(1-\epsilon)$.
    After substituting into the expression of $\beta_{\text{jump}}$, we prove the proposition.
\end{proof}

\clearpage
\section{Fluctuated Version and Thermal Attention}
\label{sec:fluctuated_version}

In this section, we revisit the modern Hopfield models but with the fluctuated variables, $x,p,$ and $z$.
First, we study how the temperature influences the fluctuation property.
Then we move on to discuss the modern Hopfield models.

\subsection{Physics of Temperature}
In this section, we discuss the relation of $\varphi$ and $x_J$ first.
Then we introduce the role that temperature plays in the perturbative analysis.

\begin{proposition}[Saddle Points as Stationary Points in the Zero-Temperature Limit \cite{Landau:1980mil}]
    Consider the zero-temperature limit, $\beta\to\infty$. If $(-\beta\mathcal{H}[x]+\sum_{i=1}^dJ_ix_i)$ has a unique maximizer, then the saddle points approach to the stationary points, $\varphi_i=(x_J)_i$.
\end{proposition}

\begin{proof}
    We repeat the proof from \cite{Landau:1980mil,dembo2010largedeviation} for completeness.
    We denote the maximizer as $x_J$. 
    Since it is unique, for any $x_i\neq(x_J)_i$, we have
    \begin{align*}
        \frac{p_J(x)}{p_J(x_J)}=\exp\{-\beta(\mathcal{H}[x]-\mathcal{H}[x_J])+\sum_{i=1}^dJ_i[x_i-(x_J)_i]\}\coloneq  e^{-\beta\Delta+C}~,
    \end{align*}
    where $\Delta>0$, and $C$ is a finite value. 
    Then in the zero-temperature limit, $p_J(x)\to p_J(x_J)$.
    Thus, $\varphi_i=\mathbb{E}_{p_J}[x_i]\to(x_J)_i$.
\end{proof}

\begin{proposition}[Fluctuation and Loop Expansion \cite{Thouless01031977}]
\label{prp:fluctuation}
    The free energy captures noise and fluctuations. 
    The fluctuation term behaves like a data-dependent regularizer. 
    Specifically, if we expand $x$ around $\varphi$, then
    \begin{align}
    \label{eq:loop-expansion}
        \Gamma[\varphi] = & ~\beta\mathcal{H}[\varphi]-\frac{d}{2}\log(2\pi)+\frac{1}{2}\log\det(\beta\nabla^2_x\mathcal{H}[\varphi])+O(\beta^{-1})\\
        \coloneq & ~ \Gamma_{\text{tree}}[\varphi]+\Gamma_{\text{1-loop}}[\varphi]+\Gamma_{\text{2-loop}}[\varphi]+\cdot\cdot\cdot~.
    \end{align}
    This free energy admits higher-order (higher-loop) corrections arising from fluctuations, which we denote as $O(\beta^{-1})$.
    Specifically, the two-loop is of order $\beta^{-1}$, three-loop is of order $\beta^{-2}$, and so forth.
    Conventionally, we ignore the constant term, $d\log(2\pi)/2$, since it is independent of $\varphi$ and does not affect the dynamics.

\end{proposition}

\begin{proof}
    We include the proof for completeness, adapting the standard argument from \cite{Thouless01031977,Zinn-Justinqft}. 
    By \cref{def:log_partition_effective_action}, we write
    \begin{align*}
        \exp\{-\Gamma[\varphi]\}& = ~ \exp\{W[J]-\sum_{i=1}^dJ_i\varphi_i\}\\
        & = ~ \int_{\mathcal{X}}\exp\{-\beta\mathcal{H}[x]+\sum_{i=1}^dJ_i(x_i-\varphi_i)\}~\dd\mu(x)~.
    \end{align*}
    Then we expand the right hand side around the stationary point, $\varphi$, perturbatively:
    \begin{align*}
        \exp\{-\Gamma[\varphi]\} = & ~\int_{\mathcal{X}}\exp\{-\beta\mathcal{H}[\varphi]-\beta\sum_{i=1}^d\frac{\partial\mathcal{H}}{\partial x_i}|_{x=\varphi}(x_i-\varphi_i)\\
        & ~ 
        -\frac{\beta}{2}\sum_{i,j=1}^d\frac{\partial^2\mathcal{H}}{\partial x_i\partial x_j}|_{x=\varphi}(x_i-\varphi_i)(x_j-\varphi_j)
        \\
        & ~ 
        -\frac{\beta}{3!}\sum_{i,j,k=1}^d\frac{\partial^3\mathcal{H}}{\partial x_i\partial x_j\partial x_k}|_{x=\varphi}(x_i-\varphi_i)(x_j-\varphi_j)(x_k-\varphi_k)\\
        & ~ +O(\|x-\varphi\|_2^4)\}\exp\{\sum_{i=1}^dJ_i(x_i-\varphi_i)\}~\dd\mu(\mathbf{x})\\
        = & ~ \exp\{-\beta\mathcal{H}[\varphi]\}\int_{\mathcal{X}}\exp\{-\frac{\beta}{2}\sum_{i,j=1}^d(x_i-\varphi_i)(\nabla^2_x\mathcal{H}[\varphi])_{ij}(x_j-\varphi_j)\} \cdot (1-\frac{\beta}{3!}\times\\
        & ~ \sum_{k,\ell,m=1}^d\frac{\partial^3\mathcal{H}}{\partial x_k\partial x_\ell\partial x_m}|_{x=\varphi}(x_k-\varphi_k)(x_\ell-\varphi_\ell)(x_m-\varphi_m)+O(\beta\|x-\varphi\|_2^4))~\dd\mu(x)\annot{By \cref{def:stationary_eq}}
    \end{align*}
    where $\|x'-\varphi\|_2^2\coloneq\sum_{i=1}^d(x'_i-\varphi_i)^2$.
    Firstly,
    \begin{align*}
        \int_{\mathcal{X}}\exp\{-\frac{\beta}{2}\sum_{i,j=1}^d(x_i-\varphi_i)(\nabla^2_x\mathcal{H}[\varphi])_{ij}(x_j-\varphi_j)\}~\dd\mu(x)=(2\pi)^{d/2}(\det(\beta\nabla^2_x\mathcal{H}[\varphi]))^{-1/2}\annot{By Gaussian Formula}~.
    \end{align*}
    Since the odd function vanishes after the integration, the first leading correction is of $O(\|x-\varphi\|^4_2)$.
    We focus on this integration:
    \begin{align*}
        & ~ \int_{\mathcal{X}}\exp\{-\frac{\beta}{2}\sum_{i,j=1}^d(x_i-\varphi_i)(\nabla^2_x\mathcal{H}[\varphi])_{ij}(x_j-\varphi_j)\}O(\beta\|x-\varphi\|_2^4)~\dd\mu(x)\\
        = & ~ \int_{\mathcal{Y}}\exp\{-\frac{1}{2}\sum_{i,j=1}^d(\nabla^2_x\mathcal{H}[\varphi])_{i,j}y_i y_j\}O(\beta\|y\|_2^4)\cdot(\beta^{-1/2})^4~\dd\mu(y)\annot{By $x\to y=\sqrt{\beta}(x-\varphi)$}\\
        = & ~ O(\beta^{-1})~.
    \end{align*}
    Putting everything together, we obtain
    \begin{align*}
        \exp\{-\Gamma[\varphi]\}=\exp\{-\beta\mathcal{H}[\varphi]-\frac{1}{2}\log\det(\beta\nabla^2_x\mathcal{H}[\varphi])+\frac{d}{2}\log(2\pi)+O(\beta^{-1})\}~.
    \end{align*}
    This proves the proposition. 
\end{proof} 

\begin{proposition}[Low-Temperature Limit and Laplace Approximation \cite{Zinn-Justinqft,dembo2010largedeviation}]
\label{prp:zero_tem_laplace}
    If $(-\beta\mathcal{H}[x]+\sum_{i=1}^dJ_ix_i)$ has unique maximizer $x_J$ with positive-definite Hessian, then for continuous cases,
    \begin{align*}
        \frac{1}{\beta}\log\int_{\mathcal{X}}\exp\{-\beta\mathcal{H}[x]+\sum_{i=1}^dJ_ix_i\}~\dd\mu(x)=\sup_{x\in\mathcal{X}}~\{-\mathcal{H}[x]+\frac{1}{\beta}\sum_{i=1}^dJ_ix_i\}+O(\frac{\log\beta}{\beta})~.
    \end{align*}
    For discrete cases, 
    \begin{align}
    \label{eq:zero_tem_laplace_discrete}
        \frac{1}{\beta}\log\sum_{x\in\mathcal{X}}\exp\{-\beta\mathcal{H}[x]+\sum_{i=1}J_ix_i\}=\sup_{x\in\mathcal{X}}~\{-\mathcal{H}[x]+\frac{1}{\beta}\sum_{i=1}^dJ_ix_i\}+O(\frac{\log|\mathcal{X}|}{\beta})~.
    \end{align}
\end{proposition}

\begin{proof}
    We repeat the proof from \cite{Zinn-Justinqft,dembo2010largedeviation} for completeness.
    We start from the continuous case. 
    Since the maximizer is unique, we Taylor expand up to the first leading order as
    \begin{align*}
        -\beta\mathcal{H}[x]+\sum_{i=1}^dJ_ix_i = & - \beta\mathcal{H}[\varphi]+\sum_{i=1}^dJ_i\varphi_i-\frac{1}{2}\beta[x_i-\varphi_i](\nabla^2_x\mathcal{H}[\varphi])_{ij}[x_j-\varphi_j]~.\annot{By \cref{def:stationary_eq}}
    \end{align*} 
    This yields
    \begin{align*}
        & ~ \frac{1}{\beta}\log\int_{\mathcal{X}=\R^d}\exp\{-\beta\mathcal{H}[x]+\sum_{i=1}^dJ_ix_i\}~\dd\mu(x) \\
        = & ~\frac{1}{\beta}\log(\exp\{-\beta\mathcal{H}[\varphi]+\sum_{i=1}^dJ_i\varphi_i\}) +\frac{1}{\beta}\log((\frac{2\pi}{\beta})^{d/2}(\det(\beta\nabla^2_x\mathcal{H}[\varphi]))^{-1/2})\\
        = & ~ -\mathcal{H}[\varphi]+\frac{1}{\beta}\sum_{i=1}^dJ_i\varphi_i+\mathcal{O}(\frac{\log\beta}{\beta})~.
    \end{align*}
    This proves the first equation. 
    For discrete cases, 
    \begin{align*}
        \text{max}_{x\in\mathcal{X}}~(-\beta\mathcal{H}[x]+\sum_{i=1}^dJ_ix_i)\leq & ~ \log\sum_{x\in\mathcal{X}}\exp\{-\beta\mathcal{H}[x]+\sum_{i=1}^dJ_ix_i\}\\
        \leq & ~ \text{max}_{x\in\mathcal{X}}~ (-\beta\mathcal{H}[x]+\sum_{i=1}^dJ_ix_i)+\log|\mathcal{X}|~.
    \end{align*}
    After plugging this into \eqref{eq:zero_tem_laplace_discrete}, we complete the proof.
\end{proof}

\begin{remark}
    In the zero-temperature limit, \cref{prp:zero_tem_laplace} shows that $W[J]/\beta$ is a Fenchel conjugate of the Hamiltonian by \cref{def:fenchel_conjugation}.
\end{remark}

\begin{remark}
    From the proof in \cref{prp:fluctuation} and \cref{prp:zero_tem_laplace}, we notice that the perturbation only works in the low-temperature regime. 
    When the temperature increases, the strength of fluctuations increases and the higher-loop corrections starts to grow. 
    There exists a level where the perturbative method breaks down.
\end{remark}

\subsection{Fluctuated Variables in Modern Hopfield Model}

In the following, we start to consider the fluctuated version of MHMs.

\begin{definition}[Statistics-MHM (Fluctuated) Mapping]
\label{def:map_fluctuated_variables}
    The mapping between the variables in statistical mechanics and which in the fluctuated version of MHMs is:
    
    \begin{itemize}
        \item \textbf{Microscopic State} $x\in\R^d$: It represents continuous query. 
        It is also the dynamical variable whose energy landscape defines retrieval basins.
        \item \textbf{External Source} $J\in\R^d$: It represents an external bias of the query.
        \item \textbf{Classical Solution} $\varphi\in\R^d$: It represents macroscopic state. 
        Retrieval corresponds to $\varphi_i$ aligning with some pattern $\xi^\mu_i$.
        \item \textbf{Probability Function} $p\in\Delta^M$: It represents the retrieval distribution over memories.
        \item \textbf{Tsallis Entropy} $\Psi_\alpha$: It represents a regularizer controlling how sparse the memory weights $p$ are. 
        \item \textbf{Conjugate Field} $z\in\R^M$: It represents the query-key similarity between current state and each memory, i.e., $z_\mu\coloneq\sum_{i=1}^d\xi_i^\mu x_i/\sqrt{d}$.
    \end{itemize}
    
\end{definition}

Notice that we distinguish the microscopic state, $x$, from the macroscopic one, $\varphi$, and so are $p$ and $z$ in this case.

\begin{proposition}[Expectation Values and Macroscopic States]
\label{prp:expectation_values_and_macroscopic_states_fluctuated}
    Let $\mathcal{X}$ be a finite-volume domain in $\R^d$ and $J\in\text{dom}W_\alpha$.
    The macroscopic states are determined by the mean of microscopic states.
    Specifically,
    \begin{align}
        & ~ \mathbb{E}^{(\alpha)}_{x\sim \rho_{J}(\cdot|p)}[x_i|p]=\frac{1}{\sqrt{d}}\sum_{\mu=1}^M\xi_i^\mu p_\mu+\frac{1}{\beta}J_i+B^{\mathcal{X}}_i(p,J)\coloneq m_i(p,J)+B^{\mathcal{X}}_i(p,J)~,\label{eq:expectation_x_p}\\
        & ~ \mathbb{E}^{(\alpha)}_{(x,p)\sim \rho_J}[x_i]=\varphi_i=\frac{1}{\sqrt{d}}\sum_{\mu=1}^M\xi_i^\mu\mathbb{E}^{(\alpha)}_{(x,p)\sim\rho_J}[p_\mu]+\frac{1}{\beta}J_i+\mathbb{E}^{(\alpha)}_{(x,p)\sim\rho_J}[B_i^{\mathcal{X}}(p,J)]~,\label{eq:exptation_x_fluc}
    \end{align}
    where $\rho_J(x,p)$ is the probability density function in $\mathcal{X}\times\Delta^M$, $p_J(x)$ is which in $\mathcal{X}$, and $p_{\Delta,J}(p)$ is which in $\Delta^M$.
    Furthermore, there is a boundary correction
    \begin{align*}
        B_i^{\mathcal{X}}(p,J)\coloneq\frac{\int_{\mathcal{Y}}y_i\exp^\alpha_\alpha\{-\frac{\beta}{2}(\|y\|_2^2-\|m\|^2_2)\}~\dd\mu(y)}{\int_{\mathcal{Y}}\exp^\alpha_\alpha\{-\frac{\beta}{2}(\|y\|_2^2-\|m\|^2_2)\}~\dd\mu(y)},\qquad\text{with }~\mathcal{Y}=\mathcal{X}-m.
    \end{align*}
\end{proposition}

\begin{proof}
    We consider $x$ and $p$ follow the distribution $\rho_J(x,p)$ with the joint energy function, \eqref{eq:mhm_joint_hamiltonian}.
    Then
    \begin{align*}
        \mathbb{E}^{(\alpha)}_{x\sim \rho_J(\cdot|p)}[x_i|p] = & ~ \frac{\int_{\mathcal{X}}x_i\rho^\alpha_J(x|p)~\dd\mu(x)}{\int_{\mathcal{X}}\rho_J^\alpha(x|p)~\dd\mu(x)}\\
        = & ~ 
        \frac{\int_{\mathcal{X}} x_i\exp_\alpha^\alpha\{-\frac{\beta}{2}\sum_{j=1}^dx_j^2+\frac{\beta}{\sqrt{d}}\sum_{\mu=1}^M\sum_{j=1}^d\xi_j^\mu x_jp_\mu+\sum_{j=1}^dJ_jx_j\}~\dd\mu(x)}{\int_{\mathcal{X}}\exp_\alpha^\alpha\{-\frac{\beta}{2}\sum_{j=1}^dx_j^2+\frac{\beta}{\sqrt{d}}\sum_{\mu=1}^M\sum_{j=1}^d\xi_j^\mu x_jp_\mu+\sum_{j=1}^dJ_jx_j\}~\dd\mu(x)}\\
        = & ~ 
        \frac{\int_{\mathcal{Y}}(y_i+m_i)\exp^\alpha_\alpha\{\frac{-\beta}{2}\sum_{j=1}^dy_j^2+\frac{\beta}{2}\sum_{j=1}^dm_j^2\}~\dd \mu(y)}{\int_{\mathcal{Y}}\exp^\alpha_\alpha\{\frac{-\beta}{2}\sum_{j=1}^dy_j^2+\frac{\beta}{2}\sum_{j=1}^dm_j^2\}~\dd \mu(y)}\\
        = & ~
        m_i(p,J)+B_i^{\mathcal{X}}(p,J).
    \end{align*}
    This proves the first equation.
    In this way,
    \begin{align*}
        \mathbb{E}^{(\alpha)}_{(x,p)\sim \rho_J}[x_i] = & ~ \frac{\int_{\Delta^M}\int_{\mathcal{X}}x_i\rho^\alpha_J(x,p)~\dd\mu(x)\dd\mu(p)}{\int_{\Delta^M}\int_{\mathcal{X}}\rho_J^\alpha(x,p)~\dd\mu(x)\dd\mu(p)}\notag\\
        = & ~ 
        \frac{\int_{\Delta^M}(\int_{\mathcal{X}}x_i\rho^\alpha_J(x|p)~\dd\mu(x))p^\alpha_{\Delta,J}(p)~\dd\mu(p)}{\int_{\Delta^M}(\int_{\mathcal{X}}\rho^\alpha_J(x|p)~\dd\mu(x))p^\alpha_{\Delta,J}(p)~\dd\mu(p)}\annot{By $p_{\Delta,J}(p)\rho_J(x|p)\coloneq\rho_J(x,p)$}\\
        = & ~
        \frac{\int_{\Delta^M}(\int_{\mathcal{X}}\rho_J^\alpha(x'|p)~\dd\mu(x')\mathbb{E}_{x\sim\rho_J(\cdot|p)}^{(\alpha)}[x_i|p])p_{\Delta,J}^\alpha(p)~\dd\mu(p)}{\int_{\Delta^M}(\int_{\mathcal{X}}\rho^\alpha_J(x|p)~\dd\mu(x))p^\alpha_{\Delta,J}(p)~\dd\mu(p)}\notag\\
        = & ~
        \frac{1}{\sqrt{d}}\sum_{\mu=1}^M\xi_i^\mu\tilde{\mathbb{E}}_{p\sim p_{\Delta,J}}^{(\alpha)}[p_\mu]+\frac{1}{\beta}J_i+\tilde{\mathbb{E}}_{p\sim p_{\Delta,J}}^{(\alpha)}[B_i^{\mathcal{X}}(p,J)],
    \end{align*}
    where 
    \begin{align*}
        C(p)\coloneq\int_{\mathcal{X}}\rho_J^\alpha(x|p)~\dd\mu(x),\qquad\text{and }~\tilde{\mathbb{E}}_{p\sim p_{\Delta,J}}^{(\alpha)}[p_\mu]\coloneq\mathbb{E}^{(\alpha)}_{(x,p)\sim\rho_J}[p_\mu]=\frac{\int_{\Delta^M}C(p)p_\mu p^\alpha_{\Delta,J}(p)~\dd\mu(p)}{\int_{\Delta^M}C(p)p^\alpha_{\Delta,J}(p)~\dd\mu(p)}.
    \end{align*}
    This completes the proof.
\end{proof}

\begin{remark}
    The correction $B_i^{\mathcal{X}}$ vanishes on translation-invariant domains (e.g., $\R^d$ if the integral is normalizable).
    Furthermore, it is negligible in the large-volume or low-temperature regime when the dominant integrand is localized far from the boundary of $\mathcal{X}$.
\end{remark}

\subsection{Connection between Free Energy and Tsallis Entropy in Fluctuated Case}

In \cref{sec:entropy_function}, we relate the log-partition function to the Tsallis entropy via \eqref{eq:free_energy_tsallis_entropy}.
Here, we express the log-partition function using the variables in fluctuated version of MHMs, \cref{def:map_fluctuated_variables}.
To make $Z_\alpha[J]$ finite, we assume that the integration region of $x$ is finite volume, i.e., $\mathcal{X}\subset\R^d$ such that $|\mathcal{X}|<+\infty$.

\begin{proposition}
\label{prp:free_energy_mhm}
    Let $\rho_J(x,p)$ and $p_J(x)$ be the probability density function in $\mathcal{X}\times\Delta^M$ and $\mathcal{X}$, respectively.
    Inspired by \cref{thm:free_energy_tsallis_entropy}, we have
    \begin{align}
    \label{eq:w_fluctuated}
        -W_\alpha[J,\xi] = & ~
        \frac{\beta}{2}\sum_{i=1}^d\mathbb{E}^{(\alpha)}_{(x,p)\sim\rho_J}[x_i^2]-\beta\sum_{\mu=1}^M\mathbb{E}^{(\alpha)}_{(x,p)\sim \rho_J}[z_\mu p_\mu]\notag\\
        & ~
        -\frac{\Psi_\alpha[\rho_J]}{Z_\alpha^{\alpha-1}[J,\xi]\int_{\mathcal{X}}\int_{\Delta^M}\rho_J^\alpha(x,p)~\dd\mu(p)\dd\mu(x)}-\sum_{i=1}^dJ_i\mathbb{E}^{(\alpha)}_{(x,p)\sim\rho_J}[x_i].
    \end{align}
    Then the free energy is
    \begin{align}
    \label{eq:fluc_eff_action_mhm}
        \tilde{\Gamma}_\alpha = & ~ 
        \frac{\beta}{2}\sum_{i=1}^d\mathbb{E}_{(x,p)\sim\rho
        _{J^\star}}^{(\alpha)}[x_i^2]-\beta\sum_{\mu=1}^M\mathbb{E}_{(x,p)\sim\rho_{J^\star}}^{(\alpha)}[z_\mu p_\mu]\notag\\
        & ~
        -\frac{\Psi_\alpha[\rho_{J^\star}]}{Z_\alpha^{\alpha-1}[J^\star,\xi]\int_{\mathcal{X}}\int_{\Delta^M}\rho_{J^\star}^\alpha(x,p)~\dd\mu(p)\dd\mu(x)},
    \end{align}
    where $J^\star$ is the maximizing source, i.e., $J^\star=\arg\max_{J}\{\sum_{i=1}^dJ_i\varphi_i-W_\alpha[J,\xi]\}$, and $\varphi=\mathbb{E}^{(\alpha)}_{(x,p)\sim\rho_J}[x]$.
\end{proposition}

\begin{proof}
    Following the proof of \cref{thm:free_energy_mhm}, we have \eqref{eq:w_mhm}.
    This is exactly the first equation we would like to prove.
    For the free energy, we consider the Fenchel conjugation \cref{def:fenchel_conjugation} of the log-partition function.
    \begin{align*}
        \tilde{\Gamma}_\alpha = & ~
        \sup_{J\in\text{dom}W_\alpha}\{\sum_{i=1}^dJ_i\varphi_i-W_\alpha[J,\xi]\}\\
        = & ~
        \sup_{J\in\text{dom}W_\alpha}\Big\{\sum_{i=1}^dJ_i(\varphi_i-\mathbb{E}_{(x,p)\sim \rho_J}^{(\alpha)}[x_i])+\frac{\beta}{2}\sum_{i=1}^d\mathbb{E}_{(x,p)\sim\rho
        _J}^{(\alpha)}[x_i^2]-\beta\sum_{\mu=1}^M\mathbb{E}_{(x,p)\sim\rho_J}^{(\alpha)}[z_\mu p_\mu]\\
        & ~
        -\frac{\Psi_\alpha[\rho_J]}{Z_\alpha^{\alpha-1}[J,\xi]\int_{\mathcal{X}}\int_{\Delta^M}\rho_J^\alpha(x,p)~\dd\mu(p)\dd\mu(x)}\Big\}\\
        = & ~
        \frac{\beta}{2}\sum_{i=1}^d\mathbb{E}_{(x,p)\sim\rho
        _{J^\star}}^{(\alpha)}[x_i^2]-\beta\sum_{\mu=1}^M\mathbb{E}_{(x,p)\sim\rho_{J^\star}}^{(\alpha)}[z_\mu p_\mu]\\
        & ~
        -\frac{\Psi_\alpha[\rho_{J^\star}]}{Z_\alpha^{\alpha-1}[J^\star,\xi]\int_{\mathcal{X}}\int_{\Delta^M}\rho_{J^\star}^\alpha(x,p)~\dd\mu(p)\dd\mu(x)},
    \end{align*}
    where $\varphi=\mathbb{E}^{(\alpha)}_{(x,p)\sim\rho_J}[x]$.
    This completes the proof.
\end{proof}

\begin{proposition}[MHMs in the Zero-Temperature Limit]
\label{prp:mhm_in_zero_temperature}
    In the zero-temperature limit, $\beta\to\infty$, minimizing the free energy becomes minimizing the following quantity:
    \begin{align}
    \label{eq:mhm_effective_free_energy}
        \mathcal{F}[\varphi,\xi]=\frac{1}{2}\sum_{i=1}^d\varphi_i^2-\max_\mu\bar{z}_\mu~,
    \end{align}
    where $\bar{z}_\mu\coloneq\sum_{i=1}^d\xi_i^\mu\varphi_i/\sqrt{d}$.
    Notably, this is the free energy in  the fixed case, \eqref{eq:mhm_effective_action}, up to a scaling , $\beta$.
\end{proposition}

\begin{proof}
    In the zero-temperature limit, the Gibbs probability density function concentrates to an indicator, i.e.,
    \begin{align*}
        \rho_J(x,p)\to\one_{\{x,\mathbb{E}^{(\alpha)}_{(x,p)\sim\rho_J}[x]\}}\one_{\{p,\mathbb{E}^{(\alpha)}_{(x,p)\sim\rho_{J}}[p]\}}~,
    \end{align*}
    as $\beta\to\infty$.
    Hence,
    \begin{align*}
        \mathbb{E}^{(\alpha)}_{(x,p)\sim\rho_J}[x^2]\to(\mathbb{E}^{(\alpha)}_{(x,p)\sim\rho_J}[x])^2=(\tilde{\mathbb{E}}_{x\sim p_J}^{(\alpha)}[x])^2=\varphi^2,
    \end{align*}
    and
    \begin{align*}
        \mathbb{E}_{(x,p)\sim\rho_J}[z_\mu p_\mu]\to\frac{1}{\sqrt{d}}\sum_{i=1}^d\xi_i^\mu\mathbb{E}^{(\alpha)}_{(x,p)\sim\rho_J}[x_i]\mathbb{E}_{(x,p)\sim\rho_J}^{(\alpha)}[p_\mu]=\frac{1}{\sqrt{d}}\sum_{i=1}^d\xi_i^\mu\varphi_i\tilde{\mathbb{E}}^{(\alpha)}_{p\sim p_{\Delta,J}}[p_\mu]\coloneq\bar{z}_\mu\bar{p}_\mu,
    \end{align*}
    where $\tilde{\mathbb{E}}^{(\alpha)}_{x\sim p_J}[x]$ and $\tilde{\mathbb{E}}_{p\sim p_{\Delta,J}}^{(\alpha)}[p]$ are defined in \eqref{eq:tilde_exp_x} and \eqref{eq:tilde_exp_p}, respectively. 
    The free energy \eqref{eq:fluc_eff_action_mhm} becomes
    \begin{align*}
        \Gamma_\alpha\to\frac{\beta}{2}\sum_{i=1}^d\varphi_i^2-\beta\sum_{\mu=1}^M\bar{z}_\mu\bar{p}_\mu-\Psi_\alpha[\one]~.
    \end{align*}
    Since $\Psi_\alpha[\one]=0$, we have
    \begin{align*}
        \inf_{\bar{p}\in\Delta^M}\{-\sum_{\mu=1}^M\bar{z}_\mu \bar{p}_\mu\}=-\max_\mu \bar{z}_\mu~.
    \end{align*}
    Substituting it into the free energy to eliminate the dependence of $\bar{p}_\mu$, we prove this proposition.
\end{proof}

\begin{remark}
    \cref{prp:mhm_in_zero_temperature} is consistent with \cref{prp:fluctuation} in the zero-temperature limit.
    Specifically, $\Gamma_\alpha[\varphi]\to(\beta\mathcal{H}[\varphi]+\text{constant})$ as $\beta\to\infty$.
    Following the proof of \cref{thm:free_energy_mhm}, we conclude the same result.
\end{remark}

\subsection{Retrieval Dynamics in Fluctuated Case and Thermal Attention}

In \cref{sec:background}, we introduce the equation of motions according to \cref{prp:fenchel_young_duality_relation} and investigate the attention mechanism from the trajectories in associative memory systems.
In this section, we consider the fluctuated version, which gives an alternative thermal attention mechanism.
Let $\tilde{W}_\alpha\coloneq\log_\alpha\int_{\Delta^M}\exp_\alpha\{-\beta\tilde{\mathcal{H}}_{\text{MHM}}+\sum_{i=1}^dJ_ix_i\}~\dd\mu(p)$.
Notice that in the fixed setting, it reduces to the simpler form, \eqref{eq:eom_w_p}.
Then,

\begin{proposition}[EOMs of Free Energies in MHMs]
\label{prp:eom_hopfield}
    The EOM of the log-partition function in the $x$-sector corresponds to the Hopfield retrieval
    \begin{align}
    \label{eq:hopfield_retrieval_fluctuated}
        \nabla_J W_\alpha[J,\xi]=\int_{\mathcal{X}}\int_{\Delta^M}x_i\rho_J^\alpha(x,p)~\dd\mu(p)\dd\mu(x)~.
    \end{align}
    The EOM of the log-partition function in the $p$-sector corresponds to the unnormalized escort moment
    \cite{Naudts_2004}
    \begin{align}
    \label{eq:escort_moment}
        \nabla_z \tilde{W}_\alpha[z,J,\xi]=\int_{\Delta^M}\beta p\rho_J^\alpha(p|x)~\dd\mu(p)\coloneq\beta p^{(\text{esc})}~.
    \end{align}
\end{proposition}

\begin{proof}
    By \cref{def:log_partition_effective_action}, $W_\alpha[J,\xi]=\log_\alpha Z_\alpha[J,\xi]$.
    After taking the derivative with respect to $J$, we obtain
    \begin{align*}
        \frac{\partial}{\partial J_i}W_\alpha[J,\xi] = & ~
        \frac{\partial}{\partial J_i}\log_\alpha Z_\alpha[J,\xi]\\ 
        = & ~
        Z_\alpha^{-\alpha}[J,\xi]\frac{\partial Z_\alpha[J,\xi]}{\partial J_i}\annot{By \cref{prp:alpha_deformed_functions_properties}}\\
        = & ~
        Z_\alpha^{-\alpha}[J,\xi]\frac{\partial}{\partial J_i}\int_{\mathcal{X}}\int_{\Delta^M}\exp_\alpha\{-\beta\tilde{\mathcal{H}}_{\text{MHM}}[x,p,\xi]+\sum_{j=1}^dJ_jx_j\}~\dd\mu(p)\dd\mu(x)\\
        = & ~
        Z_\alpha^{-\alpha}[J,\xi]\int_{\mathcal{X}}\int_{\Delta^M}\exp^\alpha_\alpha\{-\beta\tilde{\mathcal{H}}_{\text{MHM}}[x,p,\xi]+\sum_{j=1}^dJ_jx_j\}\annot{By \cref{prp:alpha_deformed_functions_properties}}\\
        & ~
        \times\frac{\partial}{\partial J_i}(-\beta\tilde{\mathcal{H}}_{\text{MHM}}[x,p,\xi]+\sum_{k=1}^dJ_kx_k)~\dd\mu(p)\dd\mu(x)\\
        = & ~
        \int_{\mathcal{X}}\int_{\Delta^M}x_i~\frac{\exp_\alpha^\alpha\{-\beta\tilde{\mathcal{H}}_{\text{MHM}}[x,p,\xi]+\sum_{j=1}^dJ_jx_j\}}{Z_\alpha^\alpha[J,\xi]}~\dd\mu(p)\dd\mu(x)\\
        = & ~
        \int_{\mathcal{X}}\int_{\Delta^M}x_i\rho_J^\alpha(x,p)~\dd\mu(p)\dd\mu(x).
    \end{align*}
    For the second equation, we substitute the definition of $\tilde{W}_\alpha$ directly: 
    \begin{align*}
        \frac{\partial}{\partial z_\mu}\tilde{W}_\alpha[z,J,\xi] = & ~ \frac{\partial}{\partial z_\mu}\log_\alpha\int_{\Delta^M}\exp_\alpha\{-\frac{\beta}{2}\sum_{i=1}^dx_i^2+\beta\sum_{\mu=1}^Mz_\mu p_\mu+\sum_{i=1}^dJ_ix_i\}~\dd\mu(p)\\
        = & ~ \Big(\int_{\Delta^M}\exp_\alpha\{-\frac{\beta}{2}\sum_{i=1}^dx_i^2+\beta\sum_{\mu=1}^Mz_\mu p_\mu+\sum_{i=1}^dJ_ix_i\}~\dd\mu(p)\Big)^{-\alpha}\\
        & ~
        \times\int_{\Delta^M}\beta p_\mu\exp^\alpha_\alpha\{-\frac{\beta}{2}\sum_{i=1}^dx_i^2+\beta\sum_{\mu=1}^Mz_\mu p_\mu+\sum_{i=1}^dJ_ix_i\}~\dd\mu(p)\annot{By \cref{prp:alpha_deformed_functions_properties}}\\
        = & ~ \int_{\Delta^M}\beta p_\mu\rho_J^\alpha(p|x)~\dd\mu(p)~.
    \end{align*}
    This completes the proof.
\end{proof}

\begin{remark}
\label{rmk:alpha_1_gibbs}
    If $\alpha\to1$, then everything simplifies to the ordinary Gibbs case.
    In this way, the Hopfield retrieval becomes
    \begin{align*}
        \nabla_JW_{\alpha\to1}[J,\xi] = & ~
        \int_{\mathcal{X}}\int_{\Delta^M}x\rho_J(x,p)~\dd\mu(p)\dd\mu(x)\\
        = & ~
        \mathbb{E}_{(x,p)\sim\rho_J}[x]\\
        = & ~ 
        \int_{\mathcal{X}}xp_J(x)~\dd\mu(x)\\
        = & ~
        \mathbb{E}_{x\sim p_J}[x].
    \end{align*}
    Furthermore, the escort moment becomes the conditional expectation value of the probability function for general $\beta$. 
    \begin{align*}
        \nabla_z\tilde{W}_{\alpha\to1}[z,J,\xi]=\beta\mathbb{E}_{p\sim\rho_J(\cdot|x)}[p]~.
    \end{align*}
\end{remark}

\begin{remark}
\label{rmk:thermal_attention}
    If we do the similar analysis of \cref{prp:eom_attention} for the fluctuated case, then we have the thermal attention mechanism.
    In particular, if $\alpha\to1$, then the learning trajectory follows the ordinary Gibbs expectation of the conditional distribution function, i.e., $\nabla_z\tilde{W}_{\alpha\to1}=\beta\mathbb{E}_{p\sim\rho_J(\cdot|x)}[p]$, due to \cref{rmk:alpha_1_gibbs}.
    If $\alpha>1$, then the system amplifies high-density regions.
    This leads to a sharper attention mechanism.
    If $1>\alpha>0$, then it corresponds to a softer attention mechanism.
\end{remark}

\begin{remark}
    From \eqref{eq:hopfield_retrieval_fluctuated}, we obtain a macroscopic retrieval state depending on the external field at microscopic queries, $\{x_i\}$. 
    It is similar to a spin system in an external magnetic field.
    After turning on the magnetic field, the probability distribution of spins starts to deform and yields an expectation value $\mathbb{E}[S_i]$, where $S_i$ is the notation of a spin at site $i$.
    We call it a Landau response.
    Hence, the Hopfield retrieval is a Landau response in the $x$-sector ($J$-space).
    For the attention mechanism, it is an inner Landau response in the $p$-sector ($z$-space).
    In this way, we view the query-to-key score, $z_\mu=\sum_{i=1}^d\xi_i^\mu x_i/\sqrt{d}$, as an external field at $\{p_\mu\}$ in $\Delta^M$ space.
    This formalism gives a physical interpretation of the retrieval dynamics.
\end{remark}

\begin{lemma}[Linear Function and its Minimizer on Simplex \cite{rockafellar1997convex}]
\label{lem:linear_function_minimizer_simplex}
    Consider a function $h(p)$ that is linear in $p$, i.e., 
    \begin{align*}
        h(p)=\sum_{\mu=1}^Ma_\mu p_\mu+b,
    \end{align*}
    for some constant $a_\mu\in\R^M$.
    Then $h(p)$ cannot have a unique minimizer $p^\star$ in the interior of the simplex.
\end{lemma}

\begin{proof}
    This is a standard convex analysis \cite{rockafellar1997convex}.
    We repeat the proof for the completeness.
    Assume that a minimizer lies in the interior of $\Delta^M$.
    Then $p^\star>0$.
    Since $p^\star$ is a minimizer, we have
    \begin{align*}
        \frac{d}{d\varepsilon}h(p^\star+\varepsilon v)|_{\varepsilon=0}=0,
    \end{align*}
    for all tangent vectors $v\in T_{p^\star}\Delta^M=\{v\in\R^M|\sum_{\mu=1}^Mv_\mu=0\}\simeq\R^{M-1}$.
    Then, after expanding the linear function $h$, it yields
    \begin{align*}
        0=\frac{d}{d\varepsilon}(\sum_{\mu=1}^Ma_\mu p^\star_\mu+\varepsilon\sum_{\mu=1}^Ma_\mu v_\mu+b)|_{\varepsilon=0}=\sum_{\mu=1}^Ma_\mu v_\mu.
    \end{align*}
    Since the tangent vector satisfies $\sum_{\mu=1}^Mv_\mu=0$, we conclude that $a_\mu=\lambda$ for all $\mu$ and some constant $\lambda$.
    Hence,
    \begin{align*}
        h(p)=\lambda\sum_{\mu=1}^Mp_\mu+b=\lambda+b,
    \end{align*}
    which is constant on $\Delta^M$.
    This contradicts the assumption that the minimizer $p^\star$ is unique.
    This proves the lemma.
\end{proof}

\begin{lemma}[Normalization Factor in the Zero-Temperature Limit]
\label{lem:noramlization_zero_temperature}
    Consider the zero-temperature limit $\beta\to\infty$ and $\alpha>1$.
    Assume that $\tilde{\mathcal{H}}_{\text{MHM}}\geq0$, and $\rho_J(p|x)$ has a unique maximizer $p^\star$ such that $\tilde{\mathcal{H}}_{\text{MHM}}[x,p^\star,\xi]=0$.
    If the $p$ is continuous on the simplex $\Delta^M$, where $M>1$, then the normalization function has the following limit
    \begin{align*}
        \lim_{\beta\to\infty}C(x)=\lim_{\beta\to\infty}\int_{\Delta^M}\rho_J^\alpha(p|x)~\dd\mu(p)=\infty.
    \end{align*}
    If $\alpha\to1$ for general $\beta$ or if $p$ is discrete on the simplex $\Delta^M$ for general $\alpha$, then 
    \begin{align*}
        \lim_{\alpha\to1}C(x)=\lim_{\beta\to\infty}C^{\text{dis}}(x)=\lim_{\beta\to\infty}\sum_{p\in\Delta^M}(\rho_J^{\text{dis}})^\alpha(p|x)=1.
    \end{align*}
\end{lemma}

\begin{proof}
    We first consider the continuous case.
    Since $\tilde{\mathcal{H}}_{\text{MHM}}$ is linear in $p$, its unique minimizer $p^\star$ must be at a vertex of $\Delta^M$ according to \cref{lem:linear_function_minimizer_simplex}.
    Hence, we denote 
    \begin{align}
    \label{eq:p_vertex_simplex}
        p^\star_\mu=\begin{cases}
            1,\qquad&\text{if }~\mu=\mu^\star\\
            0,\qquad&\text{if }~\mu\neq\mu^\star
        \end{cases}.
    \end{align}
    In this way, for $p\neq p^\star$, we have
    \begin{align*}
        \tilde{\mathcal{H}}_{\text{MHM}}[x,p,\xi] = & ~
        \tilde{\mathcal{H}}_{\text{MHM}}[x,p,\xi]-\tilde{\mathcal{H}}_{\text{MHM}}[x,p^\star,\xi] \annot{By $\tilde{\mathcal{H}}_{\text{MHM}}[x,p^\star,\xi]=0$}\\
        = & ~
        \frac{1}{2}\sum_{i=1}^dx_i^2+\sum_{\mu=1}^Mz_\mu p_\mu-(\frac{1}{2}\sum_{i=1}^dx_i^2+\sum_{\mu=1}^M z_\mu p^\star_\mu )\\
        = & ~
        \sum_{\mu=1}^Mz_\mu(p_\mu-p^\star_\mu)\\
        = & ~
        \sum_{\mu=\mu^\star}z_\mu(p_\mu-p^\star_\mu)+\sum_{\substack{\mu=1\\\mu\neq\mu^\star}}^Mz_\mu(p_\mu-p^\star_\mu)\\
        = & ~
        z_{\mu^\star}p_{\mu^\star}-z_{\mu^\star}p^\star_{\mu^\star}+\sum_{\substack{\mu=1\\\mu\neq\mu^\star}}^Mz_\mu(p_\mu-p^\star_\mu)\\
        = & ~
        z_{\mu^\star}p_{\mu^\star}-z_{\mu^\star}+\sum_{\substack{\mu=1\\\mu\neq\mu^\star}}^Mz_\mu p_\mu\annot{By \eqref{eq:p_vertex_simplex}}\\
        = & ~
        \sum_{\substack{\mu=1\\\mu\neq\mu^\star}}^M(z_\mu-z_{\mu^\star})p_\mu.\annot{By $1=p_{\mu^\star}+\sum_{\mu\neq\mu^\star}^Mp_\mu$}
    \end{align*}
    Since $p^\star$ is a unique minimizer of $\tilde{\mathcal{H}}_{\text{MHM}}$, $(z_\mu-z_{\mu^\star})>0$ for all $\mu\neq \mu^\star$.
    From \cref{def:alpha_deformed_exponential}, we obtain 
    \begin{align*}
        \exp_\alpha\{-\beta\tilde{\mathcal{H}}_{\text{MHM}}[x,p,\xi]+\sum_{i=1}^dJ_ix_i\} = & ~
        \Big\{1+(\alpha-1)(\beta\tilde{\mathcal{H}}_{\text{MHM}}[x,p,\xi]-\sum_{i=1}^dJ_ix_i)\Big\}^{1/(1-\alpha)}\\
        = & ~
        \Big\{1+(\alpha-1)\beta\sum_{\mu\neq\mu^\star}^M(z_\mu-z_{\mu^\star})p_\mu-(\alpha-1)\sum_{i=1}^dJ_ix_i\Big\}^{1/(1-\alpha)}
    \end{align*}
    for $\alpha>1$.
    Then
    \begin{align*}
        C(x) = & ~
        \frac{\int_{\Delta^M}\exp^\alpha_\alpha\{-\beta\tilde{\mathcal{H}}_{\text{MHM}}[x,p,\xi]+\sum_{i=1}^dJ_ix_i\}~\dd\mu(p)}{(\int_{\Delta^M}\exp_\alpha\{-\beta\tilde{\mathcal{H}}_{\text{MHM}}[x,p',\xi]+\sum_{i=1}^dJ_ix_i\}~\dd\mu(p'))^\alpha}\\
        = & ~
        \frac{\int_{\Delta^M}\Big\{[1-(\alpha-1)\sum_{i=1}^dJ_ix_i]+(\alpha-1)\beta\sum_{\mu\neq \mu^\star}^M(z_\mu-z_{\mu^\star})\Big\}^{\alpha/(1-\alpha)}~\dd\mu(p)}{\Big(\int_{\Delta^M}\Big\{[1-(\alpha-1)\sum_{i=1}^dJ_ix_i]+(\alpha-1)\beta\sum_{\mu\neq \mu^\star}^M(z_\mu-z_{\mu^\star})\Big\}^{1/(1-\alpha)}~\dd\mu(p')\Big)^\alpha}
    \end{align*}
    Since the integrands only depend on $\mu\neq\mu^\star$, we change the variable to $u_\mu=\beta q_\mu$, where $q_\mu\coloneq p_{\mu}$ with $\mu\neq \mu^\star$.
    There are $(M-1)$ such coordinates, so $\dd\mu(p)\sim\beta^{-(M-1)}\dd\mu(u)$.
    Then, for $\alpha>1$ and $M>1$, we have
    \begin{align*}
        C(x)\sim\frac{\beta^{-(M-1)}}{\Big(\beta^{-(M-1)}\Big)^\alpha}=\beta^{(\alpha-1)(M-1)}\to\infty
    \end{align*}
    as $\beta\to\infty$.
    For $\alpha\to1$, it is simply $C(x)=\int_{\Delta^M}\rho_J(p|x)~\dd\mu(p)=1$ for any $\beta$.
    For the discrete case, the conditional distribution is
    \begin{align*}
        \rho^{\text{dis}}_J(p|x)=\frac{\exp_\alpha\{-\beta\tilde{\mathcal{H}}_{\text{MHM}}[x,p,\xi]+\sum_{i=1}^dJ_ix_i\}}{\sum_{p'\in\Delta^M}\exp_\alpha\{-\beta\tilde{\mathcal{H}}_{\text{MHM}}[x,p',\xi]+\sum_{i=1}^dJ_ix_i\}},
    \end{align*}
    with discrete $p$.
    Notice that as $\beta\to\infty$, $\exp_\alpha\{-\beta\tilde{\mathcal{H}}_{\text{MHM}}[x,p,\xi]\}$ is suppressed to zero except for the maximizer $p^\star$, i.e., $\exp_\alpha\{-\beta\tilde{\mathcal{H}}_{\text{MHM}}[x,p^\star,\xi]\}=1$.
    Hence, as $\beta\to\infty$, the denominator of $\rho^{\text{dis}}_J(p|x)$ approaches to
    \begin{align*}
        \lim_{\beta\to\infty}\sum_{p'\in\Delta^M}\exp_\alpha\{-\beta\tilde{\mathcal{H}}_{\text{MHM}}[x,p',\xi]+\sum_{i=1}^dJ_ix_i\}=\exp_\alpha\{\sum_{i=1}^dJ_ix_i\}.
    \end{align*}
    Then
    \begin{align*}
        \lim_{\beta\to\infty}C^{\text{dis}}(x) = & ~
        \lim_{\beta\to\infty}\sum_{p\in\Delta^M}(\rho_J^{\text{dis}})^\alpha(p|x)\\
        = & ~
        \lim_{\beta\to\infty}\Big\{\sum_{p=p^\star}(\rho_J^{\text{dis}})^\alpha(p|x)+\sum_{\substack{p\in\Delta^M\\p\neq p^\star}}(\rho_J^{\text{dis}})^\alpha(p|x)\Big\}\\
        = & ~
        \lim_{\beta\to\infty}\Big\{\frac{\exp^\alpha_\alpha\{\sum_{i=1}^dJ_ix_i\}}{(\exp_\alpha\{\sum_{i=1}^dJ_ix_i\})^\alpha}\annot{By $\tilde{\mathcal{H}}_{\text{MHM}}[x,p^\star,\xi]=0$}\\
        & ~
        +\sum_{\substack{p\in\Delta^M\\p\neq p^\star}}\frac{\exp^\alpha_\alpha\{-\beta\tilde{\mathcal{H}}_{\text{MHM}}[x,p,\xi]+\sum_{i=1}^dJ_ix_i\}}{(\exp_\alpha\{\sum_{i=1}^dJ_ix_i\})^\alpha}\Big\}\\
        = & ~
        \lim_{\beta\to\infty}\Big\{1+0\Big\}\\
        = & ~1.
    \end{align*}
    This completes the proof.
\end{proof}

\begin{remark}
    In practical attention mechanisms, the retrieved memories are a finite set.
    The attention weights are finite-dimensional vectors. 
    Therefore, the discrete normalization behavior in \cref{lem:noramlization_zero_temperature} occurs in practice. 
    In contrast, the divergence of $C(x)$ in the continuous case is mainly a theoretical issue.
\end{remark}

\begin{proposition}[Emergence of Hard Argmax from Escort Moments]
\label{prp:eom_hopfield_zero_temperature}
    Assume that $\tilde{\mathcal{H}}_{\text{MHM}}\geq0$ has a unique minimizer $p^\star\in\Delta^M$.
    In the zero-temperature limit, $\beta\to\infty$, the following escort variables become deterministic hard argmax selections:
    \begin{align}
    \label{eq:eom_deterministic}
        \begin{cases}
            \alpha\to1&:p^{(\text{esc})}\\
            \alpha>1,~\text{Continuous}&:p^{(\text{esc})}/C(x)\\
            \alpha>1,~\text{Discrete}&:p^{(\text{esc})}
        \end{cases}\underset{\beta\to\infty}{\longrightarrow}p^\star=\arg\max_{p\in\Delta^M}\{\sum_{\mu=1}^Mz_\mu p_\mu\}.
    \end{align}
\end{proposition}

\begin{proof}
    For $\alpha\to1$, the escort moment reduces to the ordinary Gibbs expectation,
    \begin{align*}
        p^{(\text{esc})}=\int_{\Delta^M}p\rho_J(p|x)~\dd\mu(p).
    \end{align*}
    By \cref{prp:zero_tem_laplace}, the Gibbs distribution collapses at the unique minimizer in the zero-temperature limit,
    \begin{align*}
        \rho_J(p|x)\to\one_{\{p,p^\star\}},
    \end{align*}
    as $\beta\to\infty$.
    Hence, $p^{\text{(esc)}}\to p^\star$ as $\beta\to\infty$.
    Next, we consider $\alpha>1$ with continuous $p$. 
    By \cref{lem:noramlization_zero_temperature}, $\lim_{\beta\to\infty}C(x)$ diverges.
    However, since the escort measure still concentrates at the unique minimizer $p^\star$, we obtain
    \begin{align*}
        \lim_{\beta\to\infty}\frac{p^{(\text{esc})}}{C(x)}=\lim_{\beta\to\infty}\frac{\int_{\Delta^M}p\rho_J^\alpha(p|x)~\dd\mu(p)}{\int_{\Delta^M}\rho^\alpha_J(p'|x)~\dd\mu(p')}=p^\star.
    \end{align*}
    Finally, for $\alpha>1$ with discrete $p$, we have $\lim_{\beta\to\infty}C^{\text{dis}}(x)=1$ according to \cref{lem:noramlization_zero_temperature}.
    Hence, 
    \begin{align*}
        \lim_{\beta\to\infty}\frac{p^{(\text{esc})}}{C^{\text{dis}}(x)}=\lim_{\beta\to\infty}p^{(\text{esc})}=p^\star.
    \end{align*}
    This proves the proposition.
\end{proof}

\begin{remark}
    \cref{prp:eom_hopfield_zero_temperature} shows that escort-weighted retrieval \eqref{eq:escort_moment} reduces to a deterministic hard argmax selection rule in the zero-temperature limit.
    This statistical-mechanics framework demonstrates how sparse retrieval behavior is able to emerge from escort concentration in the general fluctuated (stochastic) setting.
\end{remark}

\clearpage
\section{Dynamical Retrieval Stability Analysis}
\label{sec:dynamical_stability}

In this section, we investigate the dynamical retrieval stability. 
Then we discuss how the convergences of $x$ and $p$ relate to each other.
To do this, we apply gradient-based optimization and perturb around the fixed point. 
The stability analysis verifies whether the training trajectory flows back to the fixed-point or not.

\begin{definition}[Retrieval Dynamics]
    The retrieval update rules are
    \begin{align}
        \frac{d}{dt}x_i(t)=-(\nabla_x\Gamma^{(\text{mod})}_\alpha[x(t),\xi])_i ~ & ~  ~ ~ \text{(Continuous)}~,\label{eq:gradient_flow_continuous}\\
        x_i(t+1)=x_i(t)-\eta(\nabla_x\Gamma^{(\text{mod})}_\alpha[x(t),\xi])_i ~ & ~  ~ ~ \text{(Discrete)}~,\label{eq:gradient_descent_discrete}
    \end{align}
    where $\eta$ is the learning rate.
\end{definition}

In \cref{sec:background}, we note that $\nabla_\varphi\Gamma_\alpha[\varphi]$ is not bijective.
Hence, it converges to a set rather than a point in general.
In our setting, we consider retrieval dynamics of the generalized model $\Gamma^{(\text{mod})}_\alpha$ in the query space ($x$-space), while the memory patterns remain fixed.
The projection mapping $z_\mu=\sum_{i=1}^d\xi_i^\mu x_i/\sqrt{d}$ is not injective generally, so distinct $x^\star$ correspond to one specific fixed-point $z^\star$.
However, due to
$\frac{\beta}{2}\sum_{i=1}^dx_i^2$ term,
such projection degeneracy does not imply the degeneracy of the full $\Gamma_\alpha^{(\text{mod})}[x,\xi]$.

By the proposition below, we are capable of studying the stability and refer the stability to the stability modulo the full degeneracy.
This interpretation is consistent with the  analysis of gradient-based optimization developed in \cite{lee2016gradient}, where the local Hessian structure governs the stability of gradient flows near critical points.
Furthermore, we only apply the gradient-based optimization in $x$-space.
Hence, the stable fixed-point in $x$-space must correspond to one specific stable fixed-point in $z$-space.

\begin{proposition}[State-Retrieval Stability Modulo Degeneracy]
\label{prp:dynamical_retrieval_stability_degeneracy}
    Let the macroscopic state space be $\mathcal{M}_x$, $\mathcal{K}(x)\coloneq\ker(\nabla^2_x\Gamma^{(\text{mod})}_\alpha[x,\xi])$ be the kernel space, and $\mathcal{K}^\perp(x)$ be the orthogonal complement of $\mathcal{K}(x)$.
    Assume there exists a local stable point $x^\star\in\R^d$ such that 
    \begin{align}
    \label{eq:gamma_stable_condition}
        \nabla_x\Gamma^{(\text{mod})}_\alpha[x^\star,\xi]=0~,
    \end{align}
    and
    \begin{align}
    \label{eq:hessian_gamma_positive_definite}
        \text{Hessian}~~\nabla_x^2\Gamma_\alpha^{(\text{mod})}[x^\star,\xi]~~\text{is positive definite on}~~\mathcal{K}^\perp(x^\star).
    \end{align}   
    Let $\mathcal{N}_{x^\star}\subset\mathcal{M}_\varphi$ be the degeneracy space through $x^\star$ such that $\Gamma^{(\text{mod})}_\alpha[x,\xi]$ is constant on it, and $T_{x^\star}\mathcal{N}_{x^\star}=\mathcal{K}(x^\star)$. 
    Then $x^\star$ is a stable local minimizer modulo degeneracy.
    That is, the gradient flow in the $x$-space converges to $(x^\star+v)$ for some $v\in\mathcal{K}(x^\star)$.
    Additionally, for discrete gradient descent, we require
    \begin{align}
    \label{eq:learning_rate_inequality}
        0<\eta<\frac{2}{\lambda_{\text{max}}(\nabla_x^2\Gamma^{(\text{mod})}_\alpha[x^\star,\xi])|_{\mathcal{K}^\perp(x^\star)}}~,
    \end{align}
    where $\lambda_{\text{max}}(\cdot)|_{\mathcal{K}^\perp(x^\star)}$ is the maximum eigenvalue with the Hessian being on $\mathcal{K}^\perp(x^\star)$.
\end{proposition}

\begin{proof}
    Let $x_i-x^\star_i\coloneq\delta_i\coloneq\delta_i^\|+\delta^\perp_i$ be infinitesimal, where $\delta^\|\in\mathcal{K}(x^\star)$ and $\delta^\perp\in\mathcal{K}^\perp(x^\star)$.
    Then for non-zero $\delta^\perp$,
    \begin{align*}
        \Gamma^{(\text{mod})}_\alpha[x,\xi] & = ~ \Gamma^{(\text{mod})}_\alpha[x^\star,\xi]+\sum_{i=1}^d(\nabla_x\Gamma^{(\text{mod})}_\alpha[x^\star,\xi])_i\delta_i+\frac{1}{2}\sum_{i,j=1}^d(\nabla_x^2\Gamma^{(\text{mod})}_\alpha[x^\star,\xi])_{ij}\delta_i\delta_j+O(\delta^3)\\
        & = ~ \Gamma^{(\text{mod})}_\alpha[x^\star,\xi]+\frac{1}{2}\sum_{i,j=1}^d(\nabla_x^2\Gamma^{(\text{mod})}_\alpha[x^\star,\xi])_{ij}\delta^\perp_i\delta^\perp_j+O(\delta^3)\annot{By \eqref{eq:gamma_stable_condition} and \eqref{eq:hessian_gamma_positive_definite}}~.
    \end{align*}
    Since the Hessian is positive definite on $\mathcal{K}^\perp(x^\star)$, we obtain
    \begin{align*}
        \Gamma^{(\text{mod})}_{\alpha}[x,\xi]=\Gamma^{(\text{mod})}_\alpha[x^\star+\delta,\xi]>\Gamma^{(\text{mod})}_\alpha[x^\star,\xi]~.
    \end{align*}
    Hence, the system flows back to the fixed-point, $x^\star$.
    If the infinitesimal change does not contain the orthogonal complementary component, i.e., $\delta_i=\delta^\|_i$, then $x^\star+\delta^\|\in\mathcal{N}_{x^\star}$.
    By the assumption of $\mathcal{N}_{x^\star}$, $\Gamma^{(\text{mod})}_\alpha[x,\xi]$ remains unchanged on $\mathcal{N}_{x^\star}$.
    Hence, the higher-order correction vanishes, i.e., $O(\delta^3)=0$.
    Consider the gradient flow:
    \begin{align*}
        \frac{d}{dt}\{\Gamma^{(\text{mod})}_\alpha[x(t),\xi]-\Gamma^{(\text{mod})}_\alpha[x^\star,\xi]\} = & ~ \sum_{i=1}^d(\nabla_x\Gamma^{(\text{mod})}_\alpha[x(t),\xi])_i\frac{dx_i(t)}{dt}\\
        = & ~ -\sum_{i=1}^d(\nabla_x\Gamma^{(\text{mod})}_\alpha[x(t),\xi])^2_i\leq0~,
    \end{align*}
    where we use \eqref{eq:gradient_flow_continuous}.
    The gradient flow approaches to $(x^\star+v)$ for some $v\in\mathcal{K}(x^\star)$.
    For gradient descent, 
    \begin{align}
        \delta_i(t+1) = & ~ x_i(t+1)-x^\star_i\notag\\
        = & ~ \delta_i(t)-\eta(\nabla_x\Gamma_\alpha[x(t),\xi])_i\annot{By \eqref{eq:gradient_descent_discrete}}\notag\\
        = & ~ \sum_{j=1}\{\one_{\{i,j\}}-\eta(\nabla^2_x\Gamma^{(\text{mod})}_\alpha[x^\star,\xi])_{ij}\}\delta_j(t)~.\label{eq:delta_recursion}
    \end{align}
    To guarantee that the gradient descent converges for $\mathcal{K}^\perp(x^\star)$, we require that the maximal absolute eigenvalue of $\{\cdot\cdot\cdot\}$ above should be smaller than one.
    According to \eqref{eq:hessian_gamma_positive_definite}, the eigenvalues are all positive along $\mathcal{K}^\perp(x^\star)$. 
    Hence,
    \begin{align*}
        |1-\eta\lambda_{\text{max}}(\nabla^2_x\Gamma^{(\text{mod})}_\alpha[x^\star,\xi])|_{\mathcal{K}^\perp(x^\star)}|<1~\implies~0<\eta<\frac{2}{\lambda_{\text{max}}(\nabla_x^2\Gamma^{(\text{mod})}_\alpha[x^\star,\xi])|_{\mathcal{K}^\perp(x^\star)}}~.
    \end{align*}
    If we do not restrict on $\mathcal{K}^\perp(x^\star)$, then the $\lambda_{\text{max}}(\cdot)=0$ along $\mathcal{K}(x^\star)$.
    Hence, the gradient descent approaches to $(x^\star+v)$ for some $v\in\mathcal{K}(x^\star)$.
    This proves \eqref{eq:learning_rate_inequality}.
\end{proof}

\begin{remark}
\label{rmk:degeneracy}
    For symmetry and projection degeneracy, \cref{prp:dynamical_retrieval_stability_degeneracy} holds since symmetries are always global, i.e., the group element does not differ for different $\varphi$, and the projection is affine. 
    However, for constrained systems, we require the condition that $\Gamma_\alpha$ remains the same on the degeneracy manifold of the constrained space. 
    This does not always hold.
\end{remark}

\begin{proposition}[State-Retrieval Stability as Query-Key-Retrieval Stability]
\label{prp:stability_correspondence}
    The dynamical retrieval stability in the macroscopic state space, $\mathcal{M}_x$, corresponds to which in the memory score $z$-space, $\mathcal{M}_z$.
    Specifically, let $x^\star$ be the stable fixed-point in $\mathcal{M}_x$, and $\mathcal{D}\coloneq\ker(\xi_i^\mu)\subset\R^d$. 
    \begin{align}
    \label{eq:x_z_converge}
        x(t)\to x^\star~\implies~z(t)\to z^\star\coloneq\frac{1}{\sqrt{d}}\sum_{i=1}^d\xi_i^\mu x^\star_i~.
    \end{align}
    Conversely, 
    \begin{align}
    \label{eq:z_x_converge}
        z(t)\to z^\star~\implies~\Pi_{\mathcal{D}^\perp}x(t)\to\Pi_{\mathcal{D}^\perp}x^\star~,
    \end{align}
    where $\Pi_{\mathcal{D}^\perp}$ is the projection map to $\mathcal{D}^\perp$.
    Moreover, the convergence rates in $z$ imply convergence rates in $x$ via
    \begin{align}
    \label{eq:convergence_rate_z_x}
        \|\delta x\|_2\leq\frac{d}{\sigma_{\text{min}}}\|\delta z\|_2~,~\forall~\delta x\in\mathcal{D}^\perp~,
    \end{align}
    where $\sigma_{\text{min}}>0$ denotes the smallest non-zero singular value of $(\xi_i^\mu)$.
\end{proposition}

\begin{proof}
    Since $\R^d=\mathcal{D}\oplus\mathcal{D}^\perp$, we decompose $x_i=x^\|_i+x^\perp_i$ with $\sum_{i=1}^d\xi_i^\mu x_i^\|=0$.
    For \eqref{eq:x_z_converge}, it is trivial since $\bar{z}$ is deterministic due to $z_\mu=\sum_{i=1}^d\xi_i^\mu x_i/\sqrt{d}$.
    For \eqref{eq:z_x_converge}, we consider $\delta x(t)=\Pi_{\mathcal{D}^\perp}(x(t)-x^\star)\in\mathcal{D}^\perp$.
    Then
    \begin{align*}
        \frac{1}{\sqrt{d}}\sum_{i=1}^d\xi_i^\mu\delta x_i(t)=\frac{1}{\sqrt{d}}\sum_{i=1}^d\xi_i^\mu(x_i(t)-x^\star_i)=z_\mu(t)-z^\star\to0.
    \end{align*}
    If $\delta x\in\mathcal{D}$, then $\sum_{i=1}^d\xi_i^\mu\delta x_i(t)=\sum_{i=1}^d\xi_i^\mu\Pi_{\mathcal{D}}(x_i(t)-x_i^\star)=0$ does not give any information of $x(t)$ and $x^\star$.
    We restrict to $\mathcal{D}^\perp$ now. Since on $\mathcal{D}^\perp$, ($\xi_i^\mu$) is injective and $\delta x\in\mathcal{D}^\perp$ satisfies
    \begin{align*}
        \sum_{\mu=1}^M(z_\mu(t)-z^\star_\mu)^2=\frac{1}{d}\sum_{\mu=1}^M(\sum_{i=1}^d\xi_i^\mu \delta x_i(t))^2\geq\frac{\sigma_{\text{min}}^2}{d}\sum_{i=1}^d\delta x_i^2(t),
    \end{align*}
    we thereby prove \eqref{eq:convergence_rate_z_x}.
\end{proof}

\begin{proposition}[Score-space Attention-Retrieval Stability]
\label{prp:attention_retrieval_stability}
    Consider the fixed-point in the $z$-space be $z^\star_\mu$.
    Then the Jacobian of the update map is
    \begin{align}
    \label{eq:jacobian_z_space}
        \mathcal{J}_{\mu\nu}(z^\star)=-\frac{1}{\beta}\sum_{\rho=1}^MG_{\mu\rho}(\nabla^2_z\tilde{\Gamma}^{(\text{mod})}_\alpha[x^\star,p,\xi])_{\rho\nu}~,
    \end{align}
    where $z^\star_\mu=\sum_{i=1}^dx^\star_i\xi^\mu_i/\sqrt{d}$.
   The dynamical retrieval update is stable at $z^\star_\mu$ if
    \begin{align*}
        \rho(\mathcal{J}(z^\star))<1~,
    \end{align*}
    where $\rho(\cdot)$ is the spectral radius, i.e., the maximal absolute value of eigenvalues.
\end{proposition}

\begin{proof}
    If we perturb around the fixed-point, $\delta z_\mu(t)\coloneq z_\mu(t)-z^\star_\mu$ gives
    \begin{align*}
        \delta z_\mu(t+1)= & ~ \sum_{\nu=1}^MG_{\mu\nu}p_\nu(z(t))+b_\mu(J)-\sum_{\nu=1}^MG_{\mu\nu}p_\nu(z^\star)-b_\mu(J)\\
        = & ~ \sum_{\nu,\rho=1}^MG_{\mu\nu}\frac{\partial p_\nu}{\partial z_\rho}|_{z=z^\star}~\delta z_\rho(t)~.
    \end{align*}
    According to \eqref{eq:eom_gamma_p}, we express $p$ in terms of the gradient of $\tilde{\Gamma}_{\alpha}^{\text{(mod)}}$.
    In this way,
    \begin{align*}
        \delta z_\mu(t+1)=-\frac{1}{\beta}\sum_{\nu,\rho=1}^MG_{\mu\nu}(\nabla^2_z\tilde{\Gamma}_\alpha^{(\text{mod})}[x^\star,p,\xi])_{\nu\rho}\delta z_\rho(t)
    \end{align*}
    By $\delta z_\mu(t+1)=\sum_{\nu=1}^M\mathcal{J}_{\mu\nu}\delta z_\nu(t)$, we obtain the Jacobian, \eqref{eq:jacobian_z_space}.
    If $\rho(\mathcal{J})<1$, then $\mathcal{J}^t\to0$ as $t\to\infty$.
    Hence, $\delta z_\mu(t)\to0$, and the system converges to the fixed-point, $z^\star_\mu$.
\end{proof}

\begin{remark}
    The update map is a non-linear discrete map.
    Hence, it is neither a gradient flow nor a gradient descent in general.
    This fixed-point perturbation analysis is consistent to the correspondence between the gradient-based optimization in $\mathcal{M}_x$.
\end{remark}

Now we define the Tsallis divergence which helps us characterizes the training dynamics and the convergences.
It generalizes KL-divergence and induces dual geometric structures on the probability simplex. 
\cite{Amari1985}

\begin{definition}[Tsallis Divergence \cite{furuichi2004tsallis}]
    The Tsallis divergence generalizes the Kullback-Leibler (KL) divergence \cite{kullback1951information}.
    It has the following form:
    \begin{align*}
        D_\alpha(p\|q)\coloneq-\sum_{\mu=1}^Mp_\mu\log_\alpha(\frac{q_\mu}{p_\mu})=\frac{1}{\alpha-1}(\sum_{\mu=1}^Mp^\alpha_\mu q_\mu^{1-\alpha}-1)~.
    \end{align*}
\end{definition}

\begin{remark}
    If $\alpha\to1$, the Tsallis divergence reduces to the KL divergence. 
    It follows from \cref{prp:alpha_deformed_functions_properties} in which $\log_\alpha\to\log$ as $\alpha\to1$.
    Hence,
    \begin{align*}
        D_{\alpha\to1}(p\|q)=-\lim_{\alpha\to1}\sum_{\mu=1}^Mp_\mu\log_\alpha(\frac{q_\mu}{p_\mu})=\sum_{\mu=1}^Mp_\mu\log(\frac{p_\mu}{q_\mu})\coloneq D_{\text{KL}}(p\|q)~.
    \end{align*}
\end{remark}

The dynamical stability of $p$ relates to the attention-retrieval stability via the Tsallis divergence.

\begin{proposition}[Tsallis Divergence and Attention-Retrieval Stability]
\label{prp:dynamical_z_p_stability}
    Let $\alpha>1$, $p^\star\in\Delta^M$ be the fixed maximizer in the simplex interior satisfying $z_\mu^\star=(Gp)_\mu^\star+b_\mu(J)$.
    Then there exist constants, $c_1,c_2>0$ such that
    \begin{align*}
        c_1\|z(t)-z^\star\|^2_2\leq D_\alpha(p(t)\|p^\star)\leq c_2\|p(t)-p^\star\|^2_2~.
    \end{align*}
\end{proposition}

\begin{proof}
    Since
    \begin{align*}
        \frac{\partial}{\partial p_\mu}D_\alpha(p\|q)=\frac{\alpha}{\alpha-1}q_\mu^{1-\alpha}p_\mu^{\alpha-1}~,~\frac{\partial^2}{\partial p_\mu\partial p_\nu}D_\alpha(p\|q)=\alpha q_\mu^{1-\alpha}p_\mu^{\alpha-2}\one_{\{\mu,\nu\}}~,
    \end{align*}
    we have the Taylor expansion
    \begin{align*}
        D_\alpha(p(t)\|p^\star) = & ~ D_\alpha(p^\star\|p^\star)+\frac{\alpha}{\alpha-1}\sum_{\mu=1}^M(p_\mu(t)-p^\star_\mu)+\frac{\alpha}{2}\sum_{\mu=1}^M\frac{1}{p_\mu^\star}(p_\mu(t)-p^\star_\mu)^2+O(\|p(t)-p^\star\|_2^3)\\
        = & ~ \frac{\alpha}{2}\sum_{\mu=1}^M\frac{1}{p_\mu^\star}(p_\mu(t)-p^\star_\mu)^2+O(\|p(t)-p^\star\|_2^3)~.
    \end{align*}
    Since $p^\star$ is in the simplex interior, $1>p^\star_\mu>0$ for all $\mu$.
    Hence, there exist $c_1,c_2>0$ such that
    \begin{align*}
        c_1\|p(t)-p^\star\|_2^2\leq D_\alpha(p(t)\|p^\star)\leq c_2\|p(t)-p^\star\|^2_2~.
    \end{align*}
    Moreover,
    \begin{align*}
        z_\mu(t)-z^\star_\mu=[G(p(t)-p^\star)]_\mu~\implies~\|z(t)-z^\star\|_2\leq\|G\|_{\text{op}}\cdot\|p(t)-p^\star\|_2~.
    \end{align*}
    Combining these two inequalities, we prove this proposition.
\end{proof}

\begin{remark}
    From \cref{prp:dynamical_z_p_stability}, we notice that the convergence in $\mathcal{M}_{z}$ does not control the convergence of the Tsallis divergence and the convergence in $p$-space.
    If the Gram matrix is invertible, then $\|p(t)-p^\star\|_2\leq\|G^{-1}\|_{\text{op}}\cdot\|z(t)-z^\star\|_2$.
    In this way, the convergence in $\mathcal{M}_{x}$ controls the convergence in $\mathcal{M}_{z}$ and the $p$-space.
    However, in the real world, the high correlations between memory patterns make $G_{\mu\nu}$ non-invertible.
    Hence, we do not guarantee the convergence of $p^\star$. 
\end{remark}

\clearpage

\section{Experiment Setup}
\label{sec:numerical_studies}

In our simulation, \cref{sec:numerical_results} shows the phase structure of modern Hopfield models by scanning the inverse temperature $\beta$, and the exponential memory load $\kappa=\log (M)/d$ in both the sparse $(\alpha=2)$ and dense $(\alpha=1)$ regimes. 
Here we introduce the support size to characterize the usage of memory weights.
Then we provide the experimental setup and further details in \cref{sec:exp_settings}.

\subsection{Effective Support Size and Sparsity}
\label{sec:size_sparsity}

To relate entropy functions to the size and the randomness of information concretely, we introduce the effective support size and the sparsity here.

\begin{definition}[Effective Support Size \cite{hill1973diversity}]
\label{def:effective_support_size}
    We define the (Tsallis-R\'enyi) effective support size as
    \begin{align*}
        K_{\text{eff}}^{(\alpha)}[p]\coloneq(\sum_{\mu=1}^Mp^\alpha_\mu)^{\frac{1}{1-\alpha}}~,
    \end{align*}
    where $p\in\Delta^M$ satisfying $\sum_{\mu=1}^Mp_\mu=1$ and $\alpha\in\R_+/\{1\}$.
\end{definition}

\begin{proposition}[Range of Effective Support Size \cite{hill1973diversity}]
\label{prp:range_effective_support_size}
    The effective support size satisfies the inequality: 
    \begin{align*}
        1\leq K_{\text{eff}}^{(\alpha)}[p]\leq M
    \end{align*}
    If $K_{\text{eff}}^{(\alpha)}[p]=1$, then $p$ is a peak. 
    If $K_{\text{eff}}^{(\alpha)}[p]=M$, then $p$ is uniform.
\end{proposition}

\begin{proof}
    We repeat the proof from \cite{hill1973diversity} for completeness.
    We decompose the range of $\alpha$ into $\alpha>1$ and $1>\alpha>0$ and focus on the former case first.
    For $\alpha>1$, we have $\sum_{\mu=1}^Mp_\mu^\alpha\leq1$ and  
    \begin{align*}
        1\leq & ~ (\sum_{\mu=1}^Mp_\mu^\alpha)^{\frac{1}{1-\alpha}}=K_{\text{eff}}^{(\alpha)}[p]\\
        = & ~ M^{\frac{1}{1-\alpha}}\cdot(\frac{1}{M}\sum_{\mu=1}^Mp_\mu^\alpha)^{\frac{1}{1-\alpha}}\\
        \leq & ~ M^{\frac{1}{1-\alpha}}\cdot [(\frac{1}{M}\sum_{\mu=1}^Mp_\mu)^\alpha]^{\frac{1}{1-\alpha}}\annot{By the Jensen's Inequality}\\
        = & ~ M~.
    \end{align*}
    For $1>\alpha>0$, then $1\leq\sum_{\mu=1}^Mp_\mu^\alpha$ and the similar derivation holds.
    We note that it saturates $1$ when $p$ is a vertex on the simplex, i.e., $p_{\mu_{\text{max}}}=1$ and $p_{\mu\neq\mu_{\text{max}}}=0$, where $\mu_{\text{max}}\coloneq\arg\max_\mu~p_\mu$ is non-degenerate. 
    It saturates the Jensen's inequality when $p_\mu=M^{-1}$ for every $\mu$.
\end{proof}

\begin{remark}
    By \cref{def:entropy_function},
    \begin{align*}
        K_{\text{eff}}^{(\alpha)}[p]=(1-(\alpha-1)\Psi_\alpha[p])^{\frac{1}{1-\alpha}}~,
    \end{align*}
    we differentiate it with respect to $\Psi_\alpha$. 
    It yields $(K_{\text{eff}}^{(\alpha)}[p])^{\alpha/(1-\alpha)}>0$.
    Hence, the effective support size is strictly increasing (monotonic) with respect to the Tsallis entropy.
    The higher the entropy is, the larger the effective support size is.
\end{remark}

\begin{lemma}[Threshold of Probability Function \cite{martins2016softmax,peters2019sparse}]
\label{lem:threshold_p_mu}
    The minimizer, $p\in\Delta^M$, of $\Gamma^{(\text{mod})}_\alpha[x,p,\xi]$ satisfies
    \begin{align}
    \label{eq:threshold_p_mu}
        p_\mu=(\frac{\alpha-1}{\alpha}\beta[z_\mu-\theta]_+)^{\frac{1}{\alpha-1}}~,
    \end{align}
    where $[u]_+\coloneq\max\{u,0\}$. 
    We determine $\theta$ by the normalization condition, $\sum_{\mu=1}^Mp_\mu=1$.
    This truncation only appears for $\alpha>1$.
    The support of the minimizers is thereby
    \begin{align*}
        \text{supp}(p)=\{\mu\in[1,M]:z_\mu>\theta\}~.
    \end{align*}
    We define $|\text{supp}(p)|$ as the exact support size.
\end{lemma}

\begin{proof}
    We include the proof for completeness, adapting the standard argument from \cite{martins2016softmax,peters2019sparse} to our notation.
    Since $0\leq p_\mu\leq1$ and $\sum_{\mu=1}^Mp_\mu=1$, we impose these constriants as Lagrangian multiplier terms:
    \begin{align*}
        \mathcal{L}(p,z,\theta,\lambda)=\beta\sum_{\mu=1}^Mz_\mu p_\mu+\frac{1}{\alpha-1}(1-\sum_{\mu=1}^Mp_\mu^\alpha)+\beta\theta(\sum_{\mu=1}^Mp_\mu-1)+\sum_{\mu=1}^M\lambda_\mu p_\mu~.
    \end{align*}
    Let the optimal saddles that maximizes the Lagrangian function be $\bar{p}$, $\bar{z}$, $\bar{\theta}$, and $\bar{\lambda}$. 
    Then the Karush-Kuhn-Tucker (KKT) conditions are
    \begin{align*}
        \bar{p}_\mu\geq0~,~\sum_{\mu=1}^M\bar{p}_\mu=1 & , \annot{Primal}\\
        \bar{\lambda}_\mu\geq0 & ,  \annot{Dual}\\
        \bar{\lambda}_\mu \bar{p}_\mu=0  & , \annot{Complementary}\\
        \frac{\partial\mathcal{L}}{\partial p_\mu}|_{p=\bar{p},z=\bar{z},\theta=\bar{\theta},\lambda=\bar{\lambda}}=\beta\bar{z}_\mu-\frac{\alpha}{\alpha-1}(\bar{p}_\mu)^{\alpha-1}+\beta\bar{\theta}+\bar{\lambda}_\mu=0 & , \annot{Stationary}
    \end{align*}
    for all $\mu\in[1,M]$.
    By solving the stationary equation, for $\bar{p}_\mu>0$ and $\alpha>1$,
    \begin{align*}
        \bar{p}_\mu=(\frac{\alpha-1}{\alpha}\beta[\bar{z}_\mu+\bar{\theta}])^{\frac{1}{\alpha-1}}~.
    \end{align*}
    For $\bar{p}_\mu=0$ and $\alpha>1$, 
    \begin{align*}
        \bar{z}_\mu+\bar{\theta}+\bar{\lambda}_\mu=0~\implies~\bar{z}_\mu+\bar{\theta}\leq0~.
    \end{align*}
    Hence, after replacing $\bar{\theta}\to-\bar{\theta}$, we have
    \begin{align*}
        p_\mu=(\frac{\alpha-1}{\alpha}\beta[z_\mu-\theta]_+)^{\frac{1}{\alpha-1}}~,
    \end{align*}
    where we neglect bars to avoid redundant notations.
\end{proof}

\begin{remark}
    For $0<\alpha<1$, the exponent of \eqref{eq:threshold_p_mu} is negative, so the base should remain positive.
    In this way, there is no truncation and all minimizers $p_\mu>0$.
    For $\alpha\to1$, the Shannon entropy has the stationary equation,
    \begin{align*}
        \beta z_\mu-\log(p_\mu)-1+\beta\theta+\lambda_\mu=0~.
    \end{align*}
    The minimizer, $p_\mu$, is positive everywhere and does not truncate for finite $(\theta,z_\mu,\lambda_\mu)$.
\end{remark}

\begin{definition}[Sparsity \cite{martins2016softmax,peters2019sparse}]
    We say $p\in\Delta^M$ is $k$-sparse if the exact support size satisfies
    \begin{align*}
        |\text{supp}(p)|=k~,
    \end{align*}
    where $0\leq k\leq M$.
    The large the $k$ is, the less the sparsity is.
    In other words, if more points satisfy $p=0$, then the system is sparser.
\end{definition}

\begin{remark}
\label{rmk:sparsity_alpha}
    Higher entropy suggests larger effective support size but not higher sparsity. The sparsity depends on the boundary geometry of the entropy function. 
    \cref{lem:threshold_p_mu} states that only for $\alpha>1$, there is a truncation and several $p_\mu$ are zero.
    It is consistent with the argument in \cref{prp:fenchel_conjugate_tsallis_entropy}.
    Due to \cref{def:essentially_smooth}, the Tsallis entropy is not essentially smooth only for $\alpha>1$.
    Therefore, for $\alpha>1$, the optimizer sit on the boundaries of the simplex when the corresponding $z_\mu$ is below the threshold.
    This introduces the sparsity.
    Otherwise, the divergence of gradient on the boundary forbids the sparsity.
\end{remark}

\subsection{Settings}
\label{sec:exp_settings}

\textbf{Setup.} We consider a system of dimension $d$ with $M=e^{\kappa d}$ normalized stored patterns, $\{\xi^\mu\}_{\mu=1}^M$.
To probe the energy landscape, $\Gamma_\alpha^{(\text{mod})}$, we perform continuous scans by varying $\beta$ or $\kappa$.
For each scan, we use both up and down sweeps to detect hysteresis and metastability.

\begin{enumerate}

    \item \textbf{Up Sweep}: Initialize from low $\beta$ (or low $\kappa$) and increase.
    
    \item \textbf{Down Sweep}: Initialize from high $\beta$ (or high $\kappa$) and decrease.
    
\end{enumerate}

At each parameter value ($\beta$ or $|\kappa$), we initialize the optimization from the converged solution at the previous step (warm start). 
Then we minimize the free energy with respect to $x$ by gradient descent, 
\begin{align*}
    x(t+1)=x(t)-\eta(t)\nabla_x\Gamma_{\alpha}^{(\text{mod})}(x(t)),\qquad\eta(t)=\eta_0 2^{-m(t)},
\end{align*}
where $\eta_0=0.5$ and $m(t)\coloneq\min\{m\in\mathbb{N}_0\mid\Gamma_{\alpha}^{\text{(mod)}}(x(t+1))\leq\Gamma_{\alpha}^{\text{(mod)}}(x(t))\}$.
This procedure tracks continuous branches of local minima of $\Gamma^{(\text{mod})}_\alpha$.
To explore competing solutions, we perform multi-branch continuation.

\textbf{Data.} In our experiment, we consider clustered normalized patterns.
\begin{align*}
    \xi^\mu_i=\sqrt{1-\delta}\eta^{c(\mu)}_i+\sqrt{\delta}\zeta^\mu_i,
\end{align*}
where $c(\mu)\in\{1,...,C\}$ assigns pattern $\mu$ to a cluster, $\eta_i^{c(\mu)}$ are cluster centers sampled as $\eta^c\sim\mathcal{N}(0,\one_d)$, $\zeta^\mu\sim\mathcal{N}(0,\one_d)$ are i.i.d. normalized noise vectors, and $\delta\in[0,1]$ indicates the noise strength.
This construction introduces off-diagonal correlations in the Gram matrix, which makes the simulation more realistic than the i.i.d. setting.
In our simulation, we choose $C=4$ and $\delta=0.04$.

\textbf{Observables.} To characterize the coexistence regimes and retrieval success, we measure the following quantities.
\begin{enumerate}
    \item \textbf{Maximum Overlap:} 
    To verify whether a stable point aligns with the stored patterns, we measure the maximum overlap.
    $z^\star_{\text{max}}=\max_\mu|z^\star_\mu|$.

    \item \textbf{Exact Support Size:}  
    Each stable point has its own support and exact support size. 
    For example, $z^\star_\mu=\sum_{i=1}^dx_i^\star\xi_i^\mu/\sqrt{d}$ gives a specific value of $p^\star_\mu$ via \cref{lem:threshold_p_mu}, which determines $K_{\text{eff}}^{(\alpha>1)}$. 
    For $\alpha=2$, we measure $|\text{supp}(p^\star)|=\#\{\mu\in[1,M]:z^\star_\mu>\theta\}$ according to \cref{lem:threshold_p_mu}, where $\#$ indicates the cardinality.
    For $\alpha=1$, we measure $\{\mu\in[1,M]:p^\star_\mu>\varepsilon\}$, where we choose $\varepsilon=10^{-10}$ in our simulation.
    Hence, we denote the measured exact support size as
    \begin{align*}
        |\text{supp}(p^\star)|\coloneq\begin{cases}
            \#\{\mu\in[1,M]:p_\mu^\star>10^{-10}\}~&,~\alpha=1\\
            \#\{\mu\in[1,M]:z_\mu^\star>\theta\}~&,~\alpha=2
        \end{cases}.
    \end{align*}
    
    \item \textbf{Fixed Scored-Threshold Support Size:}
    In temperature analysis in the sparse case, we further define
    \begin{align*}
        K_{\Delta}\coloneq\{\mu\in[1,M]:z_\mu^\star>\Delta\},
    \end{align*}
    and we choose $\Delta=0.01$.
    Hence, it serves as a diagnostic for counting how many memory components achieve large query-key alignment.
    
    \item \textbf{Inverse Participation Ratio (IPR):} It is defined as the effective support size with $\alpha=2$, i.e., 
    $K_{\text{eff}}^{\text{IPR}}\coloneq K_{\text{eff}}^{(2)}=(\sum_\mu (p_\mu^\star)^2)^{-1}$.
    If IPR is large, then there are more components of $p^\star_\mu$ that are non-zero.

    \item \textbf{Free Energy:} 
    Distinct stable points have distinct free energy in general.
    To see which is more preferable, i.e., which one has the lowest free energy, we measure $\Gamma^{(\text{mod})}_\alpha[x,\xi]$.
\end{enumerate}

\paragraph{Results.}
We display the results in \cref{fig:alpha2_phase_transition_beta,fig:alpha1_phase_transition_beta,fig:alpha2_phase_transition_kappa,fig:alpha1_phase_transition_kappa}.
Please see \cref{sec:numerical_results} for detailed illustrations.

\clearpage

\section{Additional Experiments}
\label{sec:add_exp}
In this section, we first show the experiment of dynamical stability in modern Hopfield models.
Then we analyze the empirical retrieval behaviors in realistic end-to-end Transformer-architectures.
Finally, we validate the special coexistence point formula \cref{prp:coex_point} and spinodal formula \cref{prp:spinodal_point} numerically.

\subsection{Dynamical Stability Analysis for Temperature}

In this subsection, we provide the numerical validation for the theoretical dynamical retrieval stability condition; please see \cref{sec:dynamical_stability} for a detailed discussion and derivations.
Specifically, we only analyze the ordinary gradient descent (discrete) algorithm and verify the bound in \cref{prp:dynamical_retrieval_stability_degeneracy}.

\paragraph{Setup.}
We consider the number of pattern be $M=100$, pattern dimension be $d=64$, and inverse-temperature be $\beta=300$ as an example.
For simplicity, we apply fixed-step gradient descent
\begin{align*}
    x(t+1)=x(t)-\eta\nabla_x\Gamma^{(\text{mod})}_\alpha(x(t)),
\end{align*}
where $\eta$ is a constant.
Let $\eta_{\text{crit}}=2/\lambda_{\text{max}}(\nabla^2\Gamma_\alpha^{(\text{}mod)}(x^\star))$ be the theoretical critical learning rate threshold, $\lambda_{\text{max}}(\cdot)$ be the maximal eigenvalue, and $x^\star$ be a fixed local stable point.
We investigate the convergence through iteration numbers, $t$.
We define the initial condition as $x(t=0)=x^\star+A u_{\text{soft}}$, where $A=10^{-2}$ is a small amplitude and $u_{\text{soft}}$ is the eigen-direction of the Hessian with the smallest eigenvalue.
Notice that we derive $u_{\text{soft}}$ numerically, so it is not the exact eigenvector but exhibits tiny component in other directions.

\paragraph{Data.}
We use the same pattern set $\{\xi^\mu\}_{\mu=1}^M$ as in the previous experiments, where each pattern $\xi^\mu\in\R^d$ is normalized.
We generate this collection through the method in \cref{sec:exp_settings}.

\paragraph{Observables.}
We track two quantities along the gradient-descent trajectory: the distance to the fixed point, $\|x(t)-x^\star\|_2$ and the maximal overlap, $z_{\text{max}}(x(t))\coloneq \max_\mu z_\mu(x(t))$.

\paragraph{Results.}
In the sparse regime ($\alpha=2$), we demonstrate the result in \cref{fig:dyn_alpha2}.

\begin{figure}[t]
    \centering
    \includegraphics[width=1\linewidth]{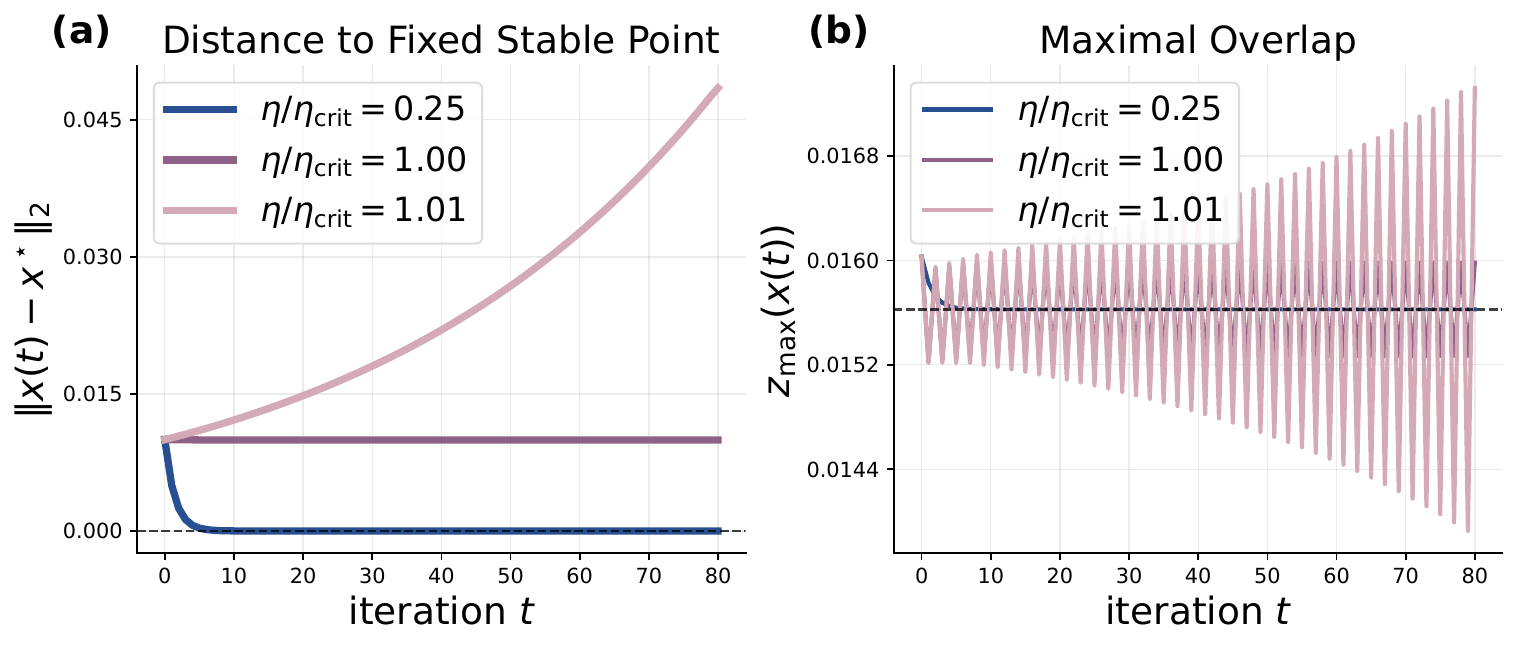}
    \caption{\textbf{Dynamical retrieval stability of MHM gradient descent in the sparse regime ($\alpha=2$).}
    {\small ($M=100,d=64$).
    We study the gradient descent update $x(t+1)=x(t)-\eta\nabla\Gamma_\alpha^{\text{(mod)}}(x(t))$. 
    We denote the critical learning rate as $\eta_{\text{crit}}=2/\lambda_{\text{max}}(\nabla^2\Gamma_\alpha^{(\text{mod})}(x^\star))$ with $\lambda_{\text{max}}$ the largest eigenvalue at a fixed stable point $x^\star$.
    \textbf{(a) Distance to the fixed point.}
    For $\eta<\eta_{\text{crit}}$, the training converges to the fixed point. 
    Otherwise, the dynamics is unstable or even diverges.
    \textbf{Maximal overlap.}
    For $\eta<\eta_{\text{crit}}$, the training converges to the fixed value of overlap $z_{\text{max}}\approx 0.0156$. 
    It is the saturation value of one-pattern perfect retrieval in \cref{fig:alpha2_phase_transition_beta} (a).
    Otherwise, the dynamics is unstable or even diverges.
    }}
    \label{fig:dyn_alpha2}
\end{figure}

\begin{enumerate}[label=(\alph*)]

    \item  For $\eta<\eta_{\text{crit}}$, the distance (blue curve) approaches to zero asymptotically, indicating that the training $x(t)$ converges to the fixed stable point, $x^\star$.
    At $\eta\approx\eta_{\text{crit}}$, the convergence becomes marginal.
    For $\eta>\eta_{\text{crit}}$, the training $x(t)$ deviates from the local point as iteration increases.
    This indicates a dynamical instability.
    Therefore, this plot provides the evidence that $\eta_{\text{crit}}$ separates the convergent regime. 

    \item The maximal overlap $z_{\text{max}}(x(t))$ exhibits the same behavior. 
    For $\eta<\eta_{\text{crit}}$, $z_{\max}(x(t))$ converges to the fixed value $z_{\max}(x^\star)$. 
    This indicates stable retrieval of the dominant memory pattern.
    Notably, the convergent value is $z_{\text{max}}(x^\star)\approx1/d\approx0.0156$, which is the one-pattern retrieval in \cref{fig:alpha2_phase_transition_beta} (a).
    At $\eta\approx\eta_{\text{crit}}$, it exhibits persistent oscillations around $z_{\text{max}}(x^\star)$. 
    For $\eta>\eta_{\text{crit}}$, the oscillations grow, indicating a dynamical instability in retrieval.
    
\end{enumerate}

In the dense regime ($\alpha=1$), we demonstrate the result in \cref{fig:dyn_alpha1}.

\begin{figure}[t]
    \centering
    \includegraphics[width=1\linewidth]{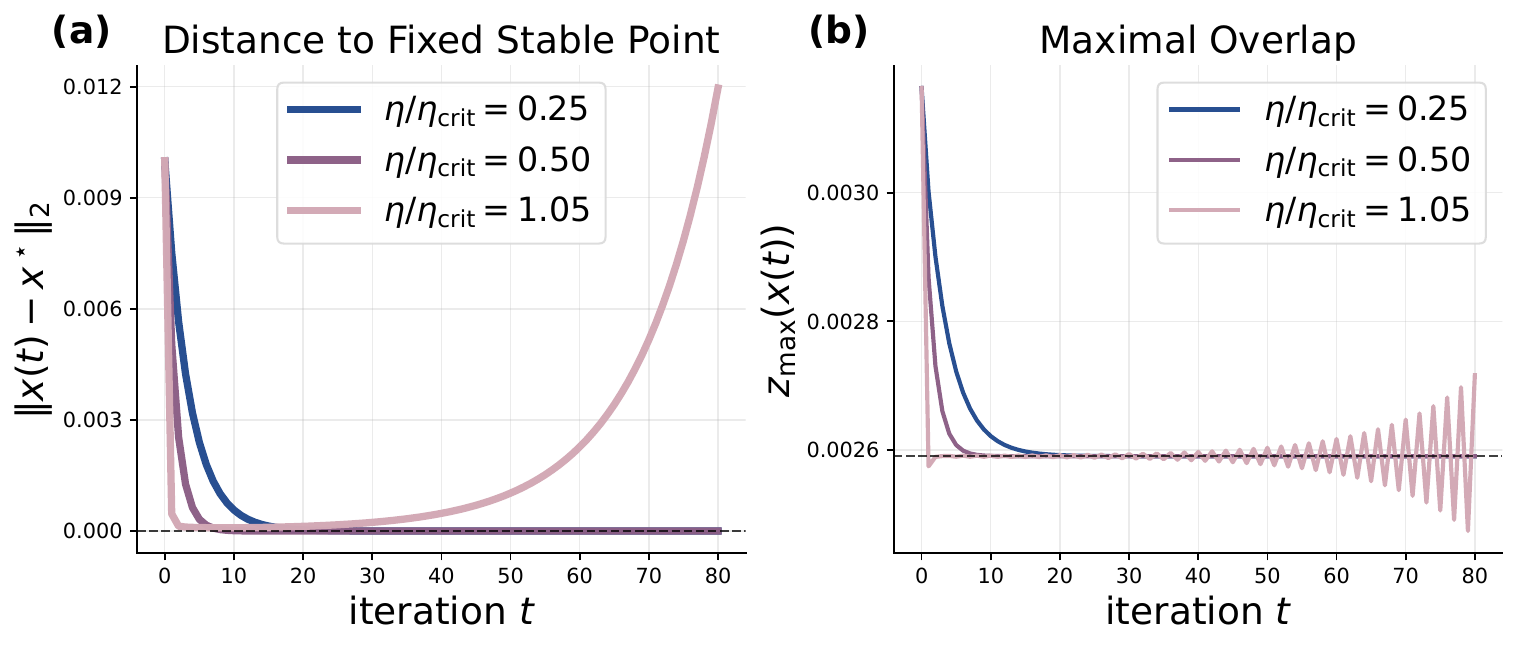}
    \caption{\textbf{Dynamical retrieval stability of MHM gradient descent in the dense regime ($\alpha=1$).}
    {\small ($M=100,d=64$).
    We study the gradient descent update $x(t+1)=x(t)-\eta\nabla\Gamma_\alpha^{\text{(mod)}}(x(t))$. 
    We denote the critical learning rate as $\eta_{\text{crit}}=2/\lambda_{\text{max}}(\nabla^2\Gamma_\alpha^{(\text{mod})}(x^\star))$ with $\lambda_{\text{max}}$ the largest eigenvalue at a fixed stable point $x^\star$.
    \textbf{(a) Distance to the fixed point.}
    For $\eta<\eta_{\text{crit}}$, the training converges to the fixed point. 
    For $\eta>\eta_{\text{crit}}$, a divergence follows after a decay.
    This reflects that the iterate close to $x^\star$ initially.
    After this, the unstable eigen-directions start to dominate.
    We provide the precise explanation in this section.
    \textbf{(b) Maximal overlap.}
    For $\eta<\eta_{\text{crit}}$, the training converges to the fixed value of overlap $z_{\text{max}}$. 
    For $\eta>\eta_{\text{crit}}$, a growing oscillations follow after an initial convergence.
    This is similar for the behavior in the distance plot.
    }}
    \label{fig:dyn_alpha1}
\end{figure}

\begin{enumerate}[label=(\alph*)]

    \item For $\eta<\eta_{\text{crit}}$, the iteration converges to the stable fixed point $x^\star$ rapidly. 
    Larger the learning rate is, more rapid the convergence is. 
    For $\eta>\eta_{\text{crit}}$, the distance declines first, and then grows up after some number of iteration.
    This phenomenon arises from unstable modes being small due to our initialization $x(t=0)$.
    In particular, from \eqref{eq:delta_recursion}, we have $\delta_i(t+1)=(1-\eta\lambda_i)\delta_i(t)$,
    where $\delta_i(t)\equiv x_i(t)-x_i^\star$, and $\lambda_i$ is the $i$-th eigenvalue of the Hessian $\nabla^2_x\Gamma_\alpha^{(\text{mod})}$.
    The convergence condition is 
    \begin{align*}
        |1-\eta\lambda_i|<1,\qquad\text{or }~0<\lambda_i<2/\eta.
    \end{align*}
    Hence, even if $\eta>\eta_{\text{crit}}$, some modes $x_i(t)$ still converge to the fixed point $x_i^\star$ if the eigenvalues satisfy $\lambda_i<2/\eta$.
    Initially, we consider the dominant mode be which with $\lambda_{\text{min}}$.
    Hence, it converges to the fixed point in the early stage.
    However, for $\lambda_i>2/\eta$, $\delta_i(t)$ grows with $t$.
    Hence, the unstable modes start to dominate, and the trajectory diverges.

    \item For $\eta<\eta_{\text{crit}}$, the iteration converges to the stable fixed value $z_{\text{max}}(x^\star)$.
    Larger the learning rate is, more rapid the convergence is.
    For $\eta>\eta_{\text{crit}}$, it exhibits the same behavior as in distance analysis.
    The dynamics converges first, then it diverges after certain iteration.
    This arises from unstable modes being small due to our initialization $x(t=0)$ as well.
    
\end{enumerate}

\subsection{Support and Capacity-Reliability Trade-off}
\label{sec:cap_reliability_memory_load}

We study how increasing the number of memories affects retrieval reliability and effective support (IPR) in a trained Transformer.
We compare three
attention maps, $\alpha\in\{1,1.5,2\}$ (softmax, entmax-1.5, and sparsemax), respectively.

\paragraph{Setup.}
The data generation pipeline is
\begin{align*}
    u_Q\xrightarrow[\text{fixed}]{A_Q}q^{\text{in}},\qquad\text{and }~u_\mu\xrightarrow[\text{fixed}]{A_K}k^{\text{in}}_\mu.
\end{align*}
We generate latent query and key variables and map them to observed inputs.
After that, we input $q^{\text{in}}$ and $k^{\text{in}}_\mu$ into a multi-head cross-attention Transformer, yielding the output $\hat{y}\in\R^{d_v}$.

\paragraph{Data.}
Each example contains a query $q^{\text{in}}\in\R^{d_{\text{in}}}$, a set of memory keys $k_\mu^{\text{in}}\in\R^{d_{\text{in}}}$, and independent values $v_\mu\sim\mathcal{N}(0,\one_{d_v})$ with $d_v=32$.
The latent dimension is $d_{\text{lat}}=16$, the input dimension is $d_{\text{in}}=24$, and the number of latent clusters is $C=4$.
For each example, we draw normalized cluster centers independently
\begin{align*}
    c_r=\frac{\tilde{c}_r}{\|\tilde{c}_r\|_2}.\qquad \tilde{c}_r\sim\mathcal{N}(0,\one_{d_{\text{lat}}}),\qquad r\in\{1,\cdot\cdot\cdot,C\}.
\end{align*}
Memory $\mu$ is assigned a cluster index $a(\mu)$, and its latent key is
\begin{align*}
    u_\mu=\frac{\sqrt{1-\delta}\,c_{a(\mu)}+\sqrt{\delta}\,\epsilon_\mu}{\|\sqrt{1-\delta}\,c_{a(\mu)}+\sqrt{\delta}\,\epsilon_\mu\|_2},\qquad\epsilon_\mu\sim\mathcal{N}(0,\one_{d_{\text{lat}}}).
\end{align*}
We choose $\delta=0.12$ to control the latent key variance.
The observed query and keys are then
\begin{align*}
    q^{\text{in}}=A_Q u_Q+\sigma_{\text{obs}}\eta_Q,\qquad k^{\text{in}}_\mu=A_K u_\mu+\sigma_K\eta_\mu,
\end{align*}
where $A_Q,A_K\in\R^{d_{\text{in}}\times d_{\text{lat}}}$ with the orthogonalization $A_Q^\top A_Q=A_K^\top A_K=\one_{\text{lat}}$, $\eta_Q,\eta_\mu\sim\mathcal{N}(0,\one_{\text{in}})$, $\sigma_{\text{obs}}=0.03$, and $\sigma_K=0.05$. Here we divide $u_Q$ generation into two cases, the selective retrieval case and the uniform-aggregation case.
\begin{enumerate}
    \item \textbf{Selective retrieval.}
    In this case, $u_Q$ retrieves one particular memory.
    Consider that we sample a target memory index $t$ from $\{1,...,M\}$. Therefore, we denote $u_t$ as the latent key located at the sampled target slot $t$.
    We then generate the latent query by adding noise
    \begin{align*}
        u_Q=\frac{u_t+\sigma_Q\epsilon_Q}{\|u_t+\sigma_Q\epsilon_Q\|_2},\qquad \epsilon_Q\sim\mathcal{N}(0,\one_{d_{\text{lat}}}).
    \end{align*}
    During training, the query-noise level cycles through $\sigma_Q\in\{0.05,0.10,0.15\}$ and the number of memories cycles through $M=\{8,16,32,64\}$ simultaneously.
    During evaluation, the query noise is fixed at $\sigma_Q=0.10$ while $M$ is increasing.

    \item \textbf{Uniform-aggregation.}
    In this case, we sample a target cluster $r^\star$. 
    Then the latent query is
    \begin{align*}
        u_Q=\frac{c_{r^\star}+\sigma_Q\epsilon_Q}{\|c_{r^\star}+\sigma_Q\epsilon_Q\|_2},\qquad \epsilon_Q\sim\mathcal{N}(0,\one_{d_{\text{lat}}}).
    \end{align*}
    We denote the relevant set as $R=\{\mu:a(\mu)=r^\star\}$.
    Since there are $C=4$ clusters, we choose $|R|=M/4$ in this case.
\end{enumerate}

\paragraph{Transformer.}
Inside the Transformer, we apply one-layer multi-head cross-attention.
Specifically, we conduct the first step via
\begin{align*}
    q_h=W^Q_h\text{LN}_Q\Big(W_Q^{\text{in}}q^{\text{in}}+r_Q\Big),\qquad k_{h,\mu}=W^K_h\text{LN}_K\Big(W_K^{\text{in}}k^{\text{in}}_\mu+r_K\Big),
\end{align*}
\begin{align*}
    v_{h,\mu}=W^V_h\text{LN}_V\Big(W^{\text{in}}_Vv_\mu+r_V\Big),
\end{align*}
where $W_Q^{\text{in}},W_K^{\text{in}}\in\R^{d_{\text{mod}}\times d_{\text{in}}}$, $W_V^{\text{in}}\in\R^{d_{\text{mod}}\times d_v}$,  $W_h^Q,W_h^K,W_h^V\in\R^{d_h\times d_{\text{mod}}}$, LN denotes the LayerNorm, and $h=\{1,\cdot\cdot\cdot,H\}$.
Then the multi-head output is followed by two residual updates
\begin{align*}
    x_1=W_Q^{\text{in}}q^{\text{in}}+r_Q+W^O\text{Concat}\Big(\sum_{\mu=1}^Mp_{1,\mu}v_{1,\mu},\cdot\cdot\cdot,\sum_{\mu=1}^Mp_{H,\mu}v_{H,\mu}\Big)\in\R^{d_{\text{mod}}},
\end{align*}
and
\begin{align*}
    x_2=x_1+\text{FFN}\Big(\text{LN}_{\text{FFN}}(x_1)\Big)\in\R^{d_{\text{mod}}},
\end{align*}
where $W^O\in\R^{d_{\text{mod}}\times d_{\text{mod}}}$, $p_{h,\mu}$ is the attention with the form \cref{prp:eom_attention}, and FFN$(x)\coloneq W_2^{\text{FFN}}\text{GELU}(W_1^{\text{FFN}}x)$.
Finally, the final predicted output is
\begin{align*}
    \hat{y}=W^{\text{out}}\text{LN}_{\text{final}}(x_2)\in\R^{d_v},
\end{align*}
where $W^{\text{out}}\in\R^{d_v\times d_{\text{mod}}}$.
We choose the architectural hyperparameters as
\begin{itemize}
    \item Model dimension $d_{\text{mod}}=d_{h}\times H=96$.

    \item 4 attention heads ($H=4$) with head dimension $d_h=24$.

    \item FFN with hidden dimension 192, i.e., $W_1^{\text{FFN}}\in\R^{192\times96}$, and $W_2^{\text{FFN}}\in\R^{96\times192}$.
    
\end{itemize} 

In the Training stage, we choose 2500 AdamW steps, batch size $B=64$, learning rate $\eta=3\times10^{-4}$, weight decay $10^{-4}$, MSE loss, $\beta=1$, $M\in\{8,16,32,64\}$.
Specifically, the MSE we evaluate is
\begin{align*}
    \text{MSE}=\frac{1}{B}\sum_{b=1}^B\Big\|\hat{y}^{(b)}-y_{\text{target}}^{(b)}\Big\|^2_2,
\end{align*}
where $y^{(b)}_{\text{target}}=y_{t_b}$ for the selective retrieval case, and $y^{(b)}_{\text{target}}=\sum_{\mu\in R_b}y_\mu^{(b)}/|R_b|$ for the uniform-aggregation case.

In inference stage, we fix $\beta=1$ and the learned parameters, e.g., $W_h^Q,W^O,r_Q,\cdot\cdot\cdot$. Then we increase the memory load from $M=8$ to $M=4096$ to measure how good the prediction $\hat{y}$ is.

\paragraph{Observables.}
We measure 
\begin{itemize}
    \item \textbf{Selective retrieval: Value top-1 accuracy.}
    We define the value top-1 accuracy as
    \begin{align*}
        A_{\text{value}}\coloneq\frac{1}{B}\sum_{b=1}^B\one\Big\{\arg\min_\mu\big\|\hat{y}^{(b)}-y^{(b)}_\mu\big\|_2^2=t_b\Big\}.
    \end{align*}
    This measures end-to-end success, rather than simply measures whether attention puts its largest weight on the correct key.
    
    \item \textbf{Selective retrieval: Support size (IPR).}
    As introduced in \cref{sec:exp_settings}, we measure the support via the inverse partition ratio (IPR).
    However, since we consider batch size and multi-head here, we further average over them, i.e., \begin{align*}
        K^{\text{IPR}}_{\text{eff}}\coloneq\frac{1}{HB}\sum_{b=1}^B\sum_{h=1}^H\frac{1}{\sum_{\mu=1}^Mp_{b,h,\mu}^2}.
    \end{align*}

    \item \textbf{Uniform-aggregation: Error/zero-output baseline.}
    The target is $y_{\text{target}}=\sum_{\mu\in R}y_\mu/|R|$, where $y_\mu\sim\mathcal{N}(0,\one_{d_v})$.
    Hence, the zero-output baseline is $E_0\coloneq\mathbb{E}\|y_{\text{target}}\|^2_2=d_v/|R|$.
    In this way, we define the error/zero-output baseline $E_{\text{rel}}$ as
    \begin{align*}
        E_{\text{rel}}\coloneq\frac{\text{MSE}}{E_0}=\frac{|R|}{Bd_v}\sum_{b=1}^B\Big\|\hat{y}^{(b)}-\frac{1}{|R_b|}\sum_{\mu\in R_b}y^{(b)}_{\mu}\Big\|^2_2.
    \end{align*}

    \item \textbf{Uniform-Aggregation: IPR / $\boldsymbol{|R|}$.}
    We measure the normalized IPR with the size of the target set
    \begin{align*}
        \rho_{\text{IPR}}\coloneq\frac{K_{\text{eff}}^{\text{IPR}}}{|R|}=\frac{1}{HB|R|}\sum_{b=1}^B\sum_{h=1}^H\frac{1}{\sum_{\mu=1}^Mp_{b,h,\mu}^2}.
    \end{align*}
    
\end{itemize}

\paragraph{Results.}

We present the experiments as follows.

\begin{figure}[h]
    \centering
    \includegraphics[width=1\linewidth]{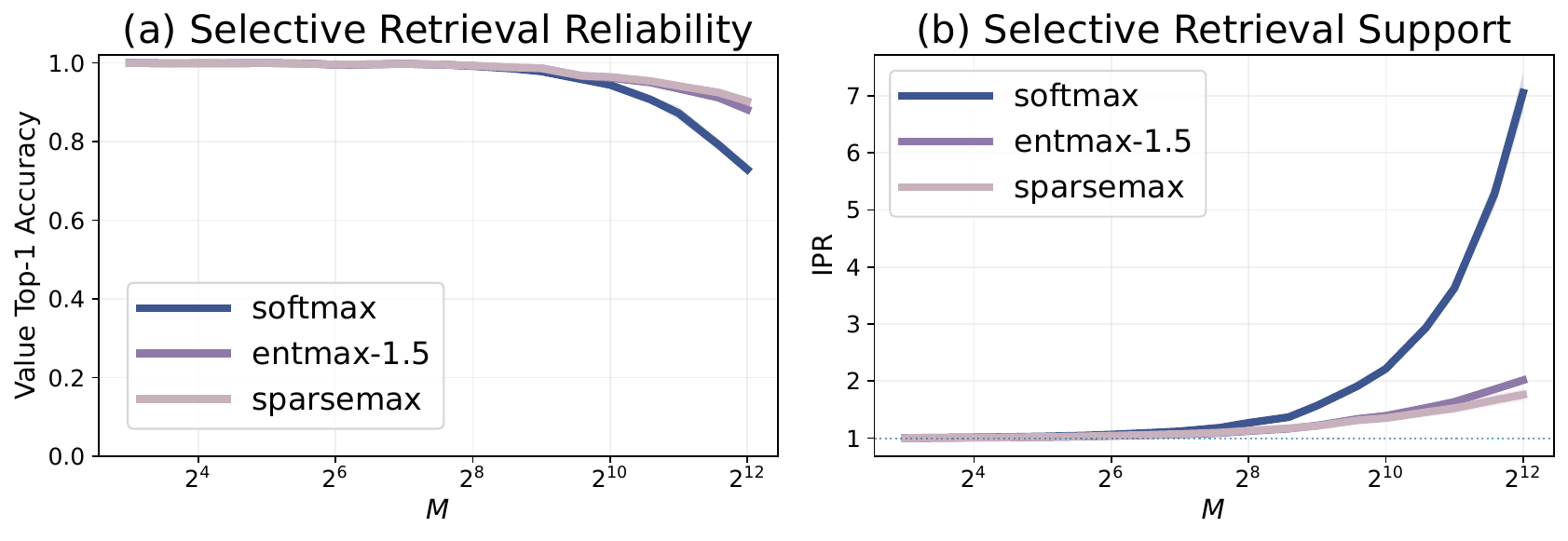}
    \caption{\textbf{Capacity-reliability trade-off for selective retrieval.}
    \small{
    We evaluate the trained Transformer at fixed $\beta_{\text{train}}=1$ while increasing $M$.
    \textbf{(a).} 
    We measure the value top-1 accuracy to characterize the retrieval reliability.
    Sparsemax retain the highest accuracy.
    Furthermore, their accuracy decreases as $M$ grows, indicating the capacity-reliability trade-off.
    \textbf{(b).}
    Softmax support remains full, while sparsemax use the least supports and is the most support-efficient.
    }}
    \label{fig:capacity_selective}
\end{figure}

From \cref{fig:capacity_selective} (a), the sparsemax arrives the best value top-1 accuracy, while the softmax arrives the worst one.
Furthermore, both of their accuracy decrease as $M$ increases.
This indicates that high memory load leads to lower reliability, and reveals the capacity-reliability trade-off.
From \cref{fig:capacity_selective} (b), the support of softmax remains full, while which of sparsemax is low.
This suggests that sparsemax is support-efficient.
Those plots conclude that for a selective task, sparsemax has not only high-accuracy but also high support-efficiency.

\begin{figure}[H]
    \centering
    \includegraphics[width=1\linewidth]{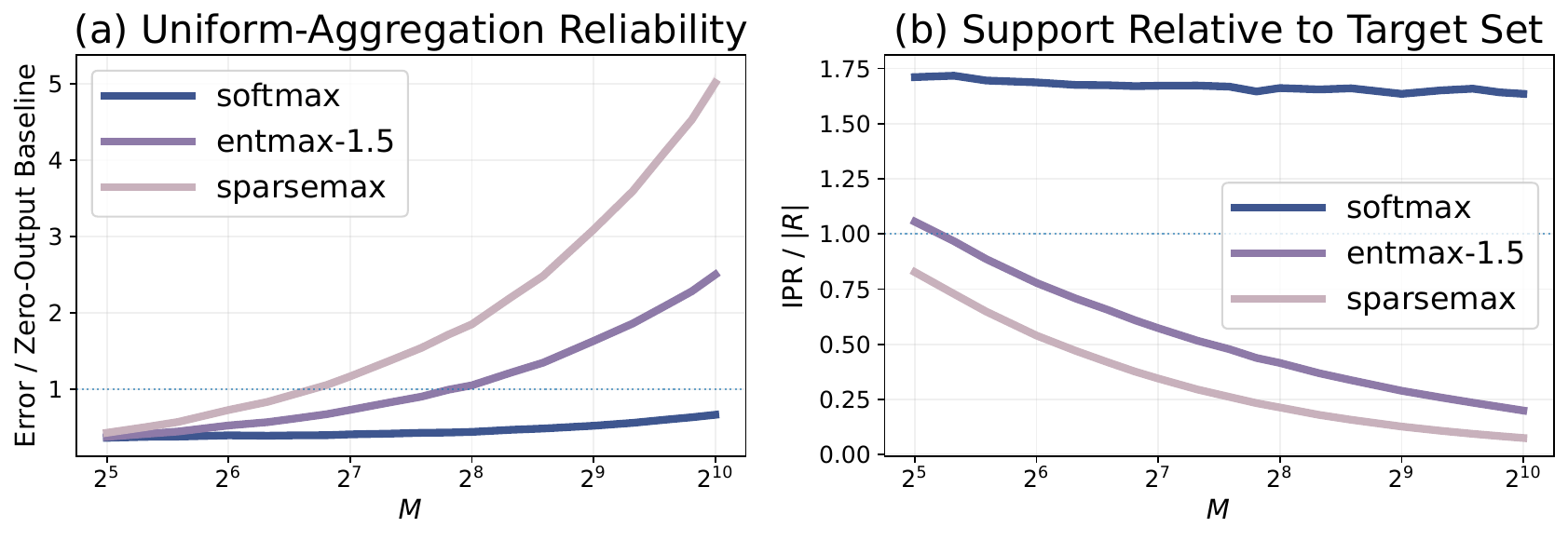}
    \caption{\textbf{Capacity-reliability trade-off for uniform-aggregated retrieval.}
    \small{
    We evaluate the trained Transformer at fixed $\beta_{\text{train}}=1$ while increasing $M$.
    \textbf{(a).} 
    We measure the ratio of the output error and the zero-output baseline to characterize the retrieval reliability.
    Sparsemax exhibits the worst reliability.
    Furthermore, their retrieval reliabilities become worse as $M$ grows, indicating the capacity-reliability trade-off.
    \textbf{(b).}
    Softmax support remains full, while sparsemax use the least supports and is the most support-efficient.
    }}
    \label{fig:capacity_aggregation}
\end{figure}

Compared to \cref{fig:capacity_selective}, \cref{fig:capacity_aggregation} (a) shows that softmax arrives the highest reliability among three attentions and the capacity-reliability trade-off as well.
From \cref{fig:capacity_aggregation} (b), we still observe that the sparsemax uses the least aggregated support size, and hence the most support-efficient.
In conclusion, although a large $\alpha$ leads to support-efficiency, the reliability depends on the task: sparsemax predicts a selective retrieval the best, while softmax predicts a uniform-aggregation retrieval the best.

\subsection{Support, Temperature, and Reliability}
\label{sec:support_tem_reliability}
In this section, we investigate the temperature effect on effective support size (IPR) and retrieval reliability in a realistic Transformer architecture numerically.
We compare three attention maps, $\alpha\in\{1, 1.5, 2\}$ (softmax,
entmax-1.5, and sparsemax), respectively.

\paragraph{Setup.}
We adopt the same setup and Transformer architecture as in \cref{sec:cap_reliability_memory_load}, i.e., one-layer four-head cross-attention Transformer with MSE training objective.
The only different setup is that we isolate $\beta$ by fixing $M$ at $M=128$ with query-noise level $\sigma_Q=0.10$ in the inference stage.
We vary the inverse-temperature over $\beta\in\{2^{-2},2^{-1},2^0,\cdot\cdot\cdot,2^9\}$.
Furthermore, we compare two protocals:

\begin{itemize}
    \item \textbf{Frozen-$\boldsymbol{\beta}$.}
    In this case, we first train each model at $\beta_{\text{train}}=1$.
    After training, we freeze every learned parameter ($W_h^Q,W^O,r_Q,$ and so forth).
    At inference, we observe how reliable the output performance is at each $\beta\in\{2^{-2},2^{-1},2^0,\cdot\cdot\cdot,2^9\}$.

    \item \textbf{Matched-$\boldsymbol{\beta}$.}
    In this case, we train a separate Transformer for every $\beta$.
    That is, for different $\beta$, we have different trained parameters $W_h^Q,W^O,r_Q,$ and so forth.
    
\end{itemize}

\paragraph{Observables.}
We adopt the same observables as which in \cref{sec:cap_reliability_memory_load}.

\paragraph{Results.}

We present the experiments as follows.

\begin{figure}[t]
    \centering
    \includegraphics[width=1\linewidth]{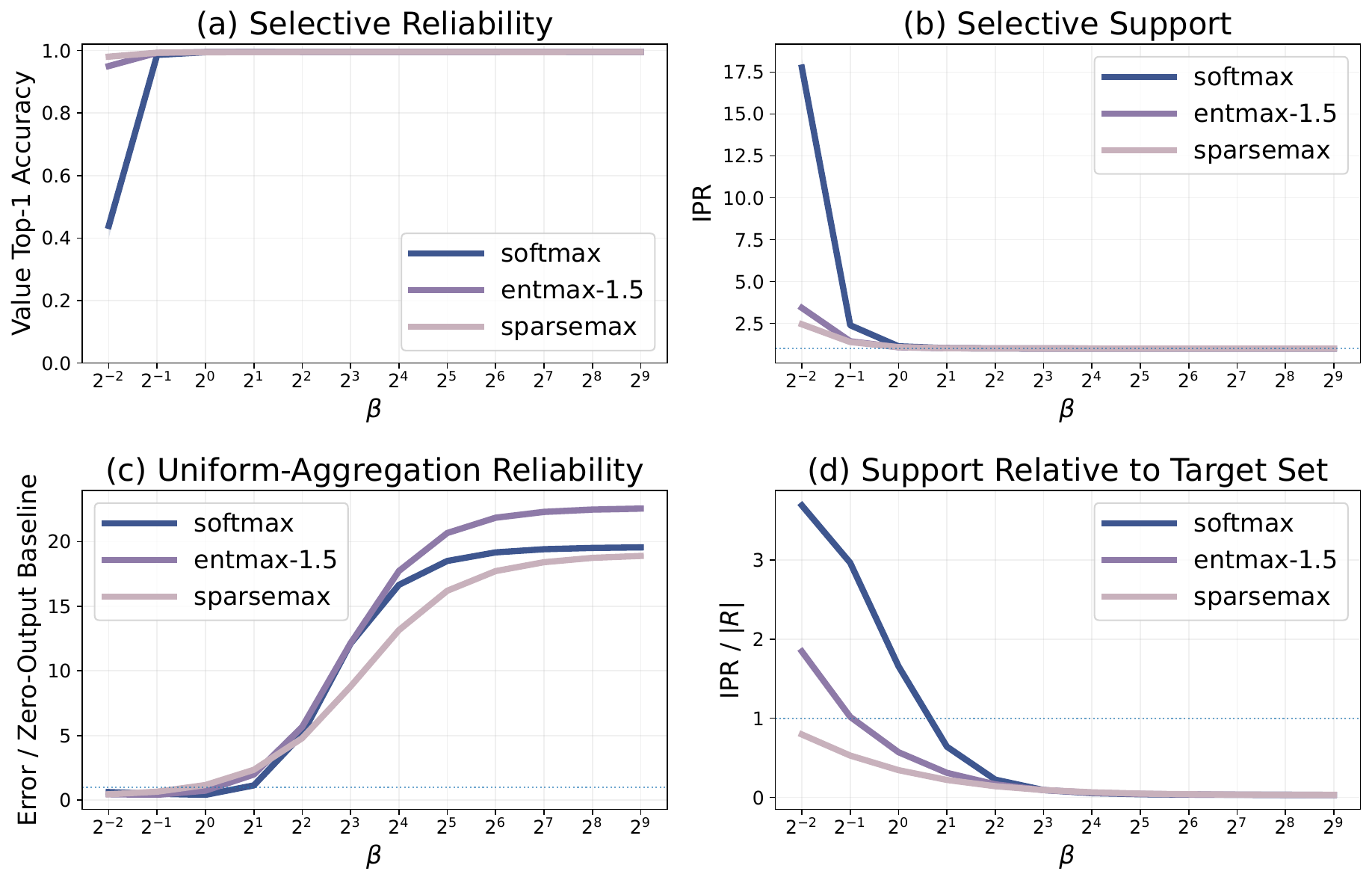}
    \caption{
    \textbf{Temperature effect on frozen-parameter Transformer.}
    \small{We train the Transformer parameters at $\beta_{\text{train}}=1$, then we fix those parameters, set $M=128$ and vary $\beta$ during inference.
    \textbf{(a).}
    The three attention families both exhibit high accuracy for a selective retrieval task.
    \textbf{(b).}
    For selective retrieval, the sparsemax uses the least effective support size (IPR), which shows the support-efficiency and reliability of high-$\beta$ sparsemax.
    \textbf{(c).}
    For a uniform-aggregation task, the three attention families have worse error relative to the zero-output baseline as $\beta$ increases.
    \textbf{(d).}
    The IPR relative to target-set size decline as $\beta$ grows, which indicates that the retrieval becomes selective.
    This leads to a mismatch for multi-memory aggregation.
    }}
    \label{fig:beta_frozen_M128}
\end{figure}

From \cref{fig:beta_frozen_M128} (a) and (b), we observe that sparsemax has highest accuracy and support-efficiency for frozen-$\beta$ case in a selective retrieval task.
Furthermore, both of the accuracy and IPR saturates at high $\beta$, arriving a single-pattern reliable retrieval for sparsemax.
For softmax, the support is full, but the IPR shows that its retrieval becomes strongly concentrated.
For the uniform-aggregation task, \cref{fig:beta_frozen_M128} (c) and (d) show that all three frozen models perform far worse than the zero-output baseline in the high-$\beta$ regime.
That is, since we train the parameters only at $\beta=1$, the Transformer exhibits greater selectivity, but creates a mismatch for multi-memory aggregation.

\begin{figure}[H]
    \centering
    \includegraphics[width=1\linewidth]{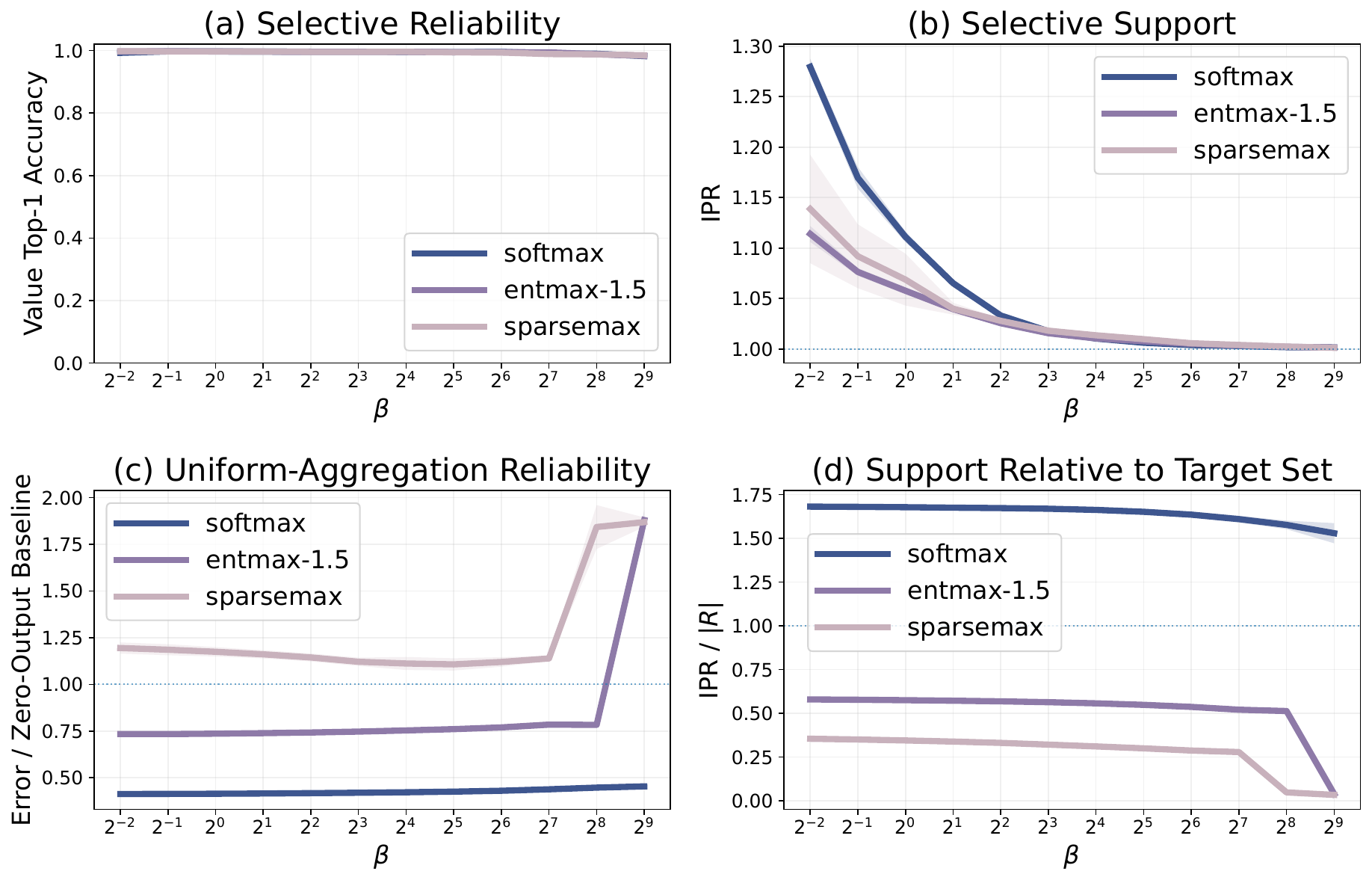}
    \caption{
    \textbf{Temperature effect on matched-parameter Transformer.}
    \small{ We train the Transformer parameters at each $\beta$, then we evaluate with $M=128$ during inference.
    \textbf{(a).}
    All three attention families maintain high accuracy in a selective retrieval task.
    \textbf{(b).}
    Sparsemax and entmax-1.5 are support-efficient in a selective retrieval task.
    \textbf{(c).}
    Softmax has the highest reliability among the three in a uniform-aggregation task.
    \textbf{(d).}
    Softmax uses numerous memories, which indicates a poor support-efficiency.
    }}
    \label{fig:beta_matched_M128}
\end{figure}

From \cref{fig:beta_matched_M128} (a) and (b), all three attention families maintain perfect accuracy in a selective retrieval task.
Furthermore, sparsemax and entmax-1.5 are more support-efficient that softmax.
For the uniform-aggregation task. \cref{fig:beta_matched_M128} (c) and (d) show that softmax achieves the highest reliability but its effective support remains broader than the relevant set.
Again, this indicates that high $\beta$ increases the selectivity, but achieves poor performance for an aggregation task.

\subsection{Numerical Validation of Coexistence and Spinodal Formula}

In this section, we validate the coexistence formula in \cref{prp:coex_point} and the spinodal formula in \cref{prp:spinodal_point} for $\alpha=2$ numerically.

\paragraph{Setup.}
We test the special two-key formulas under the uniform distribution and well-separation assumptions of \cref{prp:coex_point} and \cref{prp:spinodal_point}.
Then we construct a second dataset to test the full spinodal formula on larger supports.

\paragraph{Data.}
We construct two experiments as follows.
\begin{itemize}
    \item \textbf{Two-key systems.}
    We consider two unit-norm keys $\tilde{k}_1$ and $\tilde{k}_2$ satisfying $\tilde{k}_1^\top\tilde{k}_2=\epsilon$.
    Therefore, we set
    \begin{align*}
        \tilde{k}_1=e_1,\qquad \tilde{k}_2=\epsilon e_1+\sqrt{1-\epsilon^2}\,e_2
    \end{align*}
    Then we apply a random orthogonal transformation $Q_r$ to both keys
    \begin{align*}
        k_1=Q_r\tilde{k}_1,\qquad k_2=Q_r\tilde{k}_2,
    \end{align*}
    where the Gram matrix is preserved under this rotation, but removes the coordinate dependence. 
    We choose the state dimension as $d=16$, correlation as $\epsilon\in\{0,0.05,0.10\cdot\cdot\cdot,0.85\}$ and use 5 independent rotation.
    Hence, this gives $5\times 18=90$ two-key systems.
    
    \item \textbf{General $\boldsymbol{k}$-support systems.}
    In this experiment, we use $k\in\{3,4,5,6,8\}$ each with 10 independently generated systems and ambient dimension $d=20$.
    For each key index $\mu\in\{0,\cdot\cdot\cdot,k-1\}$, we construct the feature vector
    \begin{align*}
        \tilde{k}_\mu=\Big(b,\Big\{a_\ell\cos\Big(\frac{2\pi \ell \mu}{k}\Big),a_\ell\sin\Big(\frac{2\pi\ell \mu}{k}\Big)\Big\}^m_{\ell=1}\Big),\qquad m=\lfloor\frac{k-1}{2}\rfloor,
    \end{align*}
    where we sample $b\sim\text{Unif}(1.5,3.0)$, $a_\ell\sim\text{Unif}(0.15,0.80)$ independently.
    Then we normalize very vector and apply a random orthogonal rotation to $k_\mu$.
    
\end{itemize}

\paragraph{Observables.}
We measure

\begin{itemize}
    \item \textbf{Two-key: Special coexistence formula.}
    Under a uniform distribution and well-separation condition, the coexistence point occurs at
    \begin{align*}
        \beta_{\text{coexist}}=\frac{2}{1-\epsilon}.
    \end{align*}
    This is the special colsed-form result of \cref{prp:coex_point}.

    \item \textbf{Two-key: Numerical energy crossing.}
    We define the free energy difference
    \begin{align*}
        \Delta\Gamma(\beta)=\Gamma^{(\text{mod})}_{\alpha=2}(x_{\text{loc}};\beta)-\Gamma^{(\text{mod})}_{\alpha=2}(x_{\text{mix}};\beta),
    \end{align*}
    where $x_{\text{loc}}\coloneq k_1$ and $x_{\text{mix}}\coloneq (k_1+k_2)/2$.
    Then we compute the coexistence point from $\Delta\Gamma(\hat{\beta}_{\text{coexist}})=0$ numerically.

    \item \textbf{Two-key: Relative error.}
    We define the coexistence relative error as
    \begin{align*}
        E_{\text{coexist}}\coloneq\frac{\big|\hat{\beta}_{\text{coexist}}-\beta_{\text{coexist}}\big|}{\beta_{\text{coexist}}}.
    \end{align*}
    For the spinodal formula, we have $\beta_{\text{sp}}=2/(1-\epsilon)$ in the two-key systems.
    We obtain the empirical spinodal point by solving $\hat{c}_h(\hat{\beta}_{\text{sp}})=0$ numerically, where
    \begin{align*}
        \hat{c}_h(\beta)=\frac{\Gamma^{\text{(mod)}}_{\alpha=2}(x^\star+hu;\beta)+\Gamma^{\text{(mod)}}_{\alpha=2}(x^\star-hu;\beta)-2\Gamma^{\text{(mod)}}_{\alpha=2}(x^\star;\beta)}{h^2},
    \end{align*}
    $h=2\times10^{-5}$, and $u\coloneq(k_1-k_2)/\|k_1-k_2\|_2$.
    Then the spinodal relative error is defined as
    \begin{align*}
        E_{\text{sp}}\coloneq\frac{\big|\hat{\beta}_{\text{sp}}-\beta_{\text{sp}}\big|}{\beta_{\text{sp}}}.
    \end{align*}
    
    \item \textbf{General $\boldsymbol{k}$-support: Spinodal formula.}
    From \cref{prp:spinodal_point}, the predicted critical spinodal point is
    \begin{align*}
        \beta_{\text{sp}}=\frac{2}{\lambda_{\text{max}}(PGP)},
    \end{align*}
    where $P_{\mu\nu}=\one_{\{\mu,\nu\}}-k^{-1}$, and $G_{\mu\nu}=k_\mu^\top k_\nu/d$.
    We compute the empirical spinodal point $\hat{\beta}_{\text{sp}}$ from the equation $\hat{c}_h(\hat{\beta}_{\text{sp}})=0$ numerically.
    
\end{itemize}

\paragraph{Results.}

We present the experiments as follows.

\begin{figure}[h]
    \centering
    \includegraphics[width=1\linewidth]{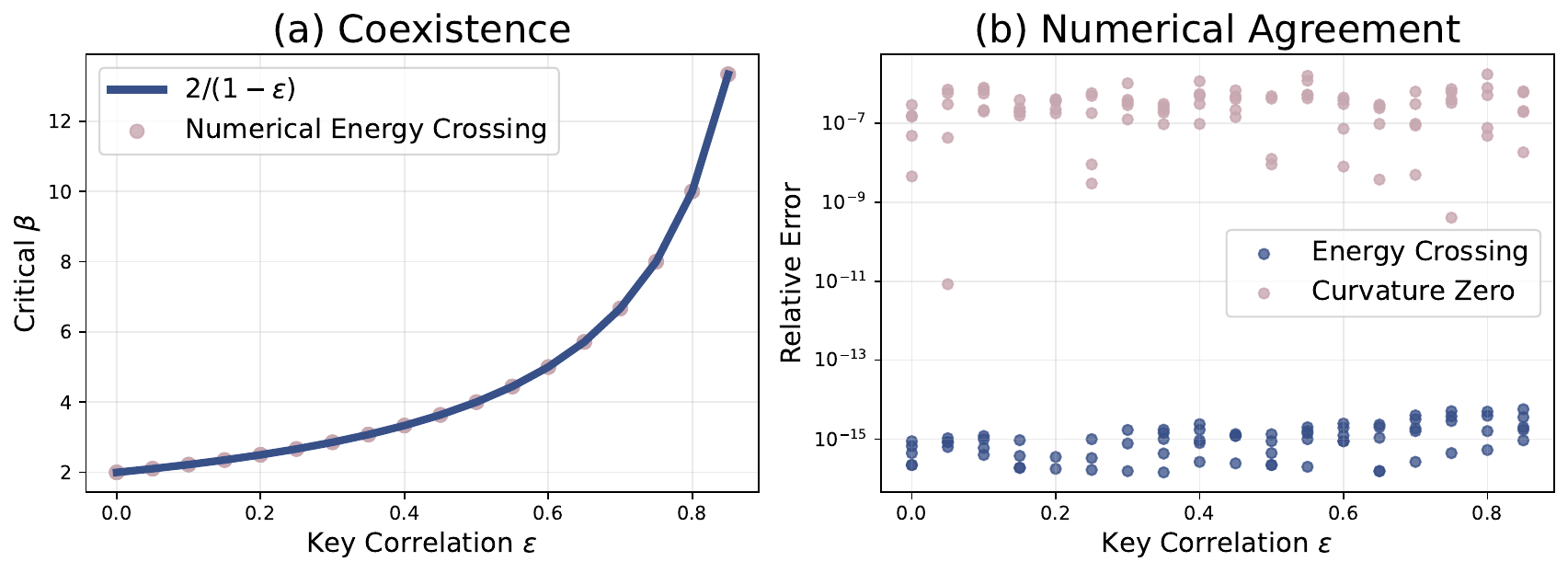}
    \caption{
    \textbf{Two-key coexistence and spinodal validation.}
    \small{
    We consider two unit keys with correlation $\epsilon$.
    \textbf{(a).} 
    The solid blue curve presents $\beta_{\text{coexist}}=2/(1-\epsilon)$.
    The numerical points are the root of the free energy difference $\Delta\Gamma(\hat{\beta}_{\text{coexist}})=0$.
    The numerical energy crossings coincide with the theoretical curve over the full correlation range.
    \textbf{(b).} 
    We plot the relative errors of energy crossing (coexistence) and curvature zero (spinodal).
    The low values validate the theoretical fromulas in \cref{prp:coex_point} and \cref{prp:spinodal_point}.
    }}
    \label{fig:two_key_coexistence_and_spinodal}
\end{figure}

From \cref{fig:two_key_coexistence_and_spinodal} (a), the blue curve represents the theoretical formula $\beta_{\text{coexist}}=2/(1-\epsilon)$.
The numerical energy crossing points lie on the theoretical curve.
This validates the special coexistence formula.
For \cref{fig:two_key_coexistence_and_spinodal} (b), we plot the relative errors of energy crossing (coexistence) and curvature zero (spinodal).
The maximal relative error of the coexistence formula is $5.73\times 10^{-15}$, and which of the spinodal formula is $1.72\times10^{-6}$.
They both validate the theoretical formulas in \cref{prp:coex_point} and \cref{prp:spinodal_point} under the uniform distribution and well-separation conditions.

\begin{figure}[h]
    \centering
    \includegraphics[width=0.5\linewidth]{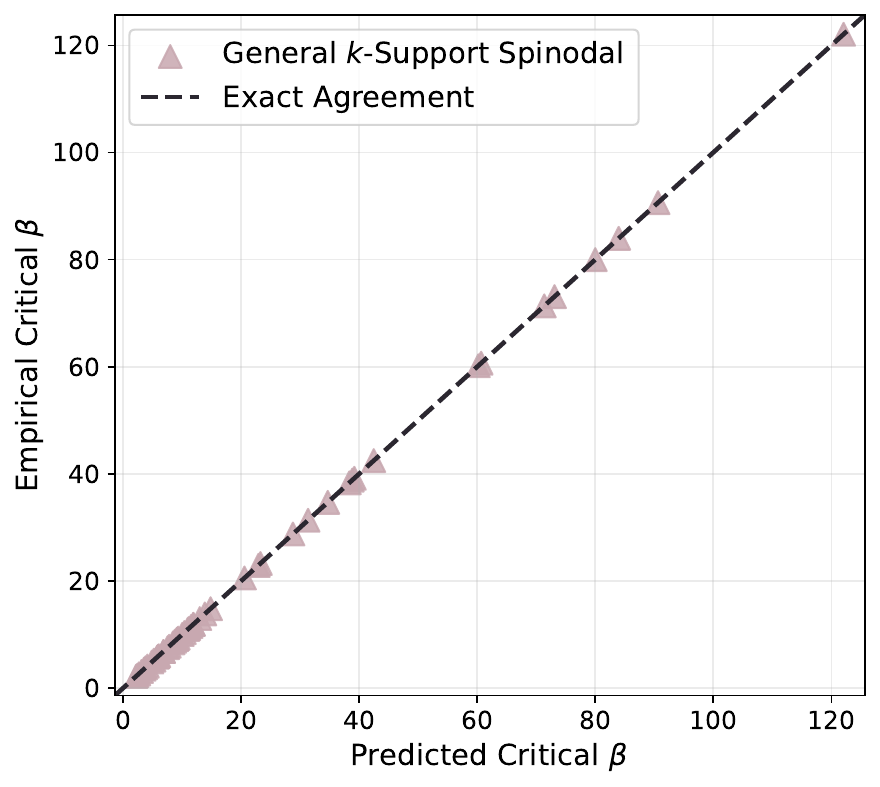}
    \caption{\textbf{General $\boldsymbol{k}$-support spinodal validation.}
    \small{
    We test the general $k$-support spinodal formula $\beta_{\text{jump}}=2/\lambda_{\text{max}}(PGP)$.
    We consider 50 systems with unit keys.
    The numerical points agree with the theoretical values, which validates \cref{prp:spinodal_point}.
    }}
    \label{fig:general_k_support_spinodal}
\end{figure}

In \cref{fig:general_k_support_spinodal}, we compare the theoretical spinodal formula \cref{prp:spinodal_point} and the empirical spinodal values.
Across all 50 general $k$-support systems, the predicted $\beta_{\text{sp}}$ aligns with the empirical one closely, with the maximal relative error $1.92\times10^{-6}$.
Hence, the agreement supports not only the qualitative existence of the transitions, but also the quantitative location of the predicted critical inverse-temperatures.

\clearpage
\def\arxivfont{\rm}
\bibliographystyle{plainnat}

\bibliography{refs}

\end{document}